\documentclass{article}

\usepackage{amssymb}
\usepackage{amsmath}
\usepackage{amscd}
\usepackage{amsthm}
\usepackage{amsfonts}
\usepackage{hyperref}
\usepackage[left=3cm,right=3cm,top=3cm,bottom=3cm,bindingoffset=0cm]{geometry}
\usepackage{latexsym}
\usepackage{graphicx}
\usepackage{bm}
\usepackage{easyeqn}
\usepackage{epigraph}
\usepackage{xcolor}
\usepackage{algorithm}
\usepackage{algorithmicx}
\usepackage{algpseudocode}

\allowdisplaybreaks

\usepackage{biblatex}
\usepackage{xcolor}
\usepackage{enumitem}
\newcommand{\T}{\mathrm{T}}
\def\F{\mathrm{F}}
\def\t{\times}

\def\nsize{n}
\def\RR{\mathbb{R}}
\def\EE{\mathbb{E}}
\def\PP{\mathbb{P}}

\def\ev{\mathbf{e}}
\def\fv{\mathbf{f}}

\def\uv{\mathbf{u}}
\def\xv{\mathbf{x}}
\def\yv{\mathbf{y}}
\def\IC{\mathcal{I}}
\def\BC{\mathcal{B}}

\def\identity{\mathbf{I}}

\def\indicator[#1]{{\operatorname{I}\left\{#1 \right\}}}
\def\cluster[#1]{\operatorname{cl}(#1)}
\def\predicate[#1]{\varphi\left(#1\right)}
\def\threshold{t}
\def\le{\leqslant}
\def\ge{\geqslant}

\def\nclusters{K}
\def\probMatrix{\mathbf{P}}
\def\adjacencyMatrix{\mathbf{A}}
\def\communityMatrix{\mathbf{B}}
\def\constCommunityMatrix{\bar{\communityMatrix}}
\def\nodeCommunityMatrix{\pmb{\mathrm{\Theta}}}
\def\nodeWeights{\bm{\theta}}
\def\sparsityParam{\rho}
\def\probEigenvectors{\mathbf{U}}
\def\probEigenvalues{\mathbf{L}}
\def\adjacencyEigenvectors{\widehat{\probEigenvectors}}
\def\adjacencyEigenvalues{\widehat{\probEigenvalues}}

\def\communityMatrixEstimate{\widehat{\mathbf{B}}}
\def\nodeCommunityMatrixEstimate{\widehat{\nodeCommunityMatrix}}

\def\displaceMatrix{\mathbf{W}}
\def\pureNodesSet{\mathcal{P}}

\def\equalityStatistic{T}
\def\equalityStatisticCenter{\bar{T}}
\def\asymptoticVariance{\mathbf{\Sigma}}
\def\chisquareDistribution[#1]{\chi^2_{#1}}
\def\chiDistribution[#1]{\chi_{#1}}
\def\noncentralChisquare[#1#2]{\chi^2_{#1}(#2)}

\def\Var{\operatorname{Var}}
\def\estimator[#1]{{\widehat{#1}}}
\def\debiasedEigenvalues{\tilde{\probEigenvalues}}

\def\CDF[#1]{\operatorname{CDF}_{#1}}

\def\diagAdjecencyMatrix{\mathbf{D}}

\def\avg{\operatorname{avg}}
\newcommand{\penalizer}{a}

\def\basisMatrix{\mathbf{F}}

\def\basisMatrixEstimate{\widehat{\mathbf{F}}}

\def\permutationMatrix{\mathbf{\Pi}}

\def\resolvent{\mathcal{R}}
\def\tResolvent{\mathcal{P}}
\def\meanFactor{A}
\def\pFactor{\widetilde{\mathcal{P}}}
\def\bv{\mathbf{b}}
\def\tr{\operatorname{tr}}
\def\jMatrix{\mathbf{J}}
\def\lMatrix{\mathbf{L}}
\def\qMatrix{\mathbf{Q}}
\def\diag{\operatorname{diag}}
\def\vv{\mathbf{v}}
\def\purePart{\mathbf{S}_1(i, j)}
\def\notPurePart{\mathbf{S}_2(i, j)}
\def\negativePart{\mathbf{S}_3(i, j)}

\def\sumMatrix{\mathbf{H}}

\def\debiasedEigenvalues{\tilde{\probEigenvalues}}
\def\meanEigs{\mathbf{T}}

\DeclareMathOperator{\KL}{KL}

\newcommand{\ones}{\mathbf{1}}

\newcommand{\zero}{\mathbf{0}}
\newcommand{\probeMatrixFunction}{{\mathtt B}}

\newtheorem{theorem}{Theorem}

\newtheorem{lemma}{Lemma}

\newtheorem{condition}{Condition}
\newtheorem{definition}{Definition}

\newtheorem{proposition}{Proposition}

\title{Optimal Noise Reduction in \\
Dense Mixed-Membership Stochastic Block Models\\
under Diverging Spiked Eigenvalues Condition}

\author{Fedor Noskov and Maxim Panov}

\begin{document}

\maketitle

\begin{abstract}
  Community detection is one of the most critical problems in modern network science. Its applications can be found in various fields, from protein modeling to social network analysis. Recently, many papers appeared studying the problem of overlapping community detection, where each node of a network may belong to several communities. In this work, we consider Mixed-Membership Stochastic Block Model (MMSB) first proposed by~\cite{Airoldi2008}. MMSB provides quite a general setting for modeling overlapping community structure in graphs. The central question of this paper is to reconstruct relations between communities given an observed network. We compare different approaches and establish the minimax lower bound on the estimation error. Then, we propose a new estimator that matches this lower bound. Theoretical results are proved under fairly general conditions on the considered model. Finally, we illustrate the theory in a series of experiments.
\end{abstract}

\section{Introduction}
\label{section: introduction}
  Over the past ten years, network analysis has gained significant importance as a research field, driven by its numerous applications in various disciplines, including social sciences~\cite{jin_mixed_2023}, computer sciences~\cite{bedru_big_2020}, genomics~\cite{li_application_2018}, ecology~\cite{geary_guide_2020}, and many others. As a result, a growing body of literature has been dedicated to fitting observed networks with parametric or non-parametric models of random graphs~\cite{borgs_graphons_2017, goldenberg_survey_2010}. In this work, we are focusing on studying some particular parametric graph models, while it is worth mentioning {\it graphons}~\cite{lovasz_large_2012} as the most common non-parametric model.

  The simplest parametric model in network analysis is the Erdős-Rényi model~\cite{erdos1960evolution}, which assumes that edges in a network are generated independently with a fixed probability $p$, the single parameter of the model. The stochastic block model (SBM; \cite{holland_stochastic_1983}) is a more flexible parametric model that allows for communities or groups within a network. In this model, the network nodes are partitioned into $\nclusters$ communities, and the probability $p_{ij}$ of an edge between nodes $i$ and $j$ depends on only what communities these nodes belong to. The mixed-membership stochastic block model (MMSB; \cite{Airoldi2008}) is a stochastic block model generalization, allowing nodes to belong to multiple communities with varying degrees of membership. This model is characterized by a set of community membership vectors, representing the probability of a node belonging to each community. The MMSB model is the focus of research in the present paper.

  In the MMSB model, for each node $i$, we assume that there exists a vector $\nodeWeights_i \in [0, 1]^\nclusters$ drawn from the $(\nclusters - 1)$-dimensional simplex that determines the community membership probabilities for the given node. Then, a symmetric matrix $\communityMatrix \in [0, 1]^{\nclusters \times \nclusters}$ determines the relations inside and between communities. According to the model, the probability of obtaining the edge between nodes $i$ and $j$ is $\nodeWeights_i^\T \communityMatrix \nodeWeights_j$. Importantly, in the considered model, we allow for self-loops.

  More precisely, let us observe the adjacency matrix of the undirected unweighted graph $\adjacencyMatrix \in \{0, 1\}^{\nsize \t \nsize}$. 
  Under MMSB model $\adjacencyMatrix_{i j} = Bern(\probMatrix_{i j})$ for $1 \le i \le j \le \nsize$, where $\probMatrix_{ij} = \nodeWeights_i^{\T} \communityMatrix \nodeWeights_j = \sparsityParam \, \nodeWeights_i^{\T} \constCommunityMatrix \nodeWeights_j$. Here we denote $\communityMatrix = \sparsityParam \constCommunityMatrix$ with $\constCommunityMatrix \in [0, 1]^{\nclusters \times \nclusters}$ being a matrix with the maximum value equal to $1$ and $\sparsityParam \in (0, 1]$ being the sparsity parameter that is crucial for the properties of this model. Stacking vectors $\nodeWeights_i$ into matrix $\nodeCommunityMatrix$, $\nodeCommunityMatrix_i = \nodeWeights_i^\T$, we get the following formula for the matrix of edge probabilities $\probMatrix$:
  \begin{align*}
    \probMatrix = \nodeCommunityMatrix \communityMatrix \nodeCommunityMatrix^{\T} = \sparsityParam \, \nodeCommunityMatrix \constCommunityMatrix \nodeCommunityMatrix^{\T}.
  \end{align*}
  %
  There is a vast literature on the inference in MMSB. We discuss it in the next section.

\paragraph{Related works} 
  A large body of literature exists on parameter estimation in various parametric graph models. The most well-studied is the Stochastic Block Model, but methods for different graph models can share the same ideas. The maximum likelihood estimator is consistent for both SBM and MMSB, but it is intractable in practice~\cite{celisse_2012, huang_2020}. Several variational algorithms were proposed to overcome this issue; {see the work~\cite{Airoldi2008} that introduced MMSB model, surveys~\cite{lee_review_2019, zhao2017survey} and references therein.} In the case of MMSB, the most common prior on vectors $\nodeWeights_i$, $i \in [\nsize]$ is Dirichlet distribution on a $(\nclusters - 1)$-dimensional simplex with unknown parameter $\bm{\alpha}$. Unfortunately, a finite sample analysis of convergence rates for variational inference is hard to establish. In the case of SBM, it is known that the maximizer of the evidence lower bound over a variational family is optimal~\cite{gaucher_optimality_2021}. Still, there are no theoretical guarantees that the corresponding EM algorithm converges to it.

  Other algorithms do not require any specified distribution of membership vectors $\nodeWeights_i$. For example, spectral algorithms work well under the general assumption of {\it identifiability} of communities~\cite{Mao17}. In the case of SBM, it is proved that they achieve optimal estimation bounds, see the paper~\cite{Yun2015OptimalCR} and references therein. 
  These results motivated several authors to develop spectral approaches for MMSB~\cite{KAUFMANN20183, Mao17}. For example, similar and simultaneously proposed algorithms SPOC~\cite{Panov2018}, SPACL~\cite{mao2021estimating} and Mixed-SCORE~\cite{jin_mixed_2023} optimally reconstruct $\nodeWeights_i$ under the mean-squared error risk~\cite{Jin2017}. Their proposed estimators $\estimator[\communityMatrix]$, $\estimator[\nodeWeights]_i$ achieve the following error rate:
  \begin{align}
    \min_{\permutationMatrix \in \mathbb{S}_\nclusters} & \max_i \Vert \nodeWeights_i - \estimator[\nodeWeights]_i \permutationMatrix \Vert_2 \lesssim \frac{C(K)}{\sqrt{\nsize \sparsityParam}} \label{eq: theta_rate}, \\
    \min_{\permutationMatrix \in \mathbb{S}_\nclusters} & \Vert \estimator[\communityMatrix] - \permutationMatrix \communityMatrix \permutationMatrix^{\T} \Vert_{\F} \lesssim C(\nclusters) \sqrt{\frac{\sparsityParam}{\nsize}}
  \label{eq:spoc_rate}
  \end{align}
  with high probability, where $C(\nclusters)$ is some constant depending on $\nclusters$. Here $\mathbb{S}_\nclusters$ stands for the set of $\nclusters \t \nclusters$ permutation matrices, and $\Vert \cdot \Vert_{\F}$ denotes the Frobenius norm. The algorithm by~\cite{Anandkumar13}, which uses the tensor-based approach, provides the same rate. But the latter has high computational costs and assumes that $\nodeWeights_i$'s are drawn from the Dirichlet distribution. 

  {It is worth mentioning models that also introduce overlapping communities but in a distinct way from MMSB and estimators for them. One example is OCCAM~\cite{zhang2020} which is similar to MMSB but uses $l_2$-normalization for membership vectors. Another example is the Stochastic Block Model with Overlapping Communities~\cite{kaufmann2018spectral,Peixoto2015, arroyo2022overlapping}. Note that the algorithm from~\cite{jin_mixed_2023} can be also applied to a generalization of MMSB, namely the \textit{degree-corrected mixed-membership stochastic block model}~\cite{jin_mixed_2023,qing2023estimating,ke2024optimal}. In our paper, we focus on MMSB only, and leave the case of the degree-corrected MMSB for future research. There is also a line of research that studies parameter estimation in the MMSB or similar models under the assumption of limited resources or missing links~\cite{ibrahim2021mixed, korlakai2014graph, korlakai2016crowdsourced, li2023community}.

  Generally, bounds~\eqref{eq: theta_rate} and~\eqref{eq:spoc_rate} are the best possible if no additional conditions are imposed on the parameters $\nodeWeights_i$ and $\communityMatrix$, see~\cite{Jin2017} for the lower bound on risk of estimating $\nodeWeights_i, i \in [\nsize],$ and Theorem~\ref{theorem: lower bound with almost no pure nodes} below for the lower bound on the risk of estimating $\communityMatrix$ (consider the case of the parameter $\alpha = 0$). However, there exist natural situations where one can consider a meaningful subclass of MMSB problems. Let us call a node $i \in [\nsize]$ \textit{pure} if it completely belongs to a single community. The algorithms discussed above require just one pure node per community to achieve the bounds~\eqref{eq: theta_rate} and~\eqref{eq:spoc_rate}. However, in practice one may have several or even many pure or near pure nodes per community.

  The following question arises: could we improve the estimation quality assuming there there exist multiple pure nodes per each community? The natural idea to improve in this case is to mitigate the noise in MMSB model via certain type of averaging or other postprocessing routine for the pure nodes. In the previous works, authors reduced noise by pruning pure and almost pure vertices to exclude outliers, see SPACL~\cite{mao2021estimating} and GeoMNF~\cite{Mao17}. Another approach is to apply kNN, which was used in~\cite{Jin2017}. Unfortunately, such procedures cannot improve the dependence on $\nsize$ in estimating community memberships $\nodeWeights_i$ in the minimax sense (the worst case example in~\cite{Jin2017} has $\Omega(\nsize)$ pure nodes per each community), although it often enhances numerical performance of such estimators. Meanwhile, we will show below that using averaging, the estimation of $\communityMatrix$ can be dramatically improved for a special subclass of MMSB problems with multiple pure nodes. For that, we will propose a new algorithm \textit{SPOC++}, show the improved upper bounds on the quality of estimation for the matrix $\communityMatrix$ and provide the matching lower bound, see Section~\ref{section: provable guarantees}. Thus, error bounds on estimating $\communityMatrix$ can be used to judge whether a noise reduction subroutine of an algorithm mitigates noise optimally. We will support this logic by showing that our algorithm numerically outperforms SPACL~\cite{mao2021estimating}, GeoMNF~\cite{Mao17} and Mixed-SCORE~\cite{jin_mixed_2023} in estimating both membership vectors $\nodeWeights_i$ and the matrix $\communityMatrix$ when there are a lot of pure nodes per each community, see Section~\ref{section: numerical experiments}.

  We should note that, while the machine learning community has mostly focused on estimation of community memberships $\nodeWeights_i$, the estimation of $\communityMatrix$ has several important applications in econometrics, particularly, in network games. Recently, Geleotti et al.~\cite{galeotti2020targeting} introduced a problem of a central planner intervening in a network game to enhance agents' welfare. The proposed \textit{social welfare problem} is computationally hard, but it can be approximately solved assuming the network has low-dimensional inner structure. One of such assumptions is that the network is sampled from low-rank graphon model or satisfies community structure~\cite{parise2023graphon, gao2019optimal, Medina2020}. Under this assumption, the framework is as follows: first, one should estimate parameters of the network, solve the problem using this parameters, and then interpolate the solution to the initial network. In the community structure case, the estimation of matrix $\communityMatrix$ of connection probabilities between communities is an important intermediate step~\cite{parise2023graphon, Medina2020}. Note that the social welfare problem is not the unique problem for which such framework can be adapted, see papers~\cite{gao2019graphon, gao2019optimal} for the challenge of optimal control in a network.}



\paragraph{Contributions}
  {
  As mentioned above, we prove that the existing estimators of the matrix $\communityMatrix$ satisfy the minimax bound under the general class of MMSB models; see Theorem~\ref{theorem: lower bound with almost no pure nodes} in the case of the parameter $\alpha = 0$. The worst-case example holds when there is only one pure node per each community, and other nodes share their memberships between communities equally. However, that seems not to be the usual setup in the real world, so we ask the following question: can we suggest a better estimator of the matrix $\communityMatrix$ when each community has multiple pure nodes?

  To answer this question, we consider a particular subclass of MMSB models for which we suppose that each community has at least $\Omega(\nsize^\alpha)$ pure nodes for some $\alpha > 0$. First of all, we show that for this class the minimax lower bound for estimation of $\communityMatrix$ becomes $\Omega(\sqrt{\rho / \nsize^{1 + \alpha}})$, which is much smaller than~\eqref{eq:spoc_rate}, see Section~\ref{section: lower bound}.


  Additionally, we aim to propose the estimator $\estimator[\communityMatrix]$ that is computationally tractable and achieves the following error bound:
  \begin{align}
  \label{eq: desired rate}
    \min_{\permutationMatrix \in \mathbb{S}_\nclusters}
      \Vert 
        \estimator[\communityMatrix]
        - 
        \permutationMatrix \communityMatrix \permutationMatrix^\T
      \Vert_{\F}
      \le
      C(\nclusters) \sqrt{\frac{\sparsityParam}{\nsize^{1 + \alpha}}}.
  \end{align}
  with high probability, thus matching the lower bound. This paper focuses on optimal estimation up to dependence on $\nclusters$, while optimal dependence on $\nclusters$ remains an interesting open problem.

  To achieve the optimality, we propose a new algorithm \textit{SPOC++}. As we will show, the resulting procedure is essentially non-trivial (see Section~\ref{section: algorithm} for the detailed description of the algorithm). We also need to impose some conditions to establish the required upper bound. These conditions should be non-restrictive and, ideally, satisfied in practice. The question of the optimality of proposed estimates achieving the rate~\eqref{eq: desired rate} is central to this research. In what follows, we give a positive answer to this question under a fairly general set of conditions, see Section~\ref{section: provable guarantees}.

  Thus, our research answers the question of how to optimally mitigate noise in Mixed-Membership Stochastic Block Model, complementing the results of papers~\cite{Mao17, mao2021estimating, Jin2017}. We hope that our results can be generalized to other factor models.}

  The rest of the paper is organized as follows. We introduce a new \textit{SPOC++} algorithm in Section~\ref{section: algorithm}. Then, in Section~\ref{section: provable guarantees}, we establish the convergence rate for the proposed algorithm and show its optimality. Finally, in Section~\ref{section: numerical experiments}, we conduct numerical experiments that illustrate our theoretical results. Section~\ref{section: discussion} concludes the study with a discussion of the results and highlights the directions for future work. All proofs of ancillary lemmas can be found in Appendix.

\section{Beyond successive projections for parameter estimation in MMSB}
\label{section: algorithm}

\subsection{SPOC algorithm}
\label{section:spoc}
  Various estimators of $\communityMatrix$ and $\nodeCommunityMatrix$ were proposed in previous \newline works~\cite{Mao17,Panov2018,jin_mixed_2023}. In this work, we will focus on the \textit{Successive Projections Overlapping Clustering (SPOC)} algorithm~\cite{Panov2018} that we present in Algorithm~\ref{algo: spoc}. However, we should note that any ``vertex hunting'' method~\cite{jin_mixed_2023} can be used instead of a successive projections algorithm as a base method for our approach.

  The main idea of SPOC is as follows. Consider a $\nclusters$-eigenvalue decomposition of $\probMatrix = \probEigenvectors \probEigenvalues \probEigenvectors^{\T}$. Then, there exists a full-rank matrix $\basisMatrix$ such that $\probEigenvectors = \nodeCommunityMatrix \basisMatrix$ and $\communityMatrix = \basisMatrix \probEigenvalues \basisMatrix^{\T}$. The proof of this statement can be found, for example, in~\cite{Panov2018}. Hence, if we build an estimator of $\basisMatrix$ and $\probEigenvalues$, we immediately get the estimator of $\communityMatrix$. Besides, since $\probEigenvectors = \nodeCommunityMatrix \basisMatrix$, rows of $\probEigenvectors$ lie in a simplex. The vertices of this simplex are rows of matrix $\basisMatrix$. Consequently, we may estimate $\probEigenvectors$ by some estimator $\adjacencyEigenvectors$ and find vertices of the simplex using rows of $\adjacencyEigenvectors$.

  \begin{figure}
    \centering
    \includegraphics[width=\textwidth]{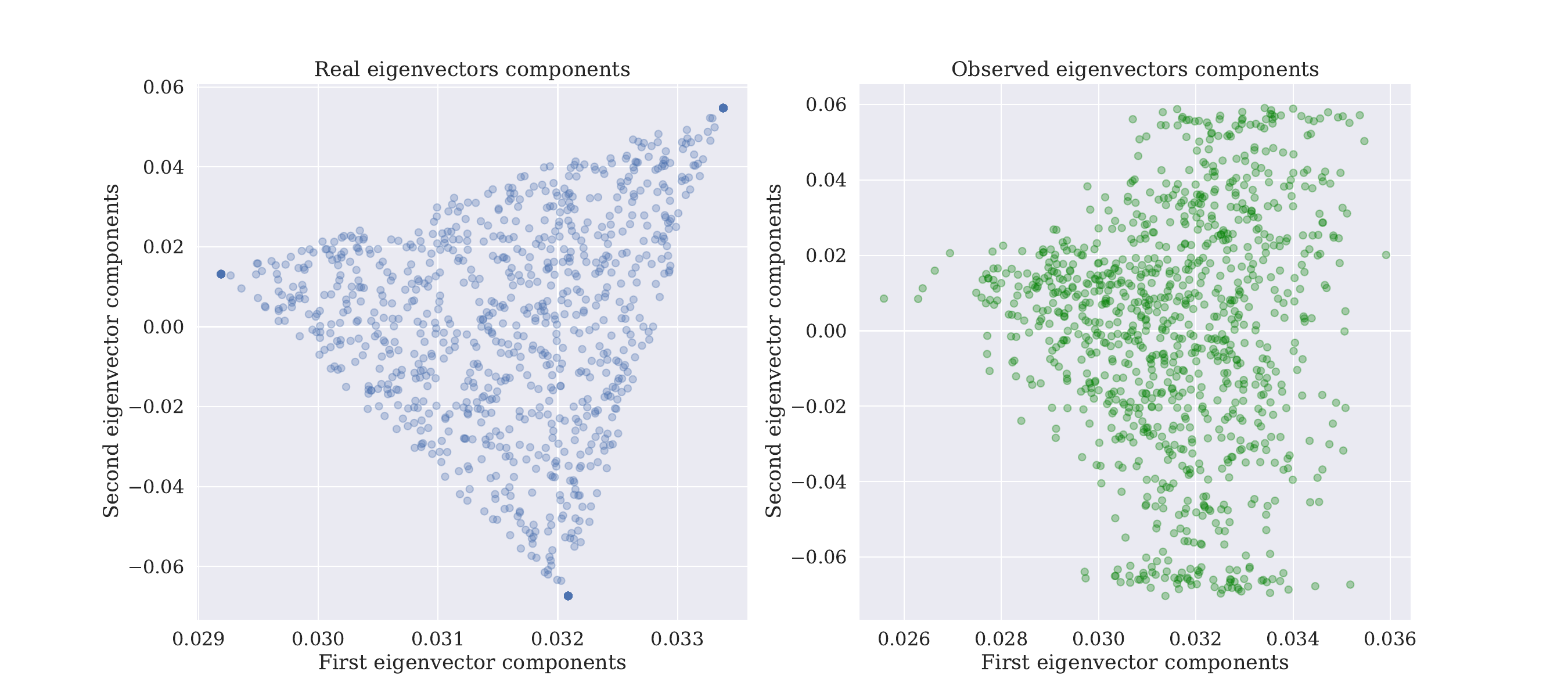}
    \caption{First and second components of rows of matrices $\probEigenvectors$, $\adjacencyEigenvectors$ in the case of $\nclusters$ being equal to \(3\).}
  \label{fig: simplex illustration}
  \end{figure}
  
  \begin{algorithm}[t!]
  \caption{SPA~\cite{Mizutani2016}}
  \label{algo: spa}
  \begin{algorithmic}[1]
    \Require{Matrix $\mathbf{V} \in \RR^{\nsize \times \nclusters}$ and integer $r \le \nclusters$}
    \Ensure{Set of indices $J \subset [\nsize]$}
    \State Set $\mathbf{S}^0 = \mathbf{V}$, $J_0 = \varnothing$
    \For{$t = 1 \ldots r$}
        \State Find $j_t = \arg \min_{i \in [\nsize]} \Vert \mathbf{S}^{t - 1}_i \Vert$
        \State Project rows of $\mathbf{S}^{t - 1}$ on the plane orthogonal to $\mathbf{S}^{t - 1}_{j_t}$:
        \begin{align*}
            \mathbf{S}^{t} = \mathbf{S}^{t - 1} \left(\identity_\nclusters - \frac{\mathbf{S}^{t - 1}_{j_t} (\mathbf{S}^{t - 1}_{j_t})^{\T}}{\Vert \mathbf{S}^{t - 1}_{j_t} \Vert^2_2} \right).
        \end{align*}
        \State Add $j_t$ to the set $J$: $J_t = J_{t - 1} \cup \{j_t\}$.
    \EndFor
    \State \Return $J_t$
  \end{algorithmic}
  \end{algorithm}

  The most natural way to estimate $\probEigenvectors$ and $\probEigenvalues$ is to use a $\nclusters$-eigenvalue decomposition of the adjacency matrix $\adjacencyMatrix \simeq \adjacencyEigenvectors \adjacencyEigenvalues \adjacencyEigenvectors^{\T}$, where columns of $\adjacencyEigenvectors$ are first $\nclusters$ eigenvectors of $\adjacencyMatrix$ and $\adjacencyEigenvalues$ is the diagonal matrix of eigenvalues. The rows of matrix $\adjacencyEigenvectors$ lie in a perturbed version of the simplex corresponding to matrix $\probEigenvectors$, see illustration on Figure~\ref{fig: simplex illustration}. To find vertices of the perturbed simplex, we run \textit{Successive Projections Algorithm (SPA)}, see Algorithm~\ref{algo: spa}. The resulting SPOC algorithm is given in Algorithm~\ref{algo: spoc}.

  \begin{algorithm}[t!]
  \caption{SPOC}
  \label{algo: spoc}
    \begin{algorithmic}[1]
    \Require{Adjacency matrix $\adjacencyMatrix$, number of communities $\nclusters$.}
    \Ensure{Estimators $\nodeCommunityMatrixEstimate$, $\communityMatrixEstimate$}
    
    \State Get the rank-$\nclusters$ eigenvalue decomposition $\adjacencyMatrix \simeq \adjacencyEigenvectors \adjacencyEigenvalues \adjacencyEigenvectors^{\T}$
    \State Run SPA algorithm with input $(\adjacencyEigenvectors, \nclusters)$, which outputs the set of indices $J$ of cardinality $\nclusters$
    \State $\basisMatrixEstimate = \adjacencyEigenvectors[J, :]$
    \State $\communityMatrixEstimate = \basisMatrixEstimate \adjacencyEigenvalues \basisMatrixEstimate^{\T}$
    \State $\nodeCommunityMatrixEstimate = \adjacencyEigenvectors \basisMatrixEstimate^{-1}$
    \end{algorithmic}
  \end{algorithm}

  However, the SPOC-based estimator $\communityMatrixEstimate$ does not allow for obtaining the optimal rate of estimation~\eqref{eq: desired rate}, only achieving the suboptimal one~\eqref{eq:spoc_rate}. The nature of the problem is in the SPA algorithm whose error is driven by the properties of rows of matrix $\adjacencyEigenvectors$ that might be too noisy. In what follows, we will provide a noise reduction procedure for it.

\subsection{Denoising via averaging}
\label{subsection: averaging}
  
  \begin{algorithm}[t!]
  \caption{Averaging procedure}
  \label{algo: averaging procedure}
    \begin{algorithmic}[1]
      \Require{Matrix of eigenvectors $\adjacencyEigenvectors$, diagonal matrix of eigenvalues $\adjacencyEigenvalues$, estimator $\debiasedEigenvalues$, number of communities $\nclusters$, threshold $\threshold_\nsize$, indices $J$, regularization parameter $a$}
      \Ensure{$\basisMatrixEstimate$ --- an estimator of the matrix $\basisMatrix$}.
      
      
      \State Calculate an estimator $\estimator[\displaceMatrix] =\adjacencyMatrix - \adjacencyEigenvectors \debiasedEigenvalues \adjacencyEigenvectors^{\T}$.
      
      \For{$j$ in $J$}
          \For{$j'= 1$ to $\nsize$}
              \State Calculate covariance matrix estimator 
              \begin{equation}
                \estimator[\asymptoticVariance](j, j') = 
                \debiasedEigenvalues^{-1} 
                \adjacencyEigenvectors^{\T} 
                \left(
                  \operatorname{diag}(
                    \estimator[\displaceMatrix]_j^2 + \estimator[\displaceMatrix]_{j'}^2
                  ) - 
                  \estimator[\displaceMatrix]_{j j'}^2 
                  (\ev_{j} \ev_{j'}^{\T} + \ev_{j'} \ev_j^{\T}) 
                \right) 
                \adjacencyEigenvectors 
                \debiasedEigenvalues^{-1},
              \label{eq:cov_matrix_est}
              \end{equation}
              \quad \quad \quad where the square is an element-wise operation.
              
              \State Calculate statistic $\estimator[\equalityStatistic]_{j j'}^{\penalizer} = (\adjacencyEigenvectors_{j} - \adjacencyEigenvectors_{j'}) \left(\estimator[\asymptoticVariance](j, j') + \penalizer \identity \right)^{-1} (\adjacencyEigenvectors_{j} - \adjacencyEigenvectors_{j'})^{\T}$.
          \EndFor
          
          \State Select nodes $\IC_j = \{j' \in [\nsize] \mid \equalityStatistic_{j j'} < \threshold_\nsize\}$
          \State Reduce bias in estimation of $\probEigenvectors:$
          \begin{align}
            \label{eq:diag_adj_matrix}
            \diagAdjecencyMatrix & = \operatorname{diag}\left( \sum_{t = 1}^\nsize \adjacencyMatrix_{it} \right)_{i = 1}^\nsize,
            \\
            \tilde{\probEigenvectors}_{ik} & = \adjacencyEigenvectors_{ik} \left(1 - \frac{\diagAdjecencyMatrix_{ii} - 3/2 \sum_{j = 1}^\nsize \diagAdjecencyMatrix_{jj} \adjacencyEigenvectors_{jk}^2}{\adjacencyEigenvalues^2_{k' k'}} \right) - \sum_{k' \in [\nclusters] \setminus \{k\}} \frac{\debiasedEigenvalues_{k' k'} \cdot \adjacencyEigenvectors_{i k'}}{\debiasedEigenvalues_{k' k'} - \adjacencyEigenvalues_{k k}} \cdot \sum_{j = 1}^{\nsize} \frac{\diagAdjecencyMatrix_{j j} \adjacencyEigenvectors_{j k'} \adjacencyEigenvectors_{j k}}{\adjacencyEigenvalues_{k k}^2}.
          \label{eq:adjusted_eigenvector}
          \end{align}
          \State Average rows of matrix $\tilde{\probEigenvectors}$ over the set $\IC_j$ and write result into vector $\estimator[\fv](j)$:
          \begin{align}
            \estimator[\fv]^{\T}(j) = \frac{1}{|\IC_j|} \sum_{j' \in \IC_j} \tilde{\probEigenvectors}_{j'}.
          \end{align}
      \EndFor
       
      \State Stack together row-vectors $\estimator[\fv]^{\T}(j)$ into matrix $\basisMatrixEstimate$:
      
      \begin{align}
        \basisMatrixEstimate = \left(\estimator[\fv]^{\T}(j) \right)_{j \in J}.
      \end{align}
      
      \State Return matrix $\basisMatrixEstimate$
    \end{algorithmic}
  \end{algorithm}

  The most common denoising tool is averaging because it decreases the variance of i.i.d. variables by $\sqrt{N}$ where $N$ is a sample size. In this work, our key idea is to reduce the error rate of the estimation of the matrix $\basisMatrix$ by $\nsize^{\alpha/2}$ times through averaging $\Theta(\nsize^\alpha)$ rows of $\adjacencyEigenvectors$. The key contribution of this work is in establishing the procedure for finding the rows similar to the rows of $\basisMatrix$ and dealing with their weak dependence on each other. 

  We call the $i$-th node ``pure'' if the corresponding row $\nodeCommunityMatrix_i$ of the matrix $\nodeCommunityMatrix$ consists only of zeros except for one particular entry, equal to $1$. Thus, for the pure node $\probEigenvectors_i = \basisMatrix_k$ for some $k \in [\nclusters]$. If we find many pure nodes and average corresponding rows of $\adjacencyEigenvectors$, we can get a better estimator of rows of $\basisMatrix$ and, consequently, matrix $\communityMatrix$. 

  To find pure nodes, we employ the following strategy. In the first step, we run the SPA algorithm and obtain one vertex per community. Below, we prove under some conditions that SPA chooses ``almost'' pure nodes with high probability. In the second step, we detect the nodes which are ``similar'' to the ones selected by SPA and use the resulting pure nodes set for averaging. The complete averaging procedure is given in Algorithm~\ref{algo: averaging procedure}, while we discuss its particular steps below.

  The choice of similarity measure for detection on similar nodes is crucial for our approach. Fan et al.~\cite{Fan2019_SIMPLE} provide a statistical test for equality of node membership vectors $\nodeCommunityMatrix_i$ and $\nodeCommunityMatrix_j$ based on the statistic $\equalityStatistic_{ij}$. This statistic is closely connected to the displace matrix
  \begin{align*}
    \displaceMatrix = \adjacencyMatrix - \probMatrix
  \end{align*}
  and covariance matrix $\asymptoticVariance(i, j)$ of the vector $(\displaceMatrix_i - \displaceMatrix_j) \probEigenvectors \probEigenvalues^{-1}$:
  \begin{align*}
    \asymptoticVariance(i, j) = \EE \bigl [ \probEigenvalues^{-1} \probEigenvectors^\T (\displaceMatrix_i - \displaceMatrix_j)^\T (\displaceMatrix_i - \displaceMatrix_j) \probEigenvectors \probEigenvalues^{-1} \bigr ].
  \end{align*}
  Thus, the test statistic $\equalityStatistic_{ij}$ is given by
  \begin{align*}
    \equalityStatistic_{ij}
    = 
    (\adjacencyEigenvectors_i - \adjacencyEigenvectors_j) 
    \asymptoticVariance(i, j)^{-1}
    (\adjacencyEigenvectors_i - \adjacencyEigenvectors_j)^{\T}.
  \end{align*}
  However, we do not observe the matrix $\asymptoticVariance(i, j)$. Instead, we use its plug-in estimator $\estimator[\asymptoticVariance](i, j)$ which is described below in Algorithm~\ref{algo: averaging procedure}, see equation~\eqref{eq:cov_matrix_est}. Thus, the resulting test statistic is given by
  \begin{align}
    \estimator[\equalityStatistic]_{ij}
    = 
    (\adjacencyEigenvectors_i - \adjacencyEigenvectors_j)
    \estimator[\asymptoticVariance](i, j)^{-1}
    (\adjacencyEigenvectors_i - \adjacencyEigenvectors_j)^{\T}.
  \end{align}

  Fan et al.\cite{Fan2019_SIMPLE} prove that under some conditions $\equalityStatistic_{ij}$ and $\estimator[\equalityStatistic]_{ij}$ both converge to non-central chi-squared distribution with $\nclusters$ degrees of freedom  and center
  \begin{align}
    \equalityStatisticCenter_{ij} = (\probEigenvectors_i - \probEigenvectors_j)
    \asymptoticVariance(i, j)^{-1}
    (\probEigenvectors_i - \probEigenvectors_j)^{\T}.
  \end{align}
  Thus, $\estimator[\equalityStatistic]_{ij}$ can be considered as a measure of closeness for two nodes. For each node $i$ we can define its neighborhood $\IC_i$ as all nodes $j$ such that $\estimator[\equalityStatistic]_{ij}$ is less than some threshold $\threshold_\nsize$: $\IC_i = \{j \in [\nsize] \mid \estimator[\equalityStatistic]_{ij} < \threshold_\nsize\}$.

  To evaluate $\equalityStatisticCenter_{ij}$, one needs to invert the matrix $\asymptoticVariance(i, j)$. However, matrix $\asymptoticVariance(i, j)$ can be degenerate in the general case. Nevertheless, one can specify some conditions on matrix $\communityMatrix$ to ensure it is well-conditioned. To illustrate it, let us consider the following proposition.
  
  \begin{proposition}
  \label{proposition: when penalizer can be zero}
   Let Conditions~\ref{cond: nonzero B elements}-\ref{cond: theta distribution-a}, defined below, hold. Assume additionally that entries of the matrix $\communityMatrix$ are bounded away from 0 and 1. Then there exist constants $C_1, C_2$ such that for large enough $\nsize$ it holds
   \begin{align*}
        \frac{C_1}{\nsize^2 \sparsityParam} \le \lambda_{\min}(\asymptoticVariance(i, j)) \le \lambda_{\max}(\asymptoticVariance(i, j)) \le \frac{C_2}{\nsize^2 \sparsityParam}
   \end{align*}
   for any nodes $i$ and $j$.
  \end{proposition}

  The proof of Proposition~\ref{proposition: when penalizer can be zero} is moved to Appendix, Section~\ref{section: proof for zero penalizer}.
  
  However, the condition on the entries of the community matrix above might be too strong, while we only need concentration bounds on $\estimator[\equalityStatistic]_{ij}$. To not limit ourselves to matrices $\communityMatrix$ with no zero entries, we consider a regularized version of $\estimator[\equalityStatistic]_{ij}$:
  \begin{align*}
    \estimator[\equalityStatistic]_{ij}^\penalizer = (\adjacencyEigenvectors_i - \adjacencyEigenvectors_j) 
    \left(\estimator[\asymptoticVariance](i, j) + \penalizer \identity\right)^{-1}
    (\adjacencyEigenvectors_i - \adjacencyEigenvectors_j)^{\T}
  \end{align*}
  for some $a > 0$. When $\penalizer = \Theta(\nsize^{-2} \sparsityParam^{-1})$, we show that the statistic $\estimator[\equalityStatistic]^\penalizer_{ij}$ 
  concentrates around 
  \begin{align*}
    \equalityStatisticCenter^a_{ij} = (\probEigenvectors_i - \probEigenvectors_j)
    \left( \asymptoticVariance(i, j) + \penalizer \identity \right)^{-1}
    (\probEigenvectors_i - \probEigenvectors_j)^{\T}.
  \end{align*}
  Practically, if $\estimator[\asymptoticVariance](i, j)$ is well-conditioned, one can use the statistic $\estimator[\equalityStatistic]_{ij}$ without any regularization. In other words, all of our results still hold if $\penalizer = 0$ and $\lambda_{\min}\bigl(\asymptoticVariance(i, j)\bigr) \ge C \nsize^{-2} \sparsityParam^{-1}$ for all $i, j$. But to not impose additional assumptions on either matrix $\communityMatrix$ or $\nodeCommunityMatrix$, in what follows we will use $\estimator[\equalityStatistic]^\penalizer_{ij}$ with $\penalizer = \Theta(\nsize^{-2} \sparsityParam^{-1})$.

\subsection{Estimation of eigenvalues and eigenvectors}
  It turns out that the eigenvalues $\adjacencyEigenvalues$ and eigenvectors $\adjacencyEigenvectors$ of $\adjacencyMatrix$ are not optimal estimators of $\probEigenvalues, \probEigenvectors$ respectively. The asymptotic expansion of $\probEigenvectors$ described in Lemma~\ref{lemma: eigenvector power expansion} suggests a new estimator $\tilde{\probEigenvectors}$ that suppresses some high-order terms in the expansion. For the exact formula, see equation~\eqref{eq:adjusted_eigenvector} in Algorithm~\ref{algo: averaging procedure}. Similarly, a better estimator $\debiasedEigenvalues$ of eigenvalues exists; see equation~\eqref{eq:improved_eigenvalues} in Algorithm~\ref{algo: general scheme}.

  The proposed estimators admit better asymptotic properties than $\adjacencyEigenvalues$ and $\adjacencyEigenvectors$, see \newline Lemmas~\ref{lemma: averaging lemma} and~\ref{lemma: debaised eigenvalues behaviour} in Appendix. In particular, for $\alpha = 1$, it allows us to achieve the convergence rate~\eqref{eq: desired rate} instead of $1 / \nsize$.

  \begin{algorithm}[t!]
    \caption{SPOC++}
    \label{algo: general scheme}
    \begin{algorithmic}[1]
      \Require{Adjacency matrix $\adjacencyMatrix$, threshold $\threshold_\nsize$, regularization parameter $a$}
      \Ensure{Estimators $\nodeCommunityMatrixEstimate$, $\communityMatrixEstimate$}
      
      \State Estimate rank with $\estimator[\nclusters] = \max \bigl \{j \mid \lambda_j(\adjacencyMatrix) \ge 2 \max_i \sqrt{\sum\nolimits_{t = 1}^\nsize \adjacencyMatrix_{it} \log^2 \nsize} \bigr \}$
      
      \State Get the rank-$\estimator[\nclusters]$ eigenvalue decomposition of $\adjacencyMatrix \simeq \adjacencyEigenvectors \adjacencyEigenvalues \adjacencyEigenvectors^{\T}$
      \State Run SPA algorithm with input $(\adjacencyEigenvectors, \estimator[\nclusters])$, which outputs the set of indices $J$ of cardinality $\nclusters$
      \State Calculate the estimator of the eigenvalues' matrix:
      \begin{equation}
        \debiasedEigenvalues_{kk} = \left[ \frac{1}{\adjacencyEigenvalues_{kk}} + \frac{\sum_{i = 1}^{\nsize} \adjacencyEigenvectors_{ik}^2 \cdot \sum_{t = 1}^{\nsize} \adjacencyMatrix_{it}}{\adjacencyEigenvalues_{kk}^3}\right]^{-1}.
      \label{eq:improved_eigenvalues}
      \end{equation}
      \State $\basisMatrixEstimate = \avg(\adjacencyEigenvectors, \adjacencyEigenvalues, \debiasedEigenvalues, \threshold_\nsize, J, a)$, where $\avg$ is the averaging procedure described in Algorithm~\ref{algo: averaging procedure}.
      \State $\communityMatrixEstimate = \basisMatrixEstimate \debiasedEigenvalues \basisMatrixEstimate^{\T}$
      \State $\nodeCommunityMatrixEstimate = \adjacencyEigenvectors \basisMatrixEstimate^{-1}$
    \end{algorithmic}
  \end{algorithm}

\subsection{Estimation of $\nclusters$}
  In the previous sections, we assumed that the number of communities $\nclusters$ is known. However, in practical scenarios, this assumption often does not hold. This section presents an approach to estimating the number of communities.

  The idea is to find the efficient rank of the matrix $\adjacencyMatrix$. Due to Weyl's inequality $|\lambda_j(\adjacencyMatrix) - \lambda_j(\probMatrix)| \le \Vert \adjacencyMatrix - \probMatrix \Vert$. Efficiently bounding the norm $\Vert \adjacencyMatrix - \probMatrix \Vert$, we obtain that it much less than  $2 \max_{i \in [\nsize]} \sqrt{\sum_{t = 1}^\nsize \adjacencyMatrix_{it} \log^2 \nsize}$. However, in its turn, $2 \max_{i \in [\nsize]} \sqrt{\sum_{t = 1}^\nsize \adjacencyMatrix_{it} \log^2 \nsize} \ll \lambda_\nclusters(\probMatrix)$. Thus, we suggest the following estimator:
  \begin{equation*}
    \estimator[\nclusters] = \max \left\{j \mid \lambda_j(\adjacencyMatrix) \ge 2 \max_{i \in [\nsize]} \sqrt{\sum\nolimits_{t = 1}^\nsize \adjacencyMatrix_{it} \log^2 \nsize} \right\}.
  \end{equation*}
  In what follows, we prove that it coincides with $\nclusters$ with high probability if $\nsize$ is large enough; see Section~\ref{sec:number_communities} of Appendix for details.

\subsection{Resulting SPOC++ algorithm}
  Combining ideas from previous sections, we split our algorithm into two procedures: Averaging Procedure (Algorithm~\ref{algo: averaging procedure}) and the resulting SPOC++ method (Algorithm~\ref{algo: general scheme}). 

  However, the critical question remains: how to select the threshold $\threshold_{\nsize}$? In our theoretical analysis (see Theorem~\ref{theorem: main result} below), we demonstrate that by setting $\threshold_{\nsize}$ to be logarithmic in $\nsize$, SPOC++ can recover the matrix $\communityMatrix$ with a high probability and up to the desired error level. However, for practical purposes, we recommend defining the threshold just considering the distribution of the statistics $\estimator[\equalityStatistic]_{i_k j}^{a}$ for different $j$, where $i_k$ is an index chosen by Algorithm~\ref{algo: spa}; see Section~\ref{section: threshold choosing} for details.

\section{Provable guarantees}
\label{section: provable guarantees}

\subsection{Sketch of the proof of consistency}
  We will need several conditions to be satisfied to obtain optimal convergence rates. The most important one is to have many nodes placed near the vertices of the simplex. We will give the exact conditions and statements below, but first, discuss the key steps that allow us to achieve the result. They are listed below.

  \textbf{Step 1. Asymptotics of $\adjacencyEigenvectors_{ik}$.} First, using results of~\cite{Fan2020_ASYMPTOTICS}, we obtain the asymptotic expansion of $\adjacencyEigenvectors_{ik}$. We show that up to a residual term of order $\sqrt{\frac{\log \nsize}{\nsize^3 \sparsityParam}}$ we have
  \begin{align*}
    \adjacencyEigenvectors_{ik} & \approx \probEigenvectors_{ik} + \frac{\ev_i^{\T} \displaceMatrix \uv_k}{t_k} + \frac{\ev_i^{\T} \displaceMatrix^2 \uv_k}{t_k^2} - \frac{3}{2} \cdot \probEigenvectors_{ik} \frac{\uv_k^{\T} \EE \displaceMatrix^2 \uv_k}{t_k^2} + \frac{1}{t_k^2}\sum_{k' \in [\nclusters] \setminus \{k\}} \frac{\lambda_{k'} \probEigenvectors_{i k'}}{\lambda_{k'} - t_k} \cdot \uv_{k'}^\T \EE \displaceMatrix^2 \uv_k,
  \end{align*}
  where $t_k \approx \lambda_k(\probMatrix)$. Matrices $\EE \displaceMatrix^2$ and $\displaceMatrix^2$ can be efficiently estimated by diagonal matrix $\diagAdjecencyMatrix = \operatorname{diag}\left( \sum_{t = 1}^\nsize \adjacencyMatrix_{it} \right)_{i = 1}^\nsize$, see also equation~\eqref{eq:diag_adj_matrix} in Algorithm~\ref{algo: averaging procedure}. Thus, we proceed with plug-in estimation of the second-order terms and obtain the estimator $\tilde{\probEigenvectors}$ defined in~\eqref{eq:adjusted_eigenvector}. Most importantly, the term linear in $\displaceMatrix$ can be suppressed using averaging.

  \textbf{Step 2. Approximating the set of pure nodes.} We show that the difference $|\estimator[\equalityStatistic]^\penalizer_{ij} - \equalityStatisticCenter^\penalizer_{ij}|$ can be efficiently bounded by sum of two terms: one depends on the difference $\Vert \nodeCommunityMatrix_i - \nodeCommunityMatrix_j \Vert_2$ and the other is at most logarithmic. If $i_k$ is an index chosen by SPA and $j \in \pureNodesSet_k$, then $\equalityStatisticCenter_{i_k j}^\penalizer$ is small. Thus, logarithmic threshold $\threshold_n$ will ensure that for all $j \in \pureNodesSet_k$ we have $\estimator[\equalityStatistic]^\penalizer_{i_k j} \le t_n$. Next, Condition~\ref{cond: theta distribution-b} implies that there are a few non-pure nodes in the set $\{ j \mid \estimator[\equalityStatistic]^\penalizer_{i_k j} \le \threshold_n\}$.

  \textbf{Step 3. Averaging.} Finally, we show that redundant terms in the asymptotic expansion of $\tilde{\probEigenvectors}_{i} - \probEigenvectors_i$ vanish after averaging, and it delivers an appropriate estimator of the simplex vertices. After that, we can obtain a good estimator of the matrix $\communityMatrix$.

\subsection{Main result}
  In order to perform theoretical analysis, we state some conditions. Most of these conditions are not restrictive, and below we discuss their limitations, if any.
  \begin{condition}
  \label{cond: nonzero B elements}
    Singular values of the matrix $\constCommunityMatrix$ are bounded away from 0.
  \end{condition}
  The full rank condition is essential as, otherwise, one loses the identifiability of communities~\cite{Mao17}.

  \begin{condition}
  \label{cond: sparsity param bound}
    There is some constant $c$ such that $0 \leqslant c < 1/3$ and $\sparsityParam > \nsize^{-c}$.
  \end{condition}
  Parameter $\rho$ is responsible for the sparsity of the resulting graph. The most general results on statistical properties of random graphs require $\sparsityParam \nsize \to \infty$ as $\nsize \to \infty$~\cite{minh_tang_asymptotically_2022}. In this work, we require a stronger condition to achieve the relatively strong statements we aim at. We think this condition can be relaxed though it would most likely need a proof technique substantially different from ours.

  Next, we demand the technical condition for the probability matrix $\probMatrix$.
  \begin{condition}[Cond.~1 of~\cite{Fan2019_SIMPLE}]
  \label{cond: eigenvalues divergency}
    There exists some constant $c_0 > 0$ such that 
    \begin{align*}
      \min \left\{ 
        \frac{
          |\lambda_i(\probMatrix)|
        }{
          |\lambda_j(\probMatrix)|
        }
        \mid 
        1 \leqslant i < j \leqslant \nclusters, 
        \lambda_i(\probMatrix) \neq \lambda_j(\probMatrix) 
      \right\} \geqslant 1 + c_0.
    \end{align*}
    In addition, we have
    \begin{align}
    \label{eq: variance tends to infinity}
      \max_j \sum_{i = 1}^\nsize \probMatrix_{ij} (1 - \probMatrix_{ij}) \to \infty
    \end{align}
    as $\nsize$ tends to $\infty$.
  \end{condition}
  This condition is required because of the method to obtain asymptotics of eigenvectors of $\adjacencyMatrix$. The idea is to apply the Cauchy residue theorem to the resolvent. Let $\estimator[\uv]_k$ be the $k$-th eigenvector of $\adjacencyMatrix$ and $\uv_k$ be the $k$-th eigenvector of $\probMatrix$. Let $\mathcal{C}_k$ be a contour in the complex plane that contains both $\lambda_k(\probMatrix)$ and $\lambda_k(\adjacencyMatrix)$. If no other eigenvalues are contained in $\mathcal{C}_k$ then
  \begin{align*}
      \oint_{\mathcal{C}_k} \frac{\xv^\T \estimator[\uv]_k \estimator[\uv]_k^\T \yv}{\lambda_k(\adjacencyMatrix) - z} dz = \oint_{\mathcal{C}_k} \xv^\T (\adjacencyMatrix - z \identity)^{-1} \yv dz = \oint_{\mathcal{C}_k} \xv^\T \left( \sum_{k = 1}^\nclusters \lambda_k(\probMatrix) \uv_k \uv_k^\T + \displaceMatrix - z \identity \right)^{-1} \yv dz
  \end{align*}
  for any vectors $\xv, \yv$. The leftmost side is simplified by calculating the residue at $\lambda_k(\adjacencyMatrix)$, and the rightmost side is analyzed via the Sherman--Morrison--Woodbury formula. For the example of obtained asymptotics, see Lemma~\ref{lemma: eigenvector power expansion}.

  The second part of Condition~\ref{cond: eigenvalues divergency} can be omitted if $\sparsityParam < 1$ or there exist $k, k' \in [\nclusters]$ such that $\communityMatrix_{k k'}$ is bounded away from $0$ and 1, since~\eqref{eq: variance tends to infinity} is granted by Conditions~\ref{cond: nonzero B elements}-\ref{cond: sparsity param bound} and~\ref{cond: theta distribution-a} in this case. However, we decided not to impose additional assumptions and left this condition as proposed by~\cite{Fan2019_SIMPLE}.

  Next, we call the $i$-th node in our graph \textit{pure} if $\nodeCommunityMatrix_i$ has $1$ in some position and $0$ in others. We also denote this non-zero position by $\cluster[i]$ and the set of pure nodes by $\pureNodesSet$. Moreover, we define $\pureNodesSet_k = \{ i \in \pureNodesSet \mid \cluster[i] = k\}$. Thus, $\pureNodesSet_k$ is a set of nodes completely belonging to the $k$-th community. It leads us to the following conditions.

  \begin{condition}
  \label{cond: theta distribution-a}
    There exists some constant $C_{\nodeCommunityMatrix}$, independent of $\nsize$, such that
    \begin{align*}
        \lambda_K(\nodeCommunityMatrix^\T \nodeCommunityMatrix) \ge C_{\nodeCommunityMatrix} \nsize,
    \end{align*}
    and $|\pureNodesSet_k| = \Omega(\nsize^{\alpha})$ for some $\alpha \in (0, 1]$ and any $k \in [K]$.
  \end{condition}

  \begin{condition}
  \label{cond: theta distribution-b}
    For any community index $k$, $\delta > 0$ and $\nsize > n_0(\delta)$ there exists $C_\delta$ such that
    \begin{align}
      \sum_{j \not \in \pureNodesSet_k} 
        \indicator[
          \Vert \nodeCommunityMatrix_j - \ev_k \Vert_2 
            \leqslant
            \delta \sqrt{
                \frac{
                   \log \nsize
                }{
                  \nsize \sparsityParam
                }
            }
        ] \le C_\delta \nsize^{\alpha/2},
    \end{align}
    where $\ev_k$ is the $k$-th standard basis vector in $\RR^{\nclusters}$.
  \end{condition}
  Condition~\ref{cond: theta distribution-a} is essential as it requires that all the communities have asymptotically significant mass. As discussed in Section~\ref{subsection: averaging}, we employ row averaging on the eigenmatrix $\adjacencyEigenvectors$ to mitigate noise, specifically focusing on rows corresponding to pure nodes. This averaging process effectively reduces noise by a factor of $n^{\alpha/2}$.  While this condition is not commonly encountered in the context of MMSB, it covers an important intermediate case bridging the gap between the Stochastic Block Model and the Mixed-Membership Stochastic Block Model. {If this condition is not satisfied, we prove that it is possible to obtain a higher minimax lower bound, see Theorem~\ref{theorem: lower bound with almost no pure nodes} for $\alpha = 0$. We consider the assumption $\lambda_K(\nodeCommunityMatrix^\T \nodeCommunityMatrix) = \Omega(\nsize)$ as non-restricting, and illustrate it by the following proposition, which proof is moved to Appendix, Section~\ref{section: proof of theta proposition}. 
  \begin{proposition}
  \label{proposition: large singular value of node community matrix}
      Suppose that for each $k \in [K]$, the ball $\mathcal{B}_{r_K}(\ev_k)$ of the radius $r_K = \frac{1}{6 K}$ contains at least $C n$ points $\nodeWeights_i$, $i \in [n]$, for some constant $C$. Then, we have
      \begin{align*}
          \lambda_{\nclusters}(\nodeCommunityMatrix^\T \nodeCommunityMatrix) \ge \frac{C n}{2}.
      \end{align*}
  \end{proposition}
  In particular, if non-pure $\nodeWeights_i$'s are sampled from the Dirichlet distribution, the least eigenvalue of $\nodeCommunityMatrix^\T \nodeCommunityMatrix$ is bounded away from zero as $\nsize$ tends to infinity, since each ball $\mathcal{B}_{1/6K}(\ev_k)$ has constant probability mass.

  Similarly,} Condition~\ref{cond: theta distribution-b} can be naturally fulfilled if non-pure $\nodeCommunityMatrix_j$ are sampled from the Dirichlet distribution. Indeed, the number of $\nodeCommunityMatrix_j$ in a ball of radius $\sqrt{\frac{
    \log \nsize
    }{
        \nsize \sparsityParam
    }}$ is proportional to $\nsize \cdot \left[ \frac{
    \log \nsize
    }{
        \nsize \sparsityParam
    } \right]^{\frac{\nclusters - 1}{2}}
  $. For example, if $\sparsityParam = \Theta(1)$ and $\nclusters \ge 3$, then we have
  \begin{align*}
    \sum_{j \not \in \pureNodesSet_k} 
        \indicator[
          \Vert \nodeCommunityMatrix_j - \ev_k \Vert_2 
            \leqslant
            \delta \sqrt{
                \frac{
                   \log \nsize
                }{
                  \nsize \sparsityParam
                }
            }
        ] \sim C_\delta \nsize \cdot \left[ \frac{
    \log \nsize
    }{
        \nsize \sparsityParam
    } \right]^{\frac{\nclusters - 1}{2}} \lesssim C_\delta \log^{(\nclusters - 1)/2} \nsize 
  \end{align*}
  with high probability. {Clearly, the latter grows slower than any polynomial function in $\nsize$.}

  One may prove the above by bounding the sum of Bernoulli random variables on the left-hand side using the Bernstein inequality.
  
  These conditions allow us to state the main result of this work.
  \begin{theorem}
  \label{theorem: main result}
  Suppose that $a = \Theta(\nsize^{-2} \sparsityParam^{-1})$. Under Conditions~\ref{cond: nonzero B elements}-\ref{cond: theta distribution-b}, for each positive $\varepsilon$ there are constants $C_t, C_{\communityMatrix}$ depending on $\varepsilon, \nclusters$ such that if we apply Algorithm~\ref{algo: general scheme} with 
    \begin{align}
    \label{cond: selection threshold}
      \threshold_\nsize = C_t \log \nsize,
    \end{align}
    then there is $n_0$ such that for all $\nsize > n_0$ the following inequality holds:
    \begin{align*}
      \PP \left(
        \min_{\permutationMatrix \in \mathbb{S}_\nclusters} 
        \Vert 
          \estimator[\communityMatrix] - 
          \permutationMatrix 
          \communityMatrix
          \permutationMatrix^{\T}
        \Vert_\F
        \geqslant
          C_{\communityMatrix} \sqrt{\frac{\sparsityParam \log \nsize}{\nsize^{1 + \alpha}}}
      \right)
      \leqslant
      \nsize^{- \varepsilon}.
    \end{align*}
  \end{theorem}
  The theorem demands $a = \Theta(\nsize^{-2} \sparsityParam^{-1})$, but the sparsity parameter $\sparsityParam$ is not observed in practice. We suppose that the most convenient choice is $a = 0$, see discussion in Section~\ref{subsection: averaging}. However, if one need to construct a quantity of order $\nsize^{-2} \sparsityParam^{-1}$, one can choose $\bigl(\nsize \lambda_1(\adjacencyMatrix)\bigr)^{-1}$, see Lemma~\ref{lemma: eigenvalues asymptotics}.

\subsection{Proof of Theorem~\ref{theorem: main result}}

    Assume that $\nclusters$ is known. Given $\varepsilon$, choose $t_n = C(\varepsilon) \log \nsize$ such that the event
    \begin{align}
    \label{eq: F permutation upper bound}
      \Vert \estimator[\basisMatrix] - \basisMatrix \permutationMatrix_{\basisMatrix} \Vert_\F \le \frac{C_{\basisMatrix} \sqrt{\log \nsize}}{\nsize^{1 + \alpha/2} \sqrt{\sparsityParam}}
    \end{align}
    has probability at least $1 - \nsize^{- \varepsilon} / 3$ for some constant $C_\basisMatrix$ and permutation matrix $\permutationMatrix_\basisMatrix$. Such $t_n$ exists due to Lemma~\ref{lemma: averaging lemma}. Without loss of generality, we assume that $\permutationMatrix_\basisMatrix = \identity$ in~\eqref{eq: F permutation upper bound}, since changing order of communities does not change the model.
    Meanwhile, due to Lemma~\ref{lemma: debaised eigenvalues behaviour}, for any $\varepsilon > 0$, there is a constant $C_{\probEigenvalues}$ such that for all sufficiently large $n$ we have
    \begin{align*}
      \PP \left (|\debiasedEigenvalues_{k k} - \probEigenvalues_{k k}|  \ge C_{\probEigenvalues} \sqrt{\sparsityParam \log \nsize} \right ) \le n^{-\varepsilon}.
    \end{align*}
    Thus, we have
    \begin{align*}
      \max_k |\debiasedEigenvalues_{k k} - \probEigenvalues_{k k}| \le C_{\probEigenvalues} \sqrt{\sparsityParam \log \nsize}
    \end{align*}
    with probability $1 - \nsize^{-\varepsilon} / 3$ and $n$ sufficiently large.
    Hence, we obtain
    \begin{equation}
      \Vert \communityMatrix - \estimator[\communityMatrix] \Vert_\F 
      \le 
      \Vert \basisMatrix - \estimator[\basisMatrix] \Vert
      \Vert \probEigenvalues \Vert
      \Vert \basisMatrix \Vert_\F
      +
      \Vert \estimator[\basisMatrix] \Vert
      \Vert \probEigenvalues - \debiasedEigenvalues \Vert
      \Vert \basisMatrix \Vert_\F
      +
      \Vert \estimator[\basisMatrix] \Vert
      \Vert \debiasedEigenvalues \Vert
      \Vert \basisMatrix - \estimator[\basisMatrix] \Vert_\F \nonumber \\
      = O \left( \sqrt{\frac{\sparsityParam \log \nsize}{\nsize^{1 + \alpha}}} \right),
    \label{eq: decomposition B error} 
    \end{equation}
    where we use $\Vert \basisMatrix \Vert_\F = O(\nsize^{-1/2})$ and $\Vert \probEigenvalues \Vert = O(\nsize \sparsityParam)$ from Lemmas~\ref{lemma: F rows tensor product} and~\ref{lemma: eigenvalues asymptotics}.
    
    Before we supposed that $\nclusters$ is known. Now consider the case when it does not hold. Due to Lemma~\ref{lemma: estimation of nclusters}, we have $\estimator[\nclusters] = \nclusters$ with probability $1 - \nsize^{-\varepsilon} / 3$ for large enough $\nsize$. It implies that the bound~\eqref{eq: decomposition B error} also holds for the estimator based on $\estimator[\nclusters]$ with probability $1 - \nsize^{-\varepsilon}$.

  \subsection{Lower bound}
  \label{section: lower bound}

  In this section, we show that Theorem~\ref{theorem: main result} is optimal.
  \begin{theorem}
      \label{theorem: lower bound with almost no pure nodes}
      Fix $\alpha \in [0, 1]$. For any estimator $\estimator[\communityMatrix]$, there exists an MMSB model with community matrix $\rho \constCommunityMatrix$ such that
      \begin{enumerate}
          \item each community contains at least $\max\{1,  \lfloor \nsize^{\alpha} / \nclusters \rfloor\}$ pure nodes;
          \item with probability at least $e^{-3.2} / 4$, it holds
          \begin{align*}
              \min_{\permutationMatrix \in \mathbb{S}_K} \Vert \rho \constCommunityMatrix - \permutationMatrix \estimator[\communityMatrix] \permutationMatrix \Vert_{\F} \ge \frac{1}{3066} \sqrt{\frac{\rho \nclusters^3}{\nsize^{1 + \alpha}}},
          \end{align*}
          where the probability is taken with respect to the distribution of the MMSB model.
      \end{enumerate}
  \end{theorem}

  {
  The proof is given in Supplemetary Materials, Section~\ref{section: proof of the genral lower bound}. One may ask whether it is possible to decrease the lower bound using some of Conditions~\ref{cond: nonzero B elements}-\ref{cond: theta distribution-b} other than $|\pureNodesSet| = \Omega(\nsize^{\alpha})$? For example, could one use the fact $\lambda_\nclusters(\nodeCommunityMatrix^\top \nodeCommunityMatrix) = \Omega(\nsize)$ to improve the averaging procedure or the whole algorithm? Unfortunately, this is not the case, and we show it for MMSB with two communities.
  }

  \begin{theorem}
   \label{theorem: lower bound} 
   If $\nsize > C$ for some constant $C$ and $\rho > \nsize^{-1/3}$, then there are two MMSB models $(\nodeCommunityMatrix_0, \rho \constCommunityMatrix_0)$ and $(\nodeCommunityMatrix_1, \rho \constCommunityMatrix_1)$ with two communities, such that
   \begin{enumerate}[label=(\roman*), ref=(\roman{enumi})]
        \item \label{theorem lower bound, property i} for each matrix $\constCommunityMatrix_\ell$, its singular values are at least $1/8$,
        \item \label{theorem lower bound, property ii} for each $\ell \in \{0, 1\}$, we have $\sigma_1(\probMatrix_\ell) / \sigma_{2}(\probMatrix_\ell) > 1 + c_0$, where $c_0 = 1/7$ and $\probMatrix_\ell = \nodeCommunityMatrix_\ell \constCommunityMatrix_\ell \nodeCommunityMatrix_\ell^\T$, and, additionally, 
        \begin{align*}
            \max_j \sum_{i = 1}^\nsize \probMatrix_{ij} (1 - \probMatrix_{ij}) \ge \frac{\nsize \sparsityParam}{16},
        \end{align*}
        \item \label{theorem lower bound, property iii} for both models $\ell \in \{0, 1\}$, each set $|\pureNodesSet_k|$, $k \in [2]$, has cardinality at least $\lfloor \nsize^\alpha / 4096 \rfloor$, and $\lambda_2(\nodeCommunityMatrix_\ell^\T \nodeCommunityMatrix_\ell) \ge C \nsize$ for some absolute constant $C$;
        \item \label{theorem lower bound, property iv} for each $\ell \in \{0, 1\}$ and $ k \in \{1, 2\}$, we have 
        \begin{align*}
            \sum_{j \not \in \pureNodesSet_k} 
                \indicator[
                  \Vert (\nodeCommunityMatrix_\ell)_j - \ev_k \Vert_2 
                    \leqslant
                    \delta \sqrt{
                        \frac{
                           \log \nsize
                        }{
                          \nsize \sparsityParam
                        }
                    }
                ] \le C(\delta),
        \end{align*}
   \end{enumerate}
   and
   \begin{align*}
        \inf_{\estimator[\communityMatrix]} \sup_{\constCommunityMatrix \in \{\constCommunityMatrix_0, \constCommunityMatrix_1\}} \PP \left( 
            \min_{\permutationMatrix \in \mathbb{S}_\nclusters} \Vert \sparsityParam \constCommunityMatrix - \permutationMatrix \estimator[\communityMatrix] \permutationMatrix^\T \Vert_\F \ge \frac{ \sqrt{\sparsityParam}}{108 \cdot \nsize^{(1 + \alpha)/2}}
        \right) \ge \frac{1}{4 e}.
   \end{align*}
  \end{theorem}
  The proof is given in Appendix, Section~\ref{section: proof of the lower bound under conditions}. One can see that Condition~\ref{cond: nonzero B elements} is satisfied by property~\ref{theorem lower bound, property i}, Condition~\ref{cond: sparsity param bound} is satisfied since we guarantee the conclusion of Theorem~\ref{theorem: lower bound} for any $\sparsityParam > \nsize^{-1/3}$, Condition~\ref{cond: eigenvalues divergency} is satisfied by property~\ref{theorem lower bound, property ii}, Condition~\ref{cond: theta distribution-a} is satisfied by property~\ref{theorem lower bound, property iii}, and Condition~\ref{cond: theta distribution-b} is satisfied by property~\ref{theorem lower bound, property iv}. Thus, the estimator defined by Algorithm~\ref{algo: general scheme} is indeed optimal up to the dependence on $\nclusters$.

\section{Numerical experiments}
\label{section: numerical experiments}

\subsection{How to choose an appropriate threshold?}
\label{section: threshold choosing}
  In the considered experiments, we fix $\nclusters$ equal to $3$ and assume that $\communityMatrix$ is well-conditioned. Empirically we show that well-conditioning is vital to achieving a high probability of choosing pure nodes with SPA (see Figure~\ref{fig: curves of t_n}).

  The crucial question in practice for the SPOC++ algorithm is the choice of the threshold. Theoretically, we have established that $t = C \log \nsize$ gives the right threshold to achieve good estimation quality. In practice, there is a simple way to choose the appropriate threshold for nodes $i_1, \ldots, i_\nclusters$ chosen by SPA. For each $i_k$, it is necessary to plot distribution of $\estimator[\equalityStatistic]_{i_k j}$ over $j$. Thus, if the averaging procedure improves the results of SPOC, then there is a corresponding plateau on the plot (see Figure~\ref{fig: distribution of T}).

  Besides, our experiments show that for small $\nclusters$, $\threshold_\nsize = 2 \log \nsize$ is good enough if nodes are generated to satisfy Conditions~\ref{cond: theta distribution-a} and~\ref{cond: theta distribution-b}. This choice corresponds well to the theory developed in this paper.

  \begin{figure}[t!]
    \centering
    \includegraphics[width=1\textwidth]{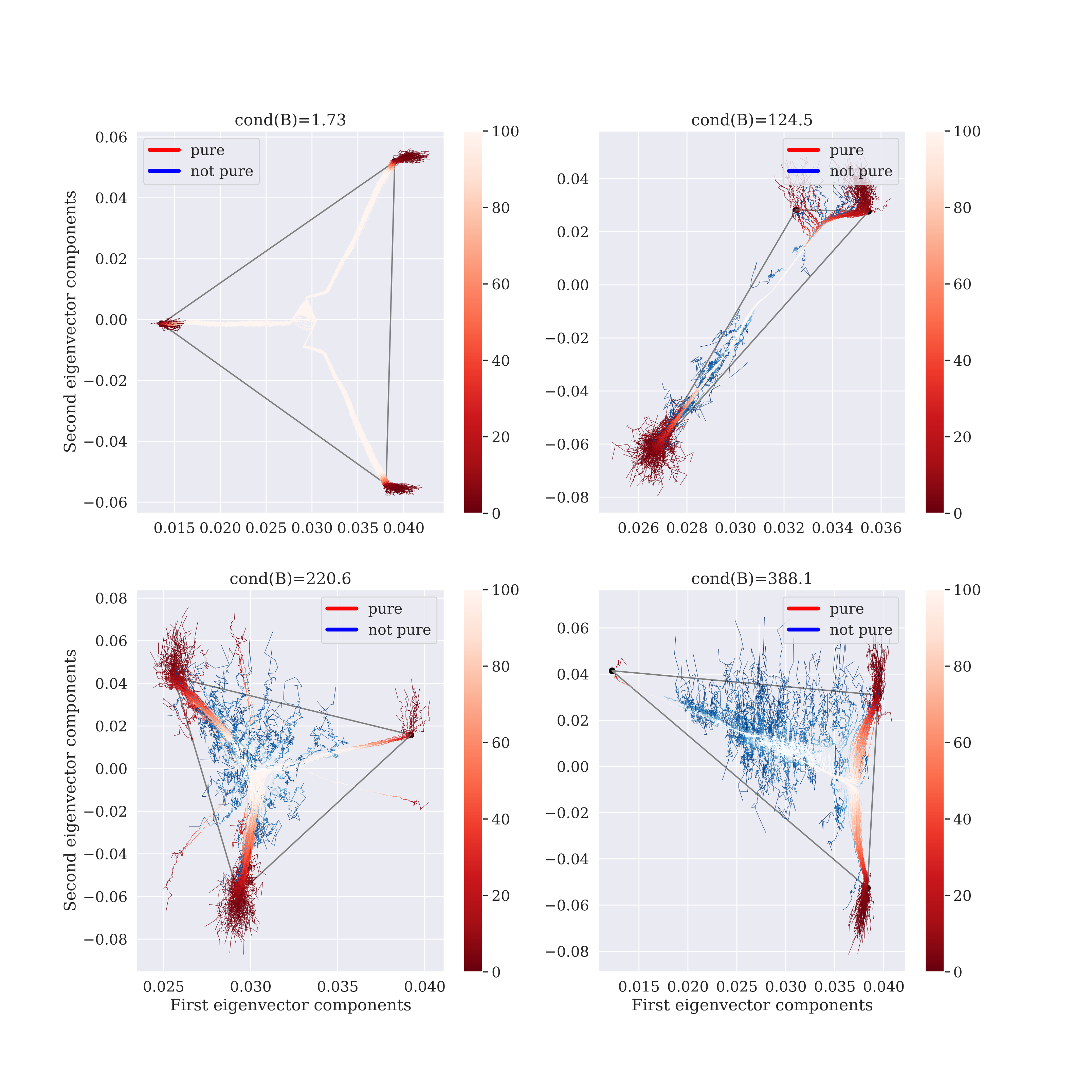}
    \caption{Varying $\threshold_\nsize$, we draw curves $\estimator[\basisMatrix]_{k}(\threshold_\nsize), k \in [\nclusters],$ projected on the two first coordinates, where $\estimator[\basisMatrix]_{k}$ is defined in Algorithm~\ref{algo: averaging procedure}. The intensity of a color corresponds to the value of $\threshold_\nsize$. A curve is red if SPA chooses a pure node, otherwise, the curve is blue. We consider four different matrices $\communityMatrix$, each has different conditional number. For each matrix $\communityMatrix$, we construct one matrix $\probMatrix$, and for this matrix $\probMatrix$, we generate 100 matrices $\adjacencyMatrix$. We choose $\nsize = 1000$ and $|\pureNodesSet_k| / \nsize = 0.07$, $k \in [\nclusters]$. Non-pure membership vectors $\nodeWeights_i$ were sampled from $Dirichlet(1, 1, 1)$.}
  \label{fig: curves of t_n}
  \end{figure}

  \begin{figure}[t!]
    \centering
    \includegraphics[width=\textwidth]{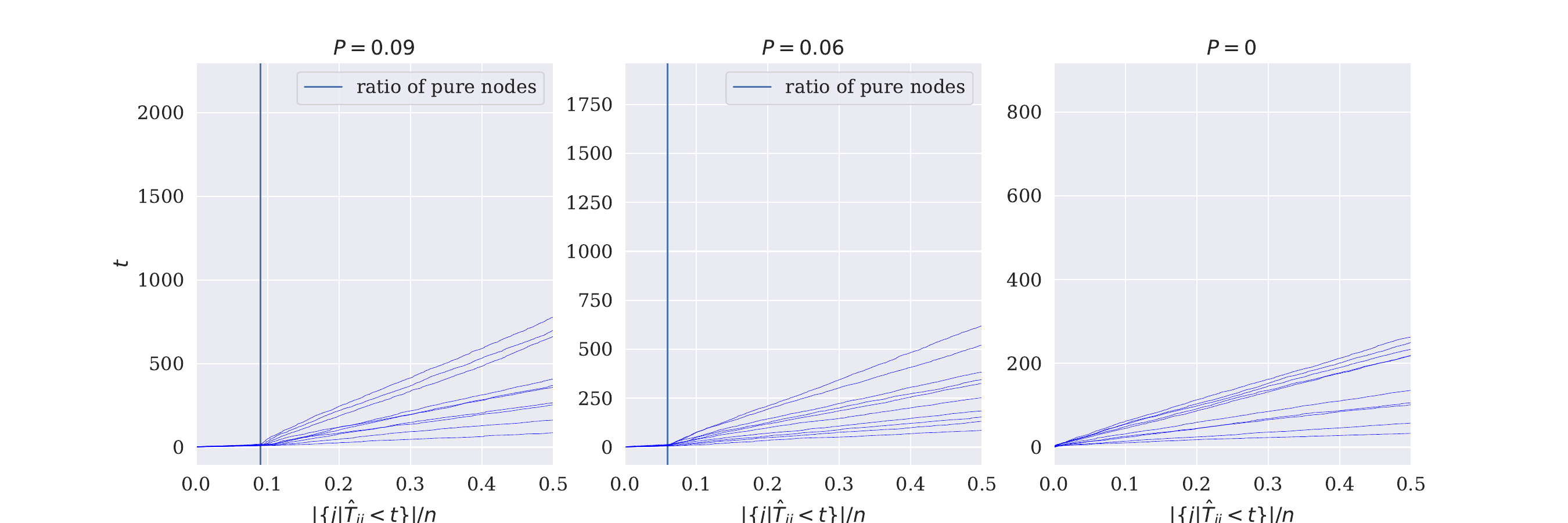}
    \caption{The distribution of $\estimator[\equalityStatistic]_{i_1 j}$ over $j$ where $i_1$ is the first choice of SPA. Here $P = \frac{|\pureNodesSet_k|}{\nsize}$ which is equal for every $k$ in our partial case. It is painted on the plot by the vertical line. Different blue curves are related to different $\nsize$.}
    \label{fig: distribution of T}
  \end{figure}

\subsection{Illustration of theoretical results}
\label{subsection: convergence rate}
  We run two experiments to illustrate our theoretical studies. First, we check the dependence of the estimation error on the number of vertices $\nsize$. Second, we study how the sparsity parameter $\sparsityParam$ influences the error.

  For the first experiment, we provide the following experimental setup. The number of clusters is chosen equal to $3$, and for each $\nsize \in \{500, 1000, 1500, \ldots, 5000\}$ we generate a matrix $\nodeCommunityMatrix$, where the fractions of pure nodes are $\frac{|\pureNodesSet_k|}{\nsize} = 0.09$ (so $\alpha = 1$ in Condition~\ref{cond: theta distribution-a}) and other (not pure) node community memberships are distributed in simplex according to $Dirichlet(1, 1, 1)$. Then we calculated the matrix $\probMatrix$ with $\sparsityParam = 1$. Besides, for each $\nsize$ (and, consequently, matrix $\probMatrix$) we generate the graph $\adjacencyMatrix$ 40 times and compute the error $\min_{\permutationMatrix} \Vert \estimator[\communityMatrix] - \permutationMatrix \communityMatrix \permutationMatrix^{\T} \Vert_\F$, where minimum is taken over all permutation matrices. Hence, for each $\nsize$, we obtain 40 different errors, and, finally, we compute their mean and their quantiles for confidence intervals. The threshold is equal to $2 \log \nsize$.

  \begin{figure}[t!]
    \centering
    \includegraphics[width=\textwidth]{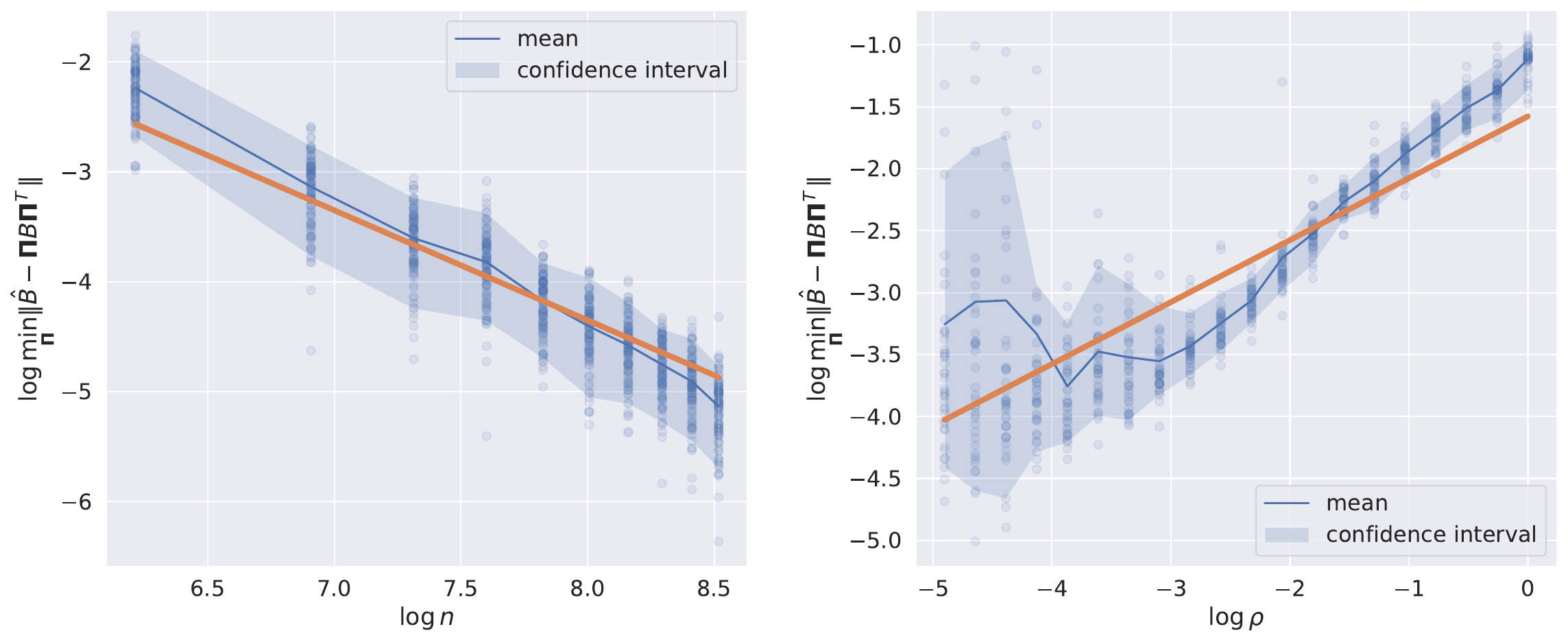}
    \caption{Convergence rate of SPOC++. See the description of setup in Section~\ref{subsection: convergence rate}. On the left subfigure, we draw a red line with slope equals $-1$ to illustrate that the predicted rate of convergence is at most as observed. On the right subfigure, we draw a red line with slope equals $1/2$ to illustrate the same. In both cases, we choose the intercept to minimize the mean squared distance to the observed errors.}
    \label{fig: convergence rate}
  \end{figure}

  We plot the error curves in logarithmic coordinates to estimate the convergence rate. The results are presented in Figure~\ref{fig: convergence rate}, left. It is easy to see that the observed error rate is a bit faster than the predicted one. The slope of the mean error is $-1.21 \pm 0.03$. However, it does not contradict the theory since the provided lower bound holds for some matrix $\communityMatrix$ that may not occur in the experiment.

  We fix $\nsize = 5000$ for the second experiment and generate some matrix $\probMatrix$ as before. Then, we generate 40 symmetric matrices $\mathbf{E}^{(1)}, \ldots, \mathbf{E}^{(40)} \in [0, 1]^{\nsize \times \nsize}$. Entries of each matrix $\mathbf{E}^{(p)}$ are uniformly distributed random variables with the support $[0, 1]$. Given the sparsity parameter $\sparsityParam$ and a matrix $\mathbf{E}^{(p)}$, we generate a matrix $\adjacencyMatrix$ as follows:
  \begin{align*}
    \adjacencyMatrix_{ij} = \indicator[\mathbf{E}^{(p)}_{ij} < \sparsityParam \cdot \probMatrix_{ij}].
  \end{align*}
  We apply our algorithm to $\adjacencyMatrix$ and compute the error of $\estimator[\communityMatrix]$.

  We study our algorithm for 20 different values of $\sparsityParam$. The results are presented on Figure~\ref{fig: convergence rate}, right. We calculate the slope of the mean error which turns out to be $0.47 \pm 0.06$.

  \subsection{Comparison with other algorithms}
  \label{subsetion: comparison}

  \begin{figure}[t!]
    \centering
    \hspace*{-1.7cm} 
    \includegraphics[width=\textwidth]{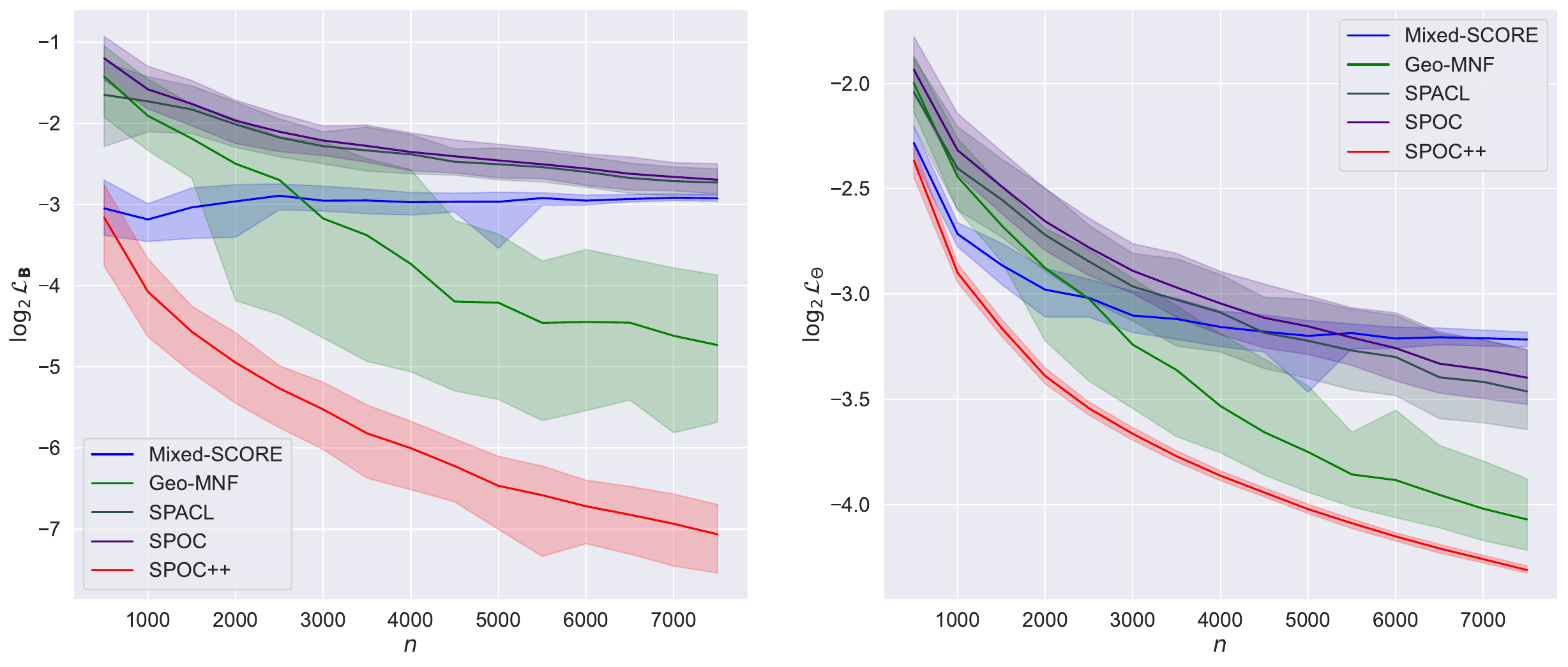}
    \caption{Error of reconstruction of $\communityMatrix$ and $\nodeCommunityMatrix$ for different algorithms. See setup in Section~\ref{subsetion: comparison}.}
    \label{fig: comparison}
  \end{figure}

  We compare the performance of our algorithm with Algorithm~\ref{algo: spoc}, GeoMNF~\cite{Mao17}, SPACL~\cite{mao2021estimating} and Mixed-SCORE~\cite{jin_mixed_2023}. We set the number of communities to $3$. As in Section~\ref{subsection: convergence rate}, we generate a well-conditioned matrix $\constCommunityMatrix$, then, for each $\nsize \in \{500, 1000, \ldots, 7500\}$, we choose $\sparsityParam = 1$ and generate a matrix $\probMatrix$. As previously, for each community, the number of pure nodes was equal to $0.09 \cdot \nsize$, and membership vectors of non-pure nodes were sampled from the $Dirichlet(1, 1, 1)$ distribution. Given a matrix of connection probabilities $\probMatrix$, we generate 100 different matrices $\adjacencyMatrix$, and for each of them, we compute the error of reconstruction of $\communityMatrix$ and $\nodeCommunityMatrix$, defined as follows:
  \begin{align*}
    \mathcal{L}_{\communityMatrix}(\communityMatrix, \estimator[\communityMatrix]) = \min_{\permutationMatrix \in \mathbb{S}_\nclusters} \Vert \estimator[\communityMatrix] - \permutationMatrix^\T \communityMatrix \permutationMatrix \Vert_{\F}, \quad
    \mathcal{L}_{\nodeCommunityMatrix}(\nodeCommunityMatrix, \estimator[\nodeCommunityMatrix]) = \min_{\permutationMatrix \in \mathbb{S}_\nclusters} \frac{\Vert \estimator[\nodeCommunityMatrix] - \nodeCommunityMatrix \permutationMatrix \Vert_{\F}}{
        \Vert \nodeCommunityMatrix \Vert_{\F}
    }.
  \end{align*}
  {Both GeoMNF~\cite{Mao17} and Mixed-SCORE~\cite{jin_mixed_2023} impose some structural assumptions on the matrix $\communityMatrix$, that are not satisfied in our case. Given an estimator $\hat{\nodeCommunityMatrix}$, we employ the following estimator $\estimator[\communityMatrix]$ for them:
  \begin{align*}
      \estimator[\communityMatrix] = (\estimator[\nodeCommunityMatrix]^\T \estimator[\nodeCommunityMatrix])^{-1} \nodeCommunityMatrix^\T \adjacencyMatrix \nodeCommunityMatrix (\estimator[\nodeCommunityMatrix]^\T \estimator[\nodeCommunityMatrix])^{-1}.
  \end{align*}
  }
  The results are presented in Figure~\ref{fig: comparison}. We plot the mean errors of each algorithm together with {empirical $0.9$-confidence intervals.} As one can see, the proposed SPOC++ algorithm significantly outperforms all the competitors. {The poor performance of Mixed-SCORE for large $\nsize$ can be explained by the fact that it is designed for the degree-corrected mixed-membership stochastic block model, which can lead to some identifiability issues in our setup.}

\section{Discussion}
\label{section: discussion} 

  In this paper, we propose a new algorithm \textit{SPOC++} which optimally reconstructs community relations in MMSB in the minimax sense. The study is done under the assumption that significant fraction of pure nodes exists among all the nodes in the network; see  Condition~\ref{cond: theta distribution-a}. Additionally, under this assumption, we show that our procedure can improve the reconstruction of the community memberships as well. Let us note that Condition~\ref{cond: theta distribution-a} covers not only Stochastic Block Model (with all the nodes being pure) and Mixed-Membership Stochastic Block Model with many pure nodes but also an important case of MMSB with almost no pure nodes. Thus, this assumption is pretty general and can be naturally satisfied in practice.

\printbibliography

\newpage

\appendix

\section{Proof of Proposition~\ref{proposition: when penalizer can be zero}}
\label{section: proof for zero penalizer}
  Let us estimate eigenvalues of matrix $\asymptoticVariance(i, j)$. After some straightforward calculations, we have
    \begin{align}
    \label{eq: asymptoticVariance}
      \asymptoticVariance (i, j) &= 
      \probEigenvalues^{-1}
      \probEigenvectors^{\T}
      \EE \left(
        \displaceMatrix_i
        -
        \displaceMatrix_j
      \right)^{\T}
      \left(
        \displaceMatrix_i
        -
        \displaceMatrix_j
      \right)
      \probEigenvectors
      \probEigenvalues^{-1} \nonumber \\
      & = 
      \probEigenvalues^{-1}
      \probEigenvectors^{\T}
      \left(
        \operatorname{diag} (
          \EE \displaceMatrix_i^2
          +
          \EE \displaceMatrix_j^2
        )
        - \EE \displaceMatrix_{ij}^2
        (\ev_i \ev_j^{\T} + \ev_j \ev_i^{\T})
      \right)
      \probEigenvectors
      \probEigenvalues^{-1}.
    \end{align}
    The maximum eigenvalue can be estimated using a norm of the matrix:
    \begin{align*}
      \lambda_{\max} \bigl( 
        \asymptoticVariance(i, j)
      \bigr)
      =
      \Vert
        \asymptoticVariance(i, j)
      \Vert
      \leqslant
      \Vert
        \probEigenvalues^{-1}
      \Vert^2
      \Vert
        \probEigenvectors
      \Vert^2
      \left(
        \Vert 
          \operatorname{diag} (
            \EE \displaceMatrix_i^2
            +
            \EE \displaceMatrix_j^2
          )
        \Vert
        + 2 \EE \displaceMatrix_{ij}^2
      \right),
    \end{align*}
    \begin{align*}
      \lambda_{\max} \bigl(
        \asymptoticVariance(i, j)
      \bigr)
      \leqslant 
      \frac{
        4 \sparsityParam
      }{
        \lambda_{\nclusters}^2 (\probMatrix)
      },
    \end{align*}
    since $\EE \displaceMatrix_{ij}^2 = \probMatrix_{ij} - \probMatrix_{ij}^2$. Due to Lemma~\ref{lemma: eigenvalues asymptotics}, we have $\lambda_{\nclusters}^2(\probMatrix) = \Omega(\nsize \sparsityParam)$, so the upper bound holds. To find the lower bound of the minimal eigenvalue of $\asymptoticVariance(i, j)$, we need Condition~\ref{cond: theta distribution-a}. Let us rewrite~\eqref{eq: asymptoticVariance} in the following way:
    \begin{align*}
      \asymptoticVariance(i, j) = 
      \probEigenvalues^{-1}
      \bigl(
        \purePart +
        \notPurePart -
        \negativePart
      \bigr)
      \probEigenvalues^{-1},
    \end{align*}
    where
    \begin{align*}
      \purePart & =
      \sum_{m \in \pureNodesSet}
          \left(
            \EE \displaceMatrix_{im}^2
            +
            \EE \displaceMatrix_{jm}^2
          \right)
          \probEigenvectors_{m}^{\T}
          \probEigenvectors_{m}, \\
      \notPurePart & =
      \sum_{m \not \in \pureNodesSet}
          \left(
            \EE \displaceMatrix_{im}^2
            +
            \EE \displaceMatrix_{jm}^2
          \right)
          \probEigenvectors_{m}^{\T}
          \probEigenvectors_{m}, \\
      \negativePart & =
      \EE \displaceMatrix_{ij}^2
      \left(
          \probEigenvectors_i^{\T}
          \probEigenvectors_j
          +
          \probEigenvectors_j^{\T}
          \probEigenvectors_i
      \right).
    \end{align*}
    Now we analyze $\purePart$. Since $\probEigenvectors = \nodeCommunityMatrix \basisMatrix$, we obtain
    \begin{align*}
      \purePart & = 
      \sum_{k = 1}^\nclusters
        \nsize_k 
        \left(
          \nodeCommunityMatrix_i \communityMatrix_{k}^{\T} - 
          (\nodeCommunityMatrix_i \communityMatrix_{k}^{\T})^2
          +
          \left(
            \nodeCommunityMatrix_j \communityMatrix_k^{\T}
            -
            (\nodeCommunityMatrix_j \communityMatrix_k^{\T})^2
          \right)
        \right)
        \basisMatrix_k^{\T}
        \basisMatrix_k \\
      & \geqslant
      2 \sum_{k = 1}^\nclusters
          \nsize_k 
          \min
          \left\{
            \min_{k'} \communityMatrix_{k' k} - 
            ( \min_{k'} \communityMatrix_{k' k})^2
            ,
            \max_{k'} \communityMatrix_{k' k} - 
            ( \max_{k'} \communityMatrix_{k' k})^2
          \right\}
          \basisMatrix_k^{\T}
          \basisMatrix_k \\
      & =
      \nsize \sparsityParam
      \sum_{k = 1}^\nclusters
        \alpha_k
        \basisMatrix_k^{\T}
        \basisMatrix_k,
    \end{align*}
    where $\alpha_k$, $k \in [\nclusters]$ are bounded away from 0 since entries of $\communityMatrix$ are bounded away from 0 and 1 by the assumptions of the proposition. Lemma~\ref{lemma: F rows tensor product} implies that there are such constants $C_1, C_2$ that
    \begin{align*}
      \sparsityParam C_1
      \leqslant
      \lambda_{\min}\bigl(\purePart\bigr)
      \leqslant
      \lambda_{\max}\bigl(\purePart\bigr)
      \leqslant
      \sparsityParam C_2.
    \end{align*}
    Since $\notPurePart$ is non-negative defined, we state that $\lambda_{\min}\bigl(\notPurePart\bigr) \geqslant 0$. 

    In order to estimate eigenvalues of $\negativePart$, we use Lemma~\ref{lemma: eigenvectors max norm}:
    \begin{align*}
      \lambda_{\max}\bigl(\negativePart\bigr)
      \leqslant 
      \sparsityParam 
      \left(
        \Vert
          \probEigenvectors_i^{\T}
          \probEigenvectors_j
        \Vert
        +
        \Vert
          \probEigenvectors_j^{\T}
          \probEigenvectors_i
        \Vert
      \right)
      \leqslant
      \frac{2 \sparsityParam \nclusters C_{\probEigenvectors}^2}{\nsize}.
    \end{align*}
    Applying multiplicative Weyl's inequality, we get
    \begin{align}
      \lambda_{\min} \bigl(
        \asymptoticVariance(i, j)
      \bigr)
      \geqslant
      \frac{1}{\lambda^2_\nclusters(\probMatrix)}
      \bigl[
        \lambda_{\min}\bigl(\purePart\bigr)
        -
        \lambda_{\max}\bigl(\negativePart\bigr)
      \bigr]
      \geqslant
      \frac{1}{\nsize^2 \sparsityParam} \left(
        c_1 - \frac{c_2}{\nsize}
      \right)
    \end{align}
    for some positive constants $c_1$, $c_2$. Thus, the proposition follows.

\section{Proof of Proposition~\ref{proposition: large singular value of node community matrix}}
\label{section: proof of theta proposition}

For each $k \in [\nclusters]$, we choose $\lceil C \nsize \rceil$ points $\nodeWeights_i$, $i \in \nsize$, that belong to $\mathcal{B}_{r_K}(\ev_k)$, and denote the set of their indices by $\mathcal{F}_k$. Note that by our choice of $r_K$ all $\mathcal{F}_k$ are disjoint. Then, we have the following lower bound:
\begin{align*}
    \nodeCommunityMatrix^\T \nodeCommunityMatrix = \sum_{i = 1}^\nsize \nodeWeights_i \nodeWeights_i^\T \succeq \sum_{k \in [\nclusters]} \sum_{i \in \mathcal{F}_k} \nodeWeights_i \nodeWeights_i^\T,
\end{align*}
where $A \succeq B$ means that $A - B$ is semi-positive definite. Let $\mathbf{g}_i$ be a vector of the norm at most $1$ such that $\nodeWeights_i = \ev_k + r_K \cdot \mathbf{g}_i$ holds for each $i \in \mathcal{F}_k$. It yields the following:
\begin{align*}
    \sum_{i \in \mathcal{F}_k} \nodeWeights_i \nodeWeights_i^\T = |\mathcal{F}_k| \ev_k \ev_k^\T + r_K \sum_{i \in \mathcal{F}_k} (\mathbf{g}_i \ev_k^\T + \ev_k \mathbf{g}_i^\T) + r_K^2 \sum_{i \in \mathcal{F}_k} \mathbf{g}_i \mathbf{g}_i^\T.
\end{align*}
Since $|\mathcal{F}_k|$ are all equal to $\lceil C n \rceil$, we have
\begin{align*}
    \nodeCommunityMatrix^\T \nodeCommunityMatrix \succeq \lceil C \nsize \rceil \cdot \identity + r_K \sum_{k \in [K]}  \sum_{i \in \mathcal{F}_k} (\mathbf{g}_i \ev_k^\T + \ev_k \mathbf{g}_i^\T) + r_K^2 \sum_{k \in [K]} \sum_{i \in \mathcal{F}_k} \mathbf{g}_i \mathbf{g}_i^\T,
\end{align*}
and so
\begin{align*}
    \lambda_\nclusters(\nodeCommunityMatrix^\T \nodeCommunityMatrix) \ge \lceil C \nsize \rceil - r_K \sum_{k \in [K]} \sum_{i \in \mathcal{F}_k} (2 \Vert \ev_k \mathbf{g}_i^\T \Vert + r_K \Vert \mathbf{g}_i \mathbf{g}_i^\T \Vert),
\end{align*}
where $\Vert \cdot \Vert$ stands for the operator norm. Note that $\Vert \ev_k \mathbf{g}_i^ \T\Vert, \Vert \mathbf{g}_i \ev_k^\T \Vert, \Vert \mathbf{g}_i \mathbf{g}_i^\T \Vert \le 1$. By our choice of $r_K = 1/6K$, we have
\begin{align*}
    \lambda_\nclusters(\nodeCommunityMatrix^\T \nodeCommunityMatrix) \ge \lceil C \nsize \rceil - \frac{3}{6\nclusters} \sum_{k \in [\nclusters]} |\mathcal{F}_k| = \frac{\lceil C \nsize \rceil}{2}.
\end{align*}

\section{Proofs for Theorem~\ref{theorem: main result}}

Here and further following~\cite{Fan2019_SIMPLE} we use the notation $O_{\prec}(\cdot)$:
  \begin{definition}
  \label{def: O_prec definition}
   Suppose $\xi$ and $\eta$ to be random variables that may depend on $\nsize$. We say that $\xi = O_{\prec}(\eta)$ if and only if for any positive $\varepsilon$ and $\delta$ there exists $\nsize_0$ such that for any $\nsize > \nsize_0$
  \begin{align}
    \PP \left(
      |\xi| > \nsize^{\varepsilon} |\eta|
    \right)
    \leqslant
    \nsize^{-\delta}.
  \end{align}
  \end{definition}
  It is easy to check the following properties of $O_{\prec}(\cdot)$. If $\xi_1 = O_{\prec}(\eta_1)$ and $\xi_2 = O_{\prec}(\eta_2)$ then $\xi_1 + \xi_2 = O_{\prec} (|\eta_1| + |\eta_2|)$, $\xi_1 + \xi_2 = O_{\prec} \left(\max\{|\eta_1|, |\eta_2|\} \right)$ and $\xi_1 \xi_2 = O_{\prec}(\eta_1 \eta_2)$.
  
  Additionally, we introduce a bit different type of convergence.
  \begin{definition}
  \label{def: O_l definition}
   Suppose $\xi$ and $\eta$ to be random variables that may depend on $\nsize$. Say $\xi = O_{\ell} (\eta)$ if for any $\varepsilon > 0$ there exist $n_0$ and $\delta > 0$ such that
  \begin{align*}
    \PP \left( \xi \ge \delta \eta \right) \le \nsize^{-\varepsilon}
  \end{align*}
  holds for all $\nsize \ge n_0$.
  \end{definition}
  It preserves the properties of $O_{\prec}(\cdot)$ described previously. Moreover, $O_{\prec}(\eta) = O_\ell (\nsize^\alpha \cdot \eta)$ for any $\alpha > 0$.
  
  Further, we will consider various random variables $\xi_i$ indexed by $i \in [\nsize]$. Mostly, they have the form $\ev_i^{\T} \mathbf{X}$ for some random matrix $\mathbf{X}$. Formally, if $\xi_i = O_{\prec} (\eta_\nsize)$, we are not allowed to state $\max_{i} \xi_i = O_{\prec}(\eta_\nsize)$ since $n_0$ for different $i$ may be distinct and not be bounded. Nevertheless, the source of $O_{\prec}(\cdot)$ is random variables of the form $\xv^{\T} (\displaceMatrix^\ell - \EE \displaceMatrix^\ell) \yv$, that can be uniformly bounded using all moments provided by Lemma~\ref{lemma: power deviation}. Thus, $\xi_i = O_{\prec}(\eta_\nsize)$ for any $i \in S \subset [\nsize]$ implies $\max_{i \in S} \Vert \xi_i \Vert_2 = O_{\prec}(\eta_\nsize)$.
  
  The order $O_\ell(\eta_{\nsize})$ appears when we combine $O_{\prec}(\eta_{\nsize} / \nsize^{\alpha})$ for some $\alpha > 0$ and random variable $X$ bounded by $\eta_{\nsize}$ via Freedman or Bernstein inequalities that provide exactly the same $n_0$ for different $i$. Consequently, taking maximum over any subset of $[\nsize]$ is also allowed.

\subsection{Asymptotics of eigenvectors}

The following lemma allows us to establish the behavior of eigenvectors.
\begin{lemma}
\label{lemma: eigenvector power expansion}
  Under Conditions~\ref{cond: nonzero B elements}-\ref{cond: theta distribution-a} it holds that
  \begin{align*}
    \adjacencyEigenvectors_{ik} & = \probEigenvectors_{ik} + \frac{\ev_i^{\T} \displaceMatrix \uv_k}{t_k} + \frac{\ev_i^{\T} \displaceMatrix^2 \uv_k}{t_k^2} - \frac{3}{2} \cdot \probEigenvectors_{ik} \frac{\uv_k^{\T} \EE \displaceMatrix^2 \uv_k}{t_k^2} \\
    &  \quad \, + \frac{1}{t_k^2}\sum_{k' \in [\nclusters] \setminus \{k\}} \frac{\lambda_{k'} \probEigenvectors_{i k'}}{\lambda_{k'} - t_k} \cdot  \uv_{k'}^\T \EE \displaceMatrix^2 \uv_k + O_{\prec} \left(\sqrt{\frac{1}{\nsize^3 \sparsityParam}} \right).
  \end{align*}
\end{lemma}

\begin{proof}
  For further derivations, we need to introduce some notations. All necessary variables are defined in Table~\ref{tab: expansion notation}. Then, we define $t_k$ as a solution of 
  \begin{align}
  \label{eq: t_k definition}
    1 + \lambda_k(\probMatrix) 
    \left\{ 
      \resolvent(\uv_k, \uv_k, z) 
      - 
      \resolvent(\uv_k, \probEigenvectors_{-k}, z) 
      [
        \probEigenvalues^{-1}_{-k} + \resolvent(\probEigenvectors_{-k}, \probEigenvectors_{-k}, z)
      ]^{-1}
      \resolvent(\probEigenvectors_{-k}, \uv_k, z)
    \right\} = 0
  \end{align}
  on the closed interval $[a_k, b_k]$, where
  \begin{align*}
    a_k
    = 
    \begin{cases}
      \lambda_k(\probMatrix)/(1 + 2^{-1}c_0), & \lambda_k(\probMatrix) > 0, \\
      (1 + 2^{-1} c_0) \lambda_k(\probMatrix), & \lambda_k(\probMatrix) < 0,
      \end{cases}
      \text{ and }
      b_k = \begin{cases}
      (1 + 2^{-1} c_0) \lambda_k(\probMatrix), & \lambda_k(\probMatrix) > 0, \\
      \lambda_k(\probMatrix) / (1 + 2^{-1} c_0), & \lambda_k(\probMatrix) < 0, 
    \end{cases}
  \end{align*}
  and $c_0$ is defined in Condition~\ref{cond: eigenvalues divergency}.

  Throughout this proof, a lot of auxiliary variables appear. For them, we exploit asymptotics established in Lemma~\ref{lemma: auxiliary variables expansion}. Lemma~\ref{lemma: log estimate vector difference} guarantees that $\xv^{\T} \displaceMatrix \yv = O_\ell(\sqrt{\sparsityParam \log \nsize})$ whenever unit $\xv$ or $\yv$ is $\uv_k$ because of Condition~\ref{cond: sparsity param bound} ($\sparsityParam \gg \nsize^{-1/3}$) and Lemma~\ref{lemma: eigenvectors max norm} ($\Vert \uv_k \Vert_{\infty} = O(\nsize^{-1/2}$)). Thus, any term of the form $\vv^{\T} \displaceMatrix \uv_k$ becomes
  \begin{align*}
    \vv^{\T} \displaceMatrix \uv_k = O_\ell(\sqrt{\sparsityParam \log \nsize}) \cdot \Vert \vv \Vert_2.
  \end{align*}
  First, from Lemma~\ref{lemma: fan asymptotic expansion w/o sigma},
  \begin{align*}
    \uv_k^{\T} \estimator[\uv]_k \estimator[\uv]_k^{\T} \uv_k & = \meanFactor_{\uv_k, k, t_k} \meanFactor_{\uv_k, k, t_k} \pFactor_{k, t_k} +  \tr \left[\displaceMatrix \jMatrix_{\uv_k, \uv_k, k, t_k} - (\displaceMatrix^2 - \EE \displaceMatrix^2) \lMatrix_{\uv_k, \uv_k, k, t_k} \right] \\ 
    & + \tr(\displaceMatrix \uv_k \uv_k^{\T}) \tr (\displaceMatrix \qMatrix_{\uv_k, \uv_k, k, t_k})
    + O_{\prec} \left(\frac{1}{\nsize^2 \sparsityParam^2}\right).
  \end{align*}
  Notice, that $\jMatrix_{\uv_k, \uv_k, k, t_k} = \uv_k \vv^{\T}_{\jMatrix}$ for
  \begin{align*}
    \vv_{\jMatrix}^{\T} & = -2 \meanFactor_{\uv_k, k, t_k} \pFactor_{k, t_k} t_k^{-1} \left(\bv_{\uv_k, k, t_k}^{\T} + \meanFactor_{\uv_k, k, t_k} \pFactor_{k, t_k} \uv_k^{\T} \right) \\
    & = - 2 \left[-1 - \frac{\uv_k^{\T} \EE \displaceMatrix^2 \uv_k}{t_k^2} + O(t_k^{-3/2})\right] 
    \times 
    \left[1 - \frac{3}{t_k^2} \uv_k^{\T} \EE \displaceMatrix^2 \uv_k + O(t_k^{-3/2})\right] t_k^{-1} \times \\
    & \quad \, \times 
    \bigg [ \uv_k + O(t_k^{-1}) + \left( 
      -1 - \frac{\uv_k^{\T} \EE \displaceMatrix^2 \uv_k}{t_k^2} + O(t_k^{-3/2}) \times
    \right) \\
    & \quad \quad \times\left( 1 - \frac{3}{t_k^2} \uv_k^{\T} \EE \displaceMatrix^2 \uv_k + O(t_k^{-3/2}) \right) \uv_k \bigg]^{\T} \\
    & = O(t_k^{-2}),
  \end{align*}
  where we use Lemma~\ref{lemma: power expectation} for estimation of $\uv_k^{\T} \EE \displaceMatrix^2 \uv_k$ and Lemma~\ref{lemma: auxiliary variables expansion} for asymptotic behaviour of the auxiliary variables. Consequently,
  \begin{align*}
    \tr (\displaceMatrix \jMatrix_{\uv_k, \uv_k, k, t_k}) = O_\ell\left(\frac{\sqrt{\sparsityParam \log \nsize}}{\nsize^2 \sparsityParam^2}\right)
  \end{align*}
  because $t_k = \Theta\bigl(\lambda_k(\probMatrix)\bigr)$ due to Lemma~\ref{lemma: t_k is well-definied} and $\lambda_k(\probMatrix) = \Theta(\nsize \sparsityParam)$ due to Lemma~\ref{lemma: eigenvalues asymptotics}.

  Next, consider $\lMatrix_{\uv_k, \uv_k, k, t_k}$ which is also can represented as $\uv_k \vv_{\lMatrix}^{\T}$, where
  \begin{align*}
    \vv_{\lMatrix} & = \pFactor_{k, t_k} t_k^{-2} 
    \big( 
      (3 \meanFactor^2_{\uv_k, k, t_k} + 2 \meanFactor_{\uv_k, k, t_k}) \uv_k \\
      & \quad + 
      2 \meanFactor_{\uv_k, k, t_k} \probEigenvectors_{-k} [\probEigenvalues^{-1}_{-k} + \resolvent(\probEigenvectors_{-k}, \probEigenvectors_{-k}, t_k)]^{-1} \resolvent(\uv_k, \probEigenvectors_{-k}, t_k)^{\T}
    \big).
  \end{align*}
  According to Lemma~\ref{lemma: auxiliary variables expansion}, we have
  \begin{align*}
    &\quad \left \Vert 2 \meanFactor_{\uv_k, k, t_k} \probEigenvectors_{-k} [\probEigenvalues^{-1}_{-k} + \resolvent(\probEigenvectors_{-k}, \probEigenvectors_{-k}, t_k)]^{-1} \resolvent(\uv_k, \probEigenvectors_{-k}, t_k)^{\T} \right \Vert \\
    & = O(1) \cdot \bigl\Vert [\probEigenvalues^{-1}_{-k} + \resolvent(\probEigenvectors_{-k}, \probEigenvectors_{-k}, t_k)]^{-1} \bigr\Vert \times t_k^{-3} \bigl\Vert \uv_k^{\T} \EE \displaceMatrix^2 \probEigenvectors_{-k} \bigr\Vert
    = O(t_k^{-1}),
  \end{align*}
  and, consequently, 
  \begin{align*}
    \vv_{\lMatrix} = \pFactor_{k, t_k} t_k^{-2} (3 \meanFactor_{\uv_k, k, t_k}^2 + 2 \meanFactor_{\uv_k, k, t_k}) \uv_k + O\bigl(t_k^{-3}\bigr).
  \end{align*}
  While $3 \meanFactor_{\uv_k, k, t_k}^2 + 2 \meanFactor_{\uv_k, k, t_k} = 3 + \frac{6 \uv_k^{\T} \EE \displaceMatrix^2 \uv_k}{t_k^2} + O(t_k^{-3/2}) - 2 - \frac{2 \uv_k^{\T} \EE \displaceMatrix \uv_k}{t_k^2} + O(t_k^{-3/2}) = 1 + 4 \cdot \frac{\uv_k^{\T} \EE \displaceMatrix^2 \uv_k}{t_k^2} + O(t_k^{-3/2})$, and, hence,
  \begin{align*}
    \vv_{\lMatrix} = \left(1 + \frac{\uv_k^{\T} \EE \displaceMatrix^2 \uv_k}{t_k^2} \right) \cdot t_k^{-2} \uv_k + O(t_k^{-7/2}) = t_k^{-2} \uv_k + O(t_k^{-3}).
  \end{align*}
  That implies
  \begin{align*}
    \tr\bigl[(\displaceMatrix^2 - \EE \displaceMatrix^2) \lMatrix_{\uv_k, \uv_k, k, t_k}\bigr] & = \frac{\uv_k^{\T} (\displaceMatrix^2 - \EE \displaceMatrix^2) \uv_k}{t_k^2} + O (t_k^{-3}) \cdot O_{\prec} (t_k^{1/2}) \\
    & = \frac{\uv_k^{\T} (\displaceMatrix^2 - \EE \displaceMatrix^2) \uv_k}{t_k^2} + O_{\prec}(t_k^{-5/2}),
  \end{align*}
  where Lemma~\ref{lemma: power deviation} was used.

  Next, representing $\qMatrix_{\uv_k, \uv_k, k, t_k}$ as $\uv_k \vv_{\qMatrix}$ with
  \begin{align*}
    \vv_{\qMatrix} = \vv_{\lMatrix} - \pFactor_{k, t_k} t_k^{-2} \meanFactor_{\uv_k, k, t_k}^2 \uv_k + 4 \pFactor^2 t_k^{-2} \meanFactor_{\uv_k, k, t_k} \bv_{\uv_k, k, t_k} = O(t_k^{-2}),
  \end{align*}
  we obtain
  \begin{align*}
    \tr(\displaceMatrix \uv_k \uv_k^{\T}) \tr(\displaceMatrix \qMatrix_{\uv_k, \uv_k, k, t_k}) = O_\ell(\sqrt{\sparsityParam \log \nsize}) \cdot O_\ell(\sqrt{\sparsityParam \log \nsize}) \cdot O(t_k^{-2}) = O_\ell(\sparsityParam \cdot t_k^{-2} \log \nsize).
  \end{align*}
  Finally, obtained via Lemma~\ref{lemma: auxiliary variables expansion}, the decomposition
  \begin{align*}
    \pFactor_{k, t_k} \meanFactor_{\uv_k, k, t_k}^2 = 1 - \frac{\uv_k^{\T} \EE \displaceMatrix^2 \uv_k}{t_k^2} + O\bigl(t_k^{-3/2}\bigr).
  \end{align*}
  provides us with expansion
  \begin{align}
    \uv_k^{\T} \estimator[\uv]_k \estimator[\uv]_k^{\T} \uv_k = 1 - \frac{\uv_k^{\T} \displaceMatrix^2 \uv_k}{t_k^2} + O_\ell\bigl(t_k^{-3/2}\bigr), \nonumber \\
    \label{eq: eigenvector angle}
    \langle \uv_k, \estimator[\uv]_k \rangle = 1 - \frac{\uv_k^{\T} \displaceMatrix^2 \uv_k}{2 t_k^2} + O_\ell\bigl(t_k^{-3/2}\bigr).
  \end{align}

  Now, we should estimate 
  \begin{align*}
    \ev_i^{\T} \estimator[\uv]_k \estimator[\uv]_k^{\T} \uv_k = & 
    \meanFactor_{\ev_i, k, t_k} \meanFactor_{\uv_k, k, t_k} \pFactor_{k, t_k} 
    + \tr \left[ 
        \displaceMatrix \jMatrix_{\ev_i, \uv_k, k, t_k}
        - 
        (\displaceMatrix^2 - \EE \displaceMatrix^2) \lMatrix_{\ev_i, \uv_k, k, t_k}
    \right] \\
    & + \tr (\displaceMatrix \uv_k \uv_k^{\T})
    \tr (\displaceMatrix \qMatrix_{\ev_i, \uv_k, k, t_k}) + O_{\prec} \left(\frac{1}{\nsize^2 \sparsityParam^2}\right),
  \end{align*}
  obtained from Lemma~\ref{lemma: fan asymptotic expansion w/o sigma}.
  For a reminder
  \begin{align*}
    \jMatrix_{\ev_i, \uv_k, k, t_k} & = - \pFactor_{k, t_k} t_k^{-1} \uv_k \left(
        \meanFactor_{\ev_i, k, t_k} \bv_{\uv_k, k, t_k}^{\T} + 
        \meanFactor_{\uv_k, k, t_k} \bv_{\ev_i, k, t_k}^{\T} +
        2 \meanFactor_{\uv_k, k, t_k} \meanFactor_{\ev_i, k, t_k} \pFactor_{k, t_k} \uv_k^{\T}
    \right),
    \\
    \lMatrix_{\ev_i, \uv_k, k, t_k} & = \pFactor_{k, t_k} t_k^{-2} \uv_k \bigg\{
        \bigl[
          \meanFactor_{\uv_k, k, t_k} 
          \resolvent(\ev_i, \probEigenvectors_{-k}, t_k) 
          + 
          \meanFactor_{\ev_i, k, t_k} \resolvent(\uv_k, \probEigenvectors_{-k}, t_k)
        \bigr]
        \\
        & \quad \, \times \left[
          \probEigenvalues_{-k}^{-1}
          +
          \resolvent(\probEigenvectors_{-k}, \probEigenvectors_{-k}, t_k)
        \right]^{-1} \probEigenvectors_{-k}
        + \meanFactor_{\ev_i, k, t_k} \uv_k^{\T} \\
        & \quad \, + \meanFactor_{\uv_k, k, t_k} \ev_i^{\T} + 3 \meanFactor_{\ev_i, k, t_k} \meanFactor_{\uv_k, k, t_k} \uv_k^{\T}
    \bigg\},
    \\
    \qMatrix_{\ev_i, \uv_k, k, t_k} & = \lMatrix_{\ev_i, \uv_k, k, t_k} - \pFactor_{k, t_k} t_k^{-2} \meanFactor_{\ev_i, k, t_k} \meanFactor_{\uv_k, k, t_k} \uv_k \uv_k^{\T} \\
    & \quad \, + 2 \pFactor_{k, t_k}^2 t_k^{-2} \uv_k \left(\meanFactor_{\ev_i, k, t_k} \bv_{\ev_i, k, t_k}^{\T} + \meanFactor_{\uv_k, k, t_k} \bv_{\uv_k, k, t_k}^{\T} \right).
  \end{align*}
    
  Applying asymptotic expansions from Lemma~\ref{lemma: auxiliary variables expansion}, we obtain
  \begin{align*}
    \meanFactor_{\uv_k, k, t_k} \bv_{\ev_i, k, t_k}^{\T} &= \left( - 1 - \frac{\uv_k^{\T} \EE \displaceMatrix^2 \uv_k}{t_k^2} + O(t_k^{-3/2}) \right) \times \left(\ev_i + O(\nsize^{-1/2}) \right) \\
    & = - \ev_i + O(\nsize^{-1/2}), \\
    \meanFactor_{\ev_i, k, t_k} \bv_{\uv_k, k, t_k}^{\T} & = \left(- \probEigenvectors_{ik} + O(t_k^{-1} / \sqrt{\nsize}) \right) \times \left( \uv_k + O(t_k^{-1}) \right) \\
    & = - \probEigenvectors_{ik} \uv_k + O(t_k^{-1} / \sqrt{\nsize}) = O(\nsize^{-1/2}), \\
    2 \meanFactor_{\uv_k, k, t_k} \meanFactor_{\ev_i, k, t_k} \pFactor_{k, t_k} \uv_k^{\T} & = O(\nsize^{-1/2}).
  \end{align*}
  Using the same notation as previously, we observe
  \begin{align*}
    \vv_{\jMatrix} & = t_k^{-1} \ev_i + O(t_k^{-1} \nsize^{-1/2}), \\
    \tr(\displaceMatrix \jMatrix_{\ev_i, \uv_k, k, t_k}) & = \frac{\ev_i^{\T} \displaceMatrix \uv_k}{t_k} + O_\ell\left(\sqrt{\frac{\log \nsize}{\nsize^3 \sparsityParam}} \right).
  \end{align*}
  
  To estimate $\tr\bigl[(\displaceMatrix^2 - \EE \displaceMatrix^2) \lMatrix_{\ev_i, \uv_k, k, t_k}\bigr]$, we obtain
  \begin{align*}
    & \bigl[
      \meanFactor_{\uv_k, k, t_k} 
      \resolvent(\ev_i, \probEigenvectors_{-k}, t_k) 
      + 
      \meanFactor_{\ev_i, k, t_k} \resolvent(\uv_k, \probEigenvectors_{-k}, t_k)
    \bigr] 
    \left[
      \probEigenvalues_{-k}^{-1}
      +
      \resolvent(\probEigenvectors_{-k}, \probEigenvectors_{-k}, t_k)
        \right]^{-1} \probEigenvectors_{-k}
     \\
    & = \bigg [  (-1 + O(t_k^{-1})) \left(-\frac{1}{t_k}\ev_i^\T \probEigenvectors_{-k} + O(t_k^{-2}/\sqrt{\nsize}) \right) \\
    & \qquad + (-\probEigenvectors_{ik} + O(t_k^{-1} /\sqrt{\nsize})) (- t_k^{-3} \uv_k^\T \EE \displaceMatrix^2 \probEigenvectors_{-k}  + O(t_k^{-5/2}) \bigg ] \\
    & \quad \times \left( \diag \left( \frac{\lambda_{k'} t_k}{t_k - \lambda_{k'}}\right)_{k' \in [\nclusters] \setminus \{k\}} + O(1) \right) \probEigenvectors_{-k} 
    = \sum_{k' \in [\nclusters] \setminus \{k\}} \frac{\lambda_{k'}}{t_k - \lambda_{k'}} \probEigenvectors_{i k'} \uv_{k'}^\T + O(t_k^{-2}),
  \end{align*}
  where we use Lemma~\ref{lemma: auxiliary variables expansion} and $\uv_k^\T \EE \displaceMatrix^2 \uv_{k'} = O(t_k), k' \in [\nclusters]$, $\ev_i^\T \EE \displaceMatrix^2 \uv_k = O(t_k / \sqrt{\nsize})$ from Lemma~\ref{lemma: power expectation}. 
  Consequently, we have
  \begin{align*}
    \vv_{\lMatrix} & = -\frac{\ev_i}{t_k^2} + t_k^{-2} \sum_{k' \in [\nclusters] \setminus \{k\}} \frac{\lambda_{k'}}{t_k - \lambda_{k'}} \probEigenvectors_{i k'} \uv_{k'} \\
    & \quad - t_k^{-2} ( \probEigenvectors_{i k} + O(t_k^{-1} / \sqrt{\nsize})) \uv_k + 3 t_k^{-2} ( \probEigenvectors_{ik} + O(t_k^{-1} / \sqrt{\nsize})) \uv_k + O(t_k^{-2}/ \sqrt{\nsize}) \\
    & = -\frac{\ev_i}{t_k^2} + t_k^{-2} \sum_{k' \in [\nclusters] \setminus \{k\}} \frac{\lambda_{k'}}{t_k - \lambda_{k'}} \probEigenvectors_{i k'} + \frac{\uv_k^\T \EE \displaceMatrix^2 \uv_{k'}}{t_k^2} \uv_{k'} + 2 t_k^{-2} \probEigenvectors_{ik} \uv_k + O(t_k^{-2}/ \sqrt{\nsize}).
  \end{align*}
  Thus, we get
  \begin{align*}
    & \tr\bigl[(\displaceMatrix^2 - \EE \displaceMatrix^2) \lMatrix_{\ev_i, \uv_k, k, t_k}\bigr] = \vv_{\lMatrix}^\T (\displaceMatrix^2 - \EE \displaceMatrix^2) \uv_k \\
    & \overset{\text{Lemma~\ref{lemma: power deviation}}}{=} - \ev_i^\T (\displaceMatrix^2 - \EE \displaceMatrix^2) \uv_k
    + t_k^{-2} \sum_{k' \in [\nclusters] \setminus \{k\} } \frac{\lambda_{k'}}{t_k - \lambda_{k'}} \probEigenvectors_{ik'} \cdot \uv_{k'} (\displaceMatrix^2 - \EE \displaceMatrix^2) \uv_k \\
    & \quad + 2 t_k^{-2} \probEigenvectors_{ik} \cdot \uv_k^\T (\displaceMatrix^2 - \EE \displaceMatrix^2) \uv_k + O(t_k^{-2} / \sqrt{\nsize}) \cdot O_{\prec}(t_k / \sqrt{\nsize}) \\
    & \overset{\text{Lemma~\ref{lemma: power deviation}}}{=} -\frac{1}{t_k^2} \ev_i^{\T} (\displaceMatrix^2 - \EE \displaceMatrix^2) \uv_k
    + O(t_k^{-2} / \sqrt{\nsize}) \cdot O_\prec(\sparsityParam \sqrt{\nsize} ) + O_\prec \left(t_k^{-1} / \nsize\right) \\
    & = -\frac{1}{t_k^2} \ev_i^{\T} (\displaceMatrix^2 - \EE \displaceMatrix^2) \uv_k + O_{\ell} \left( \frac{\log \nsize}{\nsize^2 \sparsityParam}\right) + O_{\prec} \left(\frac{1}{\nsize^2 \sparsityParam} \right).
  \end{align*}
  Finally, we obtain
  \begin{align*}
    \vv_{\qMatrix} = \vv_{\lMatrix} + O(t_k^{-2}),
  \end{align*}
  and
  \begin{align*}
    \tr(\displaceMatrix \uv_k \uv_k^{\T}) \tr(\displaceMatrix \qMatrix_{\ev_i, \uv_k, k, t_k}) = O_\ell(\sqrt{\sparsityParam \log \nsize}) \cdot O_\ell(\sqrt{\sparsityParam \log \nsize}) O(t_k^{-2}) = O_\ell\left(
      \frac{\log \nsize}{\nsize^2 \sparsityParam}
    \right).
  \end{align*}
  Approximating $\pFactor_{k, t_k} \meanFactor_{\ev_i, k, t_k} \meanFactor_{\uv_k, k, t_k}$ with
  \begin{align*}
    \pFactor_{k, t_k} \meanFactor_{\ev_i, k, t_k} \meanFactor_{\uv_k, k, t_k} & =
    \left(1 - \frac{3}{t_k^2} \uv_k^{\T} \EE \displaceMatrix^2 \uv_k + O(t_k^{-3/2}) \right) \left(1 + \frac{\uv_k^{\T} \EE \displaceMatrix^2 \uv_k}{t_k^2} + O(t_k^{-3/2}) \right) \times \\
    & \quad \, \times \left(\probEigenvectors_{i k} + \frac{\ev_i^{\T} \EE \displaceMatrix^2 \uv_k}{t_k^2} + \sum_{k' \in [\nclusters] \setminus \{k\}} \frac{\lambda_{k'} \probEigenvectors_{i k'}}{\lambda_{k'} - t_k} \cdot \frac{\uv_{k'}^\T \EE \displaceMatrix^2 \uv_k}{t_k^2} + O(t_k^{-5/2}) \right) \\
    & = \probEigenvectors_{ik} - \frac{2}{t_k^2} \probEigenvectors_{ik} \uv_k^{\T} \EE \displaceMatrix^2 \uv_k + \frac{1}{t_k^2} \ev_i^{\T} \EE \displaceMatrix^2 \uv_k \\
    & \quad \, + \frac{1}{t_k^2}\sum_{k' \in [\nclusters] \setminus \{k\}} \frac{\lambda_{k'} \probEigenvectors_{i k'}}{\lambda_{k'} - t_k} \cdot \uv_{k'}^\T \EE \displaceMatrix^2 \uv_k + O(t_k^{-3/2} \nsize^{-1/2}),
  \end{align*}
  we obtain 
  \begin{align}
  \notag
    \langle \ev_i, \estimator[\uv]_k \rangle \langle \estimator[\uv]_k, \uv_k \rangle & = \probEigenvectors_{i k} + \frac{\ev_i^{\T} \displaceMatrix \uv_k}{t_k} + \frac{\ev_i^{\T} \displaceMatrix^2 \uv_k}{t_k^2} - \frac{2}{t_k^2} \probEigenvectors_{i k}\uv_k^{\T} \EE \displaceMatrix^2 \uv_k \\
    & \quad \, + \frac{1}{t_k^2}\sum_{k' \in [\nclusters] \setminus \{k\}} \frac{\lambda_{k'} \probEigenvectors_{i k'}}{\lambda_{k'} - t_k} \cdot \uv_{k'}^\T \EE \displaceMatrix^2 \uv_k + O_\ell\left( \sqrt{\frac{\log \nsize}{\nsize^3 \sparsityParam}} \right).
  \label{eq: second vector decomposition}
  \end{align}
  Here we use Condition~\ref{cond: sparsity param bound} to ensure that the reminder $O_{\prec} \left ( \frac{1}{\nsize^2 \sparsityParam^2} \right )$ provided by Lemma~\ref{lemma: fan asymptotic expansion w/o sigma} is less than $O_{\ell}\left ( \sqrt{\frac{\log \nsize}{\nsize^3 \sparsityParam}}\right )$. Dividing~\eqref{eq: second vector decomposition} by~\eqref{eq: eigenvector angle} results in:
  \begin{align*}
    \adjacencyEigenvectors_{ik} & = \probEigenvectors_{ik} + \frac{\ev_i^{\T} \displaceMatrix \uv_k}{t_k} + \frac{\ev_i^{\T} \displaceMatrix^2 \uv_k}{t_k^2} - 2 \probEigenvectors_{ik} \frac{\uv_k^{\T} \EE \displaceMatrix^2 \uv_k}{t_k^2} + \frac{1}{2} \probEigenvectors_{ik} \frac{\uv_k^{\T} \displaceMatrix^2 \uv_k}{t_k^2} \\
    & \quad \, + \frac{1}{t_k^2}\sum_{k' \in [\nclusters] \setminus \{k\}} \frac{\lambda_{k'} \probEigenvectors_{i k'}}{\lambda_{k'} - t_k} \cdot \uv_{k'}^\T \EE \displaceMatrix^2 \uv_k+ O_{\prec} \left(\frac{1}{\nsize\sqrt{\nsize \sparsityParam}} \right)
  \end{align*}
  due to Lemma~\ref{lemma: power deviation}. Additionally, this lemma guarantees that
  \begin{align*}
    \probEigenvectors_{ik} \frac{\uv_k^{\T} \displaceMatrix^2 \uv_k}{t_k^2} - \probEigenvectors_{ik} \frac{\uv_k^{\T} \EE \displaceMatrix^2 \uv_k}{t_k^2} = \probEigenvectors_{ik} \cdot O_\prec\left( \frac{\sparsityParam \sqrt{\nsize}}{\nsize^2 \sparsityParam^2}\right) = O_\prec\left(\frac{1}{\nsize^2 \sparsityParam} \right).
  \end{align*}
  This leads us to the statement of the lemma.
\end{proof}

\subsection{Debiasing eigenvectors}
\begin{lemma}
  \label{lemma: eigenvector debiasing}
      Define
    \begin{align*}
      \diagAdjecencyMatrix & = \operatorname{diag}\left(\sum_{t = 1}^\nsize \adjacencyMatrix_{it} \right)_{i = 1}^\nsize, \\
      \tilde \probEigenvectors_{ik} & = \adjacencyEigenvectors_{ik} \left(1 - \frac{\diagAdjecencyMatrix_{ii} - 3/2 \sum_{j = 1}^\nsize \diagAdjecencyMatrix_{jj} \adjacencyEigenvectors_{jk}^2}{\adjacencyEigenvalues^2_{kk}} \right) - \sum_{k' \in [\nclusters] \setminus \{k\}} \frac{\debiasedEigenvalues_{k' k'} \cdot \adjacencyEigenvectors_{i k'}}{\debiasedEigenvalues_{k' k'} - \adjacencyEigenvalues_{k k}} \cdot \sum_{j = 1}^{\nsize} \frac{\diagAdjecencyMatrix_{j j} \adjacencyEigenvectors_{j k'} \adjacencyEigenvectors_{j k}}{\adjacencyEigenvalues_{k k}^2}.
      \end{align*}
      Then, under Conditions~\ref{cond: nonzero B elements}-\ref{cond: theta distribution-b}, the following holds:
      \begin{align*}
        \tilde \probEigenvectors_i = \probEigenvectors_i + \ev_i^\T \displaceMatrix \probEigenvectors \meanEigs^{-1} + \ev_i^\T (\displaceMatrix^2 - \EE \displaceMatrix^2) \probEigenvectors \meanEigs^{-2} + O_{\ell} \left( \sqrt{\frac{\log \nsize}{\nsize^3 \sparsityParam}} \right),
    \end{align*}
    where $\meanEigs = \diag(t_k)_{k = 1}^\nclusters$.
  \end{lemma}
  
  \begin{proof}
  Due to Lemma~\ref{lemma: eigenvector power expansion}, we have
    \begin{align*}
      \adjacencyEigenvectors_{ik} & = \probEigenvectors_{ik} + \frac{\ev_i^{\T} \displaceMatrix \uv_k}{t_k} + \frac{\ev_i^{\T} \displaceMatrix^2 \uv_k}{t_k^2} - \frac{3}{2} \cdot \probEigenvectors_{ik} \frac{\uv_k^{\T} \EE \displaceMatrix^2 \uv_k}{t_k^2} \\
      &  \quad \, + \frac{1}{t_k^2}\sum_{k' \in [\nclusters] \setminus \{k\}} \frac{\lambda_{k'} \probEigenvectors_{i k'}}{\lambda_{k'} - t_k} \cdot  \uv_{k'}^\T \EE \displaceMatrix^2 \uv_k + O_{\ell} \left(\sqrt{\frac{\log \nsize}{\nsize^3 \sparsityParam}} \right),
    \end{align*}
    Our goal is to get asymptotic expansion for $\tilde \probEigenvectors_j$. For the terms of asymptotic expansion of $\adjacencyEigenvectors_j$, we obtain
    \begin{align}
        \frac{\ev_i^{\T} \displaceMatrix \uv_k}{t_k} & \overset{\text{Lemma~\ref{lemma: log estimate vector difference}}}{=} \frac{1}{t_k} O_{\ell}(\sqrt{\sparsityParam \log \nsize}) \overset{\text{Lemmas~\ref{lemma: t_k is well-definied},\ref{lemma: eigenvalues asymptotics}}}{=} O_{\ell}\left(\sqrt{\frac{\log \nsize}{\nsize^2 \sparsityParam}} \right), \label{eq: U hat terms asumptotics-f} \\
        \frac{\ev_i^\T \displaceMatrix^2 \uv_k}{t_k^2} & \overset{\text{Lemmas~\ref{lemma: power expectation}, \ref{lemma: power deviation}}}{=} \frac{1}{t_k^2} O_{\prec}(\sqrt{\nsize} \sparsityParam) \overset{\text{Lemmas~\ref{lemma: t_k is well-definied},\ref{lemma: eigenvalues asymptotics}}}{=} O_{\prec} \left( \frac{1}{\nsize^{3/2} \sparsityParam} \right), \\
        \frac{3}{2} \probEigenvectors_{ik} \cdot \frac{\uv_k ^\T \EE \displaceMatrix^2 \uv_k}{t_k^2} 
        & \overset{\text{Lemmas~\ref{lemma: power expectation},\ref{lemma: eigenvectors max norm}}}{=} O(\nsize^{-1/2}) \cdot O(\nsize \sparsityParam) \cdot t_k^{-2} \nonumber \\
        & \overset{\text{Lemmas~\ref{lemma: t_k is well-definied},\ref{lemma: eigenvalues asymptotics}}}{=} O \left(\frac{1}{\nsize^{3/2} \sparsityParam} \right),
    \end{align}
    \begin{align}
        \sum_{k' \in [\nclusters] \setminus \{k\} } \frac{\lambda_{k'} \probEigenvectors_{ik'}}{\lambda_{k'} - t_k} \cdot \frac{\uv_{k'}^\T \EE \displaceMatrix^2 \uv_k}{t_k^2} & \overset{
          \substack{\text{Lemma~\ref{lemma: t_k is well-definied}}, \\ \text{Condition~\ref{cond: eigenvalues divergency}}}
        }{=} O(1) \cdot \max_{k' \in [\nclusters] \setminus\{k\} } \probEigenvectors_{i k'} \cdot \frac{\uv_{k'}^\T \EE \displaceMatrix^2 \uv_k}{t_k^2} \nonumber \\
        & \overset{\text{Lemmas~\ref{lemma: eigenvectors max norm},\ref{lemma: power expectation}}}{=} t_k^{-2} O(\nsize^{1/2} \sparsityParam) \nonumber \\
        &
        \overset{\text{Lemmas~\ref{lemma: t_k is well-definied},\ref{lemma: eigenvalues asymptotics}}}{=} O \left( \frac{1}{\nsize^{3/2} \sparsityParam} \right).
    \label{eq: U hat asymptotics-l}
    \end{align}
    Next, we analyze $\tilde \probEigenvectors_{jk}$. Note that
    \begin{align*}
        \diagAdjecencyMatrix_{ii} - \EE \diagAdjecencyMatrix_{ii} = \sum_{j = 1}^\nsize (\adjacencyMatrix_{ij} - \probMatrix_{ij}) = O_\ell(\sqrt{\nsize \sparsityParam \log \nsize})
    \end{align*}
    from the Bernstein inequality. Thus, we get
    \begin{align*}
        \diagAdjecencyMatrix_{ii} \adjacencyEigenvalues^{-2}_{kk} - t_k^{-2} \EE (\diagAdjecencyMatrix_{ii}) & = (\diagAdjecencyMatrix_{ii} - \EE \diagAdjecencyMatrix_{ii}) \adjacencyEigenvalues^{-2}_{kk} + \EE (\diagAdjecencyMatrix_{ii}) (\adjacencyEigenvalues_{kk}^{-2} - t_k^{-2}) \\
        & = O_{\ell}(\sqrt{\nsize \sparsityParam \log \nsize}) \adjacencyEigenvalues^{-2}_{kk} + O(\nsize \sparsityParam) (\adjacencyEigenvalues_{kk}^{-2} - t_k^{-2}).
    \end{align*}
    Since $t_k \sim \lambda_k$ from Lemma~\ref{lemma: t_k is well-definied}, $\lambda_k = \Theta(\nsize \sparsityParam)$ from Lemma~\ref{lemma: eigenvalues asymptotics} and $\adjacencyEigenvalues_{kk} = t_k + O(\sqrt{\sparsityParam \log \nsize})$ from Lemmas~\ref{lemma: eigenvalues difference} and~\ref{lemma: log estimate vector difference}, we have $\adjacencyEigenvalues^{-2}_{kk} = O_{\ell} \left( \frac{1}{\nsize^2 \sparsityParam^2} \right)$ and $\adjacencyEigenvalues^{-2}_{k k} - t_k^{-2} = O_{\ell}(\sqrt{\sparsityParam \log \nsize}) \cdot O_\ell(\nsize^{-3} \sparsityParam^{-3}) =  O_{\ell}(\nsize^{-3} \sparsityParam^{-5/2} \log^{1/2} \nsize)$. Consequently, we have
    \begin{align}
    \label{eq: averaging lemma, D term approximation}
        \diagAdjecencyMatrix_{ii} \adjacencyEigenvalues^{-2}_{kk} - t_k^{-2} \EE (\diagAdjecencyMatrix_{ii}) = O_{\ell}\left( \sqrt{\frac{\log \nsize}{\nsize^3 \sparsityParam^3}}\right).
    \end{align}
    Next, we bound $\adjacencyEigenvalues^{-2}_{kk} \sum_{j = 1}^\nsize \diagAdjecencyMatrix_{jj} \adjacencyEigenvectors_{jk}^2$. We have
    \begin{align}
    \label{eq: averaging lemma, quad form factor approximation}
        \adjacencyEigenvalues^{-2}_{kk} \sum_{j = 1}^\nsize \diagAdjecencyMatrix_{jj} \adjacencyEigenvectors_{jk}^2 = \frac{\uv_k^\T \EE \diagAdjecencyMatrix \uv_k}{t_k^2} +  \left( \frac{\widehat{\uv}_k^\T \diagAdjecencyMatrix \widehat{\uv}_k}{\adjacencyEigenvalues^2_{kk}} - \frac{\uv_k^\T \EE \diagAdjecencyMatrix \uv_k}{t_k^2} \right) = \frac{\uv_k^\T \EE \diagAdjecencyMatrix \uv_k}{t_k^2} + O_{\ell} \left( \sqrt{\frac{\log \nsize}{\nsize^3 \sparsityParam^3}}\right)
    \end{align}
    due to Lemma~\ref{lemma: second order term estimation}.

    At the same time, given $k'$, we have
    \begin{align*}
      & \frac{1}{t_k^2} \frac{\lambda_{k'} \probEigenvectors_{i k'}}{\lambda_{k'} - t_k} \cdot \uv_{k'}^\T \EE \diagAdjecencyMatrix \uv_k 
      -  \frac{\debiasedEigenvalues_{k' k'} \cdot \adjacencyEigenvectors_{i k'}}{\debiasedEigenvalues_{k' k'} - \adjacencyEigenvalues_{k k}} \cdot \sum_{j = 1}^{\nsize} \frac{\diagAdjecencyMatrix_{j j} \adjacencyEigenvectors_{j k'} \adjacencyEigenvectors_{j k}}{\adjacencyEigenvalues_{k k}^2} \\
      & = \left(\frac{\lambda_{k'} \probEigenvectors_{i k'}}{\lambda_{k'} - t_k} - \frac{\debiasedEigenvalues_{k' k'} \cdot \adjacencyEigenvectors_{i k'}}{\debiasedEigenvalues_{k' k'} - \adjacencyEigenvalues_{k k}}\right)\cdot \frac{\uv_{k'}^\T \EE \diagAdjecencyMatrix \uv_k}{t_k^2}
      + \frac{\debiasedEigenvalues_{k' k'} \cdot \adjacencyEigenvectors_{i k'}}{\debiasedEigenvalues_{k' k'} - \adjacencyEigenvalues_{k k}} \left(\frac{\uv_{k'}^\T \EE \diagAdjecencyMatrix \uv_k}{t_k^2} - \frac{\widehat{\uv}_{k'}^\T \diagAdjecencyMatrix \widehat{\uv}_k}{\adjacencyEigenvalues_{k k}}\right).
    \end{align*}
    From Lemma~\ref{lemma: second order term estimation}, we get
    \begin{align*}
      \frac{\debiasedEigenvalues_{k' k'} \cdot \adjacencyEigenvectors_{i k'}}{\debiasedEigenvalues_{k' k'} - \adjacencyEigenvalues_{k k}} \left(\frac{\uv_{k'}^\T \EE \diagAdjecencyMatrix \uv_k}{t_k^2} - \frac{\widehat{\uv}_{k'}^\T \diagAdjecencyMatrix \widehat{\uv}_k}{\adjacencyEigenvalues_{k k}}\right) = \frac{\debiasedEigenvalues_{k' k'} \cdot \adjacencyEigenvectors_{i k'}}{\debiasedEigenvalues_{k' k'} - \adjacencyEigenvalues_{k k}} \cdot O_\ell \left(\sqrt{\frac{\log \nsize}{\nsize^3 \sparsityParam^3}}\right).
    \end{align*}
    Next, from Condition~\ref{cond: eigenvalues divergency}, we have $\lambda_{k'} - \lambda_k = \Omega(\nsize \sparsityParam)$. Since $\debiasedEigenvalues_{k' k'} = \lambda_{k'} + O_{\ell}(\sqrt{\sparsityParam \log \nsize})$ due to Lemma~\ref{lemma: debaised eigenvalues behaviour} and $\adjacencyEigenvalues_{kk} = t_k + O_{\ell}(\sqrt{\sparsityParam \log \nsize})$ due to Lemmas~\ref{lemma: log estimate vector difference} and~\ref{lemma: eigenvalues difference}, we have
    \begin{align*}
        \frac{\debiasedEigenvalues_{k' k'}}{\debiasedEigenvalues_{k' k'} - \adjacencyEigenvalues_{k k}} 
        = \frac{O(\nsize \sparsityParam)}{\Omega(\nsize \sparsityParam)} = O(1).
    \end{align*}
    Finally, we have $\adjacencyEigenvectors_{i k'} = \probEigenvectors_{i k'} + O_{\ell} \left( \sqrt{\frac{\log \nsize}{\nsize^2 \sparsityParam}} \right)$ due to Lemma~\ref{lemma: adj eigenvectors displacement}. Since $\probEigenvectors_{i k'} = O(\nsize^{-1 / 2})$ due to Lemma~\ref{lemma: eigenvectors max norm}, we conclude that $\adjacencyEigenvectors_{i k'} = O_{\ell}(\nsize^{-1/2})$. Thus, we obtain
    \begin{align*}
        \frac{\debiasedEigenvalues_{k' k'} \cdot \adjacencyEigenvectors_{i k'}}{\debiasedEigenvalues_{k' k'} - \adjacencyEigenvalues_{k k}} \left(\frac{\uv_{k'}^\T \EE \diagAdjecencyMatrix \uv_k}{t_k^2} - \frac{\widehat{\uv}_{k}^\T \diagAdjecencyMatrix \widehat{\uv}_k}{\adjacencyEigenvalues_{k k}}\right) = O \left( \sqrt{\frac{\log \nsize}{\nsize^4 \sparsityParam^3}} \right).
    \end{align*}
    Next, we have
    \begin{align*}
        \left(\frac{\lambda_{k'} \probEigenvectors_{i k'}}{\lambda_{k'} - t_k} - \frac{\debiasedEigenvalues_{k' k'} \cdot \adjacencyEigenvectors_{i k'}}{\debiasedEigenvalues_{k' k'} - \adjacencyEigenvalues_{k k}}\right)
        & = \frac{
              \lambda_{k'} (\probEigenvectors_{i k'} - \adjacencyEigenvectors_{i k'}) + (\lambda_{k'} - \debiasedEigenvalues_{k' k'}) \adjacencyEigenvectors_{i k'}
          }{\lambda_{k'} - t_k} \\
          & \quad - \frac{
              (\lambda_{k'} - \debiasedEigenvalues_{k' k'}) - (t_k - \adjacencyEigenvalues_{k k})
          }{
              (\lambda_{k'} - t_k)(\debiasedEigenvalues_{k' k'} - \adjacencyEigenvalues_{k k})
          } \cdot \debiasedEigenvalues_{k' k'} \cdot \adjacencyEigenvectors_{i k'} \\
       & = \frac{
          O(\nsize \sparsityParam) \cdot O_{\ell} \left( \sqrt{\frac{\log \nsize}{\nsize^2 \sparsityParam}}\right) + O_{\ell}(\sqrt{\sparsityParam \log \nsize}) O(\nsize^{-1/2})
       }{\Omega(\nsize \sparsityParam)} \\
       & \quad + \frac{O_{\ell}(\sqrt{\sparsityParam \log \nsize}) + O_{\ell}(\sqrt{\sparsityParam \log \nsize})}{\Omega(\nsize^2 \sparsityParam^2)} \cdot O(\nsize \sparsityParam) \cdot O(\nsize^{-1/2}) \\
       & = O_\ell \left( \sqrt{\frac{\log \nsize}{\nsize^2 \sparsityParam}}\right).
    \end{align*}
    The terms above were bounded via Lemmas~\ref{lemma: debaised eigenvalues behaviour},~\ref{lemma: log estimate vector difference},~\ref{lemma: adj eigenvectors displacement} and~\ref{lemma: eigenvalues difference}.
    Since $|\uv_k^\T \EE \diagAdjecencyMatrix \uv_{k'}| \le \Vert \EE \diagAdjecencyMatrix \Vert \le \nsize \sparsityParam$, we have
    \begin{align}
      & \frac{1}{t_k^2} \frac{\lambda_{k'} \probEigenvectors_{i k'}}{\lambda_{k'} - t_k} \cdot \uv_{k'}^\T \EE \diagAdjecencyMatrix \uv_k 
      -  \frac{\debiasedEigenvalues_{k' k'} \cdot \adjacencyEigenvectors_{i k'}}{\debiasedEigenvalues_{k' k'} - \adjacencyEigenvalues_{k k}} \cdot \sum_{j = 1}^{\nsize} \frac{\diagAdjecencyMatrix_{j j} \adjacencyEigenvectors_{j k'} \adjacencyEigenvectors_{j k}}{\adjacencyEigenvalues_{k k}^2} \label{eq: averaging lemma, bias term} \\
      &= O_{\ell} \left( \sqrt{\frac{\log \nsize}{\nsize^2 \sparsityParam}}\right) \cdot \frac{\uv_{k'}^\T \EE \diagAdjecencyMatrix \uv_k}{t_k^2} + O_{\ell} \left( \sqrt{\frac{\log \nsize}{\nsize^4 \sparsityParam^3}}\right) \nonumber
      = O_{\ell} \left( \sqrt{\frac{\log \nsize}{\nsize^4 \sparsityParam^3}}\right). \nonumber
    \end{align}
    %
    Combining~\eqref{eq: averaging lemma, D term approximation},~\eqref{eq: averaging lemma, quad form factor approximation} and~\eqref{eq: averaging lemma, bias term} and using $\adjacencyEigenvectors_{ik} = O_{\ell}(\nsize^{-1/2})$, we obtain
    \begin{align*}
        \tilde \probEigenvectors_{ik} = \adjacencyEigenvectors_{ik} \left( 1 - \frac{\EE \diagAdjecencyMatrix_{ii} - 3/2 \cdot \widehat{\uv}_{k}^\T \EE \diagAdjecencyMatrix \widehat{\uv}_k}{t_k^2}\right) - \sum_{k' \in [\nclusters] \setminus \{k\} } \frac{\lambda_{k'} \probEigenvectors_{i k'} }{\lambda_{k'} - t_k} \cdot \frac{\uv_{k'}^\T \EE \diagAdjecencyMatrix \uv_k}{t_k^2} + O_{\ell} \left(\sqrt{\frac{\log \nsize}{\nsize^4 \sparsityParam^3}} \right).
    \end{align*}
    We substitute asymptotic expansion from Lemma~\ref{lemma: eigenvector power expansion} instead of $\adjacencyEigenvectors_{ik}$, and, using \newline bounds~\eqref{eq: U hat terms asumptotics-f}-\eqref{eq: U hat asymptotics-l}, obtain:
    \begin{align*}
        \tilde \probEigenvectors_{i k} & = \probEigenvectors_{i k} + \frac{\ev_i^\T \displaceMatrix \uv_k}{t_k} + \frac{\ev_i^\T \displaceMatrix^2 \uv_k}{t_k^2} - \frac{3}{2} \cdot \probEigenvectors_{ik} \frac{\uv_k^\T \EE \displaceMatrix^2 \uv_k}{t_k^2} \\
        & \quad + \frac{1}{t_k^2} \sum_{k' \in [\nclusters] \setminus\{k\} } \frac{\lambda_{k'} \probEigenvectors_{i k'} }{\lambda_{k'} - t_k} \cdot \uv_{k'}^\T \EE \displaceMatrix^2 \uv_k - \probEigenvectors_{ik} \frac{\EE \diagAdjecencyMatrix_{ii}}{t_k^2} + \frac{3}{2} \probEigenvectors_{i k} \frac{\uv_k^\T \EE \diagAdjecencyMatrix \uv_k}{t_k^2} \\
        & \quad - \frac{1}{t_k^2} \sum_{k' \in [\nclusters] \setminus \{k\} } \frac{\lambda_{k'} \probEigenvectors_{i k'}}{\lambda_{k'} - t_k} \cdot \uv_{k'}^\T \EE \displaceMatrix^2 \uv_k + O_{\ell} \left( \sqrt{\frac{\log \nsize}{\nsize^3 \sparsityParam}} \right) \\
        & = \probEigenvectors_{i k} + \frac{\ev_i^\T \displaceMatrix \uv_k}{t_k} + \frac{\ev_i^\T (\displaceMatrix^2 - \EE \displaceMatrix^2) \uv_k}{t_k^2} + \frac{\ev_i^\T (\EE \displaceMatrix^2 - \EE \diagAdjecencyMatrix) \uv_k}{t_k^2} \\
        & \quad + \frac{3}{2} \probEigenvectors_{ik} \cdot \frac{\uv_k^\T (\EE \diagAdjecencyMatrix - \EE \displaceMatrix^2) \uv_k}{t_k^2} \\
        & \quad + \frac{1}{t_k^2} \sum_{k' \in [\nclusters] \setminus \{k\}} \frac{\lambda_{k'} \probEigenvectors_{i k'}}{\lambda_{k'} - t_k} \cdot \uv_{k'}^\T (\EE \displaceMatrix^2 - \EE \diagAdjecencyMatrix) \uv_k+ O_{\ell} \left( \sqrt{\frac{\log \nsize}{\nsize^3 \sparsityParam}} \right),
    \end{align*}
    where we use $\nsize^4 \sparsityParam^3 \ge \nsize^3 \sparsityParam$, provided $\sparsityParam \ge \nsize^{-1/2}$ due to Condition~\ref{cond: sparsity param bound}. We have 
    \begin{equation*}
        (\EE \displaceMatrix^2)_{ij} = \indicator[i = j] \sum_{t = 1}^\nsize \probMatrix_{it} - \probMatrix_{i t}^2
        = \EE \diagAdjecencyMatrix_{ij} - \indicator[i = j] \sum_{t = 1}^\nsize \probMatrix^2_{it}.
    \end{equation*}
    Consequently, we have $\Vert \EE \diagAdjecencyMatrix - \EE \displaceMatrix^2 \Vert = O(\nsize \sparsityParam^2)$ and
    \begin{align*}
        \ev_i^\T (\EE \displaceMatrix^2 - \EE \diagAdjecencyMatrix) \uv_k = (\EE \displaceMatrix^2 - \EE \diagAdjecencyMatrix)_i \probEigenvectors_{ik} \overset{\text{Lemma~\ref{lemma: eigenvectors max norm}}}{=} O(\nsize^{1/2} \sparsityParam^2).
    \end{align*}
    Analogously, we have
    \begin{align*}
        \probEigenvectors_{i k'} \cdot \frac{\uv_{k'}^\T (\EE \diagAdjecencyMatrix - \EE \displaceMatrix^2) \uv_k}{t_k^2} \le t_k^{-2} |\probEigenvectors_{i k'}| \cdot \Vert \EE \diagAdjecencyMatrix - \EE \displaceMatrix^2 \Vert = O \left( \sqrt{\frac{1}{\nsize^3}} \right)
    \end{align*}
    for any $k' \in [\nclusters]$. Since $\frac{\lambda_{k'}}{\lambda_{k'} - t_k} = O(1)$ for any $k' \in [\nclusters] \setminus \{k\}$, we get
    \begin{align*}
        \tilde \probEigenvectors_i = \probEigenvectors_i + \ev_i^\T \displaceMatrix \probEigenvectors \meanEigs^{-1} + \ev_i^\T (\displaceMatrix^2 - \EE \displaceMatrix^2) \probEigenvectors \meanEigs^{-2} + O_{\ell} \left( \sqrt{\frac{\log \nsize}{\nsize^3 \sparsityParam}} \right).
    \end{align*}
  \end{proof}

\subsection{Pure sets approximation}

The aim of this section is to investigate the difference between $\estimator[\pureNodesSet]_k = \{j \mid \estimator[\equalityStatistic]_{i_k j}^\penalizer < t_{\nsize}\}$ and $\pureNodesSet_k$. For a reminder, we have defined
\begin{align*}
    \equalityStatisticCenter^\penalizer_{ij} & = (\probEigenvectors_i - \probEigenvectors_j)
    \left( \asymptoticVariance(i, j) + \penalizer \identity \right)^{-1}
    (\probEigenvectors_i - \probEigenvectors_j)^{\T}, \\
    \equalityStatistic^\penalizer_{ij} & = (\adjacencyEigenvectors_i - \adjacencyEigenvectors_j) 
    \left(\asymptoticVariance(i, j) + \penalizer \identity\right)^{-1}
    (\adjacencyEigenvectors_i - \adjacencyEigenvectors_j)^{\T}, \\
    \estimator[\equalityStatistic]^\penalizer_{ij} & = (\adjacencyEigenvectors_i - \adjacencyEigenvectors_j) 
    \left(\estimator[\asymptoticVariance](i, j) + \penalizer \identity\right)^{-1}
    (\adjacencyEigenvectors_i - \adjacencyEigenvectors_j)^{\T}.
\end{align*}
We start with concentration of $\estimator[\equalityStatistic]_{i j}^\penalizer$.

\begin{lemma}
\label{lemma: statistic difference}
  Consider two arbitrary indices $i, j \in [\nsize]$. Then for each $\varepsilon$ there exist $n_0 \in \mathbb{N}$ and $\delta_1, \delta_2 > 0$ such that for any $\nsize \ge n_0$
  \begin{align*}
    \PP \left(
      \bigl| 
        \estimator[\equalityStatistic]_{ij}
        -
        \equalityStatisticCenter^\penalizer_{ij}
      \bigr|
      \ge \sqrt{\nsize \sparsityParam} \Vert \nodeCommunityMatrix_i - \nodeCommunityMatrix_j \Vert_2 \cdot \delta_1 \sqrt{\log \nsize}
      + 
      \delta_2 \log \nsize + \nsize^{1 - 1/12} \sparsityParam \Vert \nodeCommunityMatrix_i - \nodeCommunityMatrix_j \Vert^2
    \right)
    \le \nsize^{-\varepsilon}.
  \end{align*}
\end{lemma}

\begin{proof}
  Define
  \begin{equation*}
      \asymptoticVariance_\penalizer(i, j) = \asymptoticVariance(i, j) + \penalizer \identity, \quad \quad
      \estimator[\asymptoticVariance]_{\penalizer}(i, j) = \estimator[\asymptoticVariance](i, j) + \penalizer \identity.
  \end{equation*}
  We denote $\xi_i = \adjacencyEigenvectors_i - \probEigenvectors_i - \displaceMatrix_i \probEigenvectors \probEigenvalues^{-1}$ and observe:
  \begin{align}
    \, & \equalityStatistic^\penalizer_{ij} = 
    \equalityStatisticCenter^\penalizer_{ij} + 
    (\displaceMatrix_i - \displaceMatrix_j) \probEigenvectors \probEigenvalues^{-1}
    \asymptoticVariance_\penalizer^{-1}(i, j)
    (\probEigenvectors_i - \probEigenvectors_j)^{\T} \nonumber \\
    & \quad + (\displaceMatrix_i - \displaceMatrix_j) \probEigenvectors \probEigenvalues^{-1}
    \asymptoticVariance_\penalizer(i, j)^{-1}
    \left(
      \adjacencyEigenvectors_i - \adjacencyEigenvectors_j
    \right)^{\T} \nonumber \\
    & \quad +
    (\xi_i - \xi_j) \asymptoticVariance_\penalizer(i, j)^{-1} (\probEigenvectors_i - \probEigenvectors_j)^{\T}
    +
    (\xi_i - \xi_j) \asymptoticVariance_\penalizer(i, j)^{-1} (\adjacencyEigenvectors_i - \adjacencyEigenvectors_j)^{\T} \nonumber \\
    & = 
    \equalityStatisticCenter^\penalizer_{ij} + 
    2 (\displaceMatrix_i - \displaceMatrix_j) \probEigenvectors \probEigenvalues^{-1}
    \asymptoticVariance_\penalizer^{-1}(i, j)
    (\probEigenvectors_i - \probEigenvectors_j)^{\T} \nonumber \\
    & \quad  + 
    (\displaceMatrix_i - \displaceMatrix_j) 
    \probEigenvectors \probEigenvalues^{-1}
    \asymptoticVariance_\penalizer(i, j)^{-1}
    \probEigenvalues^{-1}
    \probEigenvectors^{\T}
    (\displaceMatrix_i - \displaceMatrix_j)^{\T} \nonumber \\
    & \quad +
    2 (\xi_i - \xi_j)
    \asymptoticVariance_\penalizer(i, j)^{-1}
    \left(
      \probEigenvectors_i - \probEigenvectors_j +
      (\displaceMatrix_i - \displaceMatrix_j)
      \probEigenvectors 
      \probEigenvalues^{-1}
    \right)^{\T}
    +
    (\xi_i - \xi_j) 
    \asymptoticVariance_\penalizer(i, j)^{-1}
    (\xi_i - \xi_j)^{\T}. \label{eq: statistic-center difference}
  \end{align}
  Due to Lemma~\ref{lemma: log estimate vector difference} and Lemma~\ref{lemma: eigenvectors max norm}, we have $\ev_i^\T \displaceMatrix \uv_k = O_\ell(\sqrt{\sparsityParam \log \nsize})$ for any $i$. So, from Lemma~\ref{lemma: eigenvalues asymptotics} we get
  \begin{align*}
      \displaceMatrix_i \probEigenvectors \probEigenvalues^{-1} = O_\ell(\sqrt{\sparsityParam \log \nsize}) \cdot O \Biggl ( \frac{1}{\nsize \sparsityParam} \Biggr ) = O_\ell \Biggl (
          \frac{\sqrt{\log \nsize}}{\nsize \sqrt{\sparsityParam}}
        \Biggr ),
  \end{align*}
  Thus, we have
  \begin{align}
  \label{eq: displacement 1 statistic differennce lemma}
    \max_{i, j}
    \Vert
      (\displaceMatrix_i - \displaceMatrix_j) 
      \probEigenvectors
      \probEigenvalues^{-1}
    \Vert_2
    =
    O_\ell \left(
      \frac{\sqrt{\log \nsize}}{\nsize \sqrt{\sparsityParam}}
    \right).
  \end{align}
  Besides, according to Lemma~\ref{lemma: adj_eigenvectors_displacement}, we have $\xi_i = O_{\prec} \left( \frac{1}{\sqrt{\nsize} \nsize \sparsityParam} \right)$ and so
  \begin{align}
  \label{eq: dispacement 2 statistic difference lemma}
    \max_{i, j}
    \Vert
      \xi_i - \xi_j
    \Vert_2
    =
    O_{\prec} \left(
      \frac{1}{ \sparsityParam \sqrt{\nsize^3}}
    \right)
  \end{align}
  holds. From Lemma~\ref{lemma: F rows tensor product} there is the constant $C$ such that
  \begin{align}
  \label{eq: displacement 3 statistic difference lemma}
    \left \Vert
      \probEigenvectors_i
      -
      \probEigenvectors_j
    \right \Vert_2
    \leqslant
    \frac{C_1 \Vert \nodeCommunityMatrix_i - \nodeCommunityMatrix_j \Vert_2}{\sqrt{\nsize}}.
  \end{align}
  In addition, from Lemma~\ref{lemma: bounds of statistic center} we get $\Vert \asymptoticVariance_\penalizer(i, j)^{-1} \Vert_2 \leqslant C_2 \nsize^2 \sparsityParam$ for some constant $C_2$. Define $\Delta_{ij} = \Vert \nodeCommunityMatrix_i - \nodeCommunityMatrix_j\Vert_2$. Using bounds~\eqref{eq: displacement 1 statistic differennce lemma}-\eqref{eq: displacement 3 statistic difference lemma}, we may bound all terms of~\eqref{eq: statistic-center difference} uniformly over $i$ and $j$ as follows:
  \begin{align*}
    (i) \quad &
      \bigl\Vert 
        2 (\displaceMatrix_i - \displaceMatrix_j) \probEigenvectors \probEigenvalues^{-1}
          \asymptoticVariance_\penalizer^{-1}(i, j)
          (\probEigenvectors_i - \probEigenvectors_j)^{\T}
      \bigr\Vert_2
        \le 2  \Vert (\displaceMatrix_i - \displaceMatrix_j) \probEigenvectors \probEigenvalues^{-1} \Vert_2 \times \\
        & \quad \times \Vert \asymptoticVariance_\penalizer^{-1}(i, j) \Vert_2 \cdot \Vert\probEigenvectors_i - \probEigenvectors_j \Vert_2 \\
      = ~ & O_\ell \left( \frac{\sqrt{\log \nsize}}{\nsize \sqrt{\sparsityParam}} \right) \cdot O(\nsize^2 \sparsityParam) \cdot O \left(\nsize^{-1/2}\right) \Delta_{ij}
      = O_\ell \left(\sqrt{\nsize \sparsityParam \log \nsize}\right) \Delta_{i j}, \\
    (ii) \quad &
      \bigl\Vert
        (\displaceMatrix_i - \displaceMatrix_j) 
        \probEigenvectors \probEigenvalues^{-1}
        \asymptoticVariance_\penalizer(i, j)^{-1}
        \probEigenvalues^{-1}
        \probEigenvectors^{\T}
        (\displaceMatrix_i - \displaceMatrix_j)^{\T}
      \bigr\Vert_2 \\
      & 
      \le \Vert (\displaceMatrix_i - \displaceMatrix_j) \probEigenvectors \probEigenvalues^{-1} \Vert_2^2 \cdot \Vert \asymptoticVariance_\penalizer(i, j)^{-1} \Vert_2 \\
      = ~ & O_\ell \left(\frac{\log \nsize}{\nsize^2 \sparsityParam} \right) O(\nsize^2 \sparsityParam)
      = O_\ell \left(\log \nsize \right), \\
    (iii) \quad &
        \bigl\Vert
          2 (\xi_i - \xi_j)
          \asymptoticVariance_\penalizer(i, j)^{-1}
          \left(
            \probEigenvectors_i - \probEigenvectors_j +
            (\displaceMatrix_i - \displaceMatrix_j)
            \probEigenvectors 
            \probEigenvalues^{-1}
          \right)^{\T}
        \bigr\Vert_2 \\
        \le ~ & 2 \Vert \xi_i - \xi_j \Vert_2 \cdot \Vert \asymptoticVariance_\penalizer(i, j)^{-1} \Vert_2 \left(\Vert \probEigenvectors_i - \probEigenvectors_j \Vert_2 + \Vert (\displaceMatrix_i - \displaceMatrix_j) \probEigenvectors \probEigenvalues^{-1} \Vert_2 \right) \\
        = ~ & O_\prec \left( \sparsityParam^{-1} \nsize^{-3/2} \right) O(\nsize^2 \sparsityParam) \left(O(\nsize^{-1/2}) \cdot \Delta_{ij}+ O_\ell \left(\frac{\sqrt{\log \nsize}}{\nsize \sqrt{\sparsityParam}} \right)\right) \\
        & = O_{\prec}(1) \cdot \Delta_{ij} + O_\prec \left( \sqrt{\frac{\log \nsize}{\nsize \sparsityParam}} \right), \\
    (iv) \quad &
      \bigl\Vert
        (\xi_i - \xi_j) 
        \asymptoticVariance_\penalizer(i, j)^{-1}
        (\xi_i - \xi_j)^{\T}
      \bigr\Vert_2
      =
      O_{\prec} \left( \nsize^{-3} \sparsityParam^{-2} \right) O(\nsize^2 \sparsityParam)
      = O_\prec \left(\frac{1}{\nsize \sparsityParam}\right).
  \end{align*}
  Thus, we obtain
  \begin{align}
      \left|\equalityStatistic^\penalizer_{i j} - \equalityStatisticCenter^\penalizer_{ij} \right| = O_\ell \left(\Delta_{i j} \sqrt{\nsize \sparsityParam \log \nsize} \right) + O_\ell \left(\log \nsize \right).
  \label{eq:bound_test_1}
  \end{align}
  Next, we get
  \begin{align}
    |\estimator[\equalityStatistic]^\penalizer_{ij} - \equalityStatistic^\penalizer_{ij}| \le \Vert \adjacencyEigenvectors_i - \adjacencyEigenvectors_j \Vert^2_2 \cdot
    \Vert \asymptoticVariance_\penalizer^{-1}(i, j) - \estimator[\asymptoticVariance]^{-1}_\penalizer(i, j)\Vert.
  \label{eq:bound_test_2}
  \end{align}
  Define $\Delta_{\asymptoticVariance}$ and $\Delta_{\asymptoticVariance}'$ as follows:
  \begin{equation*}
      \Delta_{\asymptoticVariance} = \estimator[\asymptoticVariance](i, j) - \asymptoticVariance(i, j), \quad \quad
      \Delta_{\asymptoticVariance}' = \estimator[\asymptoticVariance]_\penalizer^{-1}(i, j) - \asymptoticVariance_\penalizer^{-1}(i, j).
  \end{equation*}
  Since
  \begin{align*}
      0 & = \estimator[\asymptoticVariance]^{-1}_\penalizer(i, j) \estimator[\asymptoticVariance]_\penalizer(i, j) - \asymptoticVariance^{-1}_\penalizer(i,j) \asymptoticVariance_\penalizer(i, j) = \estimator[\asymptoticVariance]^{-1}_\penalizer(i, j) \Delta_{\asymptoticVariance} + \Delta_{\asymptoticVariance}' \asymptoticVariance_\penalizer(i,j),
  \end{align*}
  we get
  \begin{equation*}
      \Delta_{\asymptoticVariance}' = -\estimator[\asymptoticVariance]^{-1}_\penalizer(i, j) \Delta_{\asymptoticVariance} \asymptoticVariance^{-1}_\penalizer(i, j)
      =  -\bigl(\Delta_{\asymptoticVariance}' + \asymptoticVariance_\penalizer^{-1}(i, j)\bigr) \Delta_{\asymptoticVariance} \asymptoticVariance^{-1}_\penalizer(i, j).
  \end{equation*}
  Rearranging terms, we obtain
  \begin{align*}
      \Delta_{\asymptoticVariance}' & = -\bigl(\identity + \Delta_{\asymptoticVariance} \asymptoticVariance_\penalizer(i, j) \bigr)^{-1} \asymptoticVariance_\penalizer^{-1}(i, j) \Delta_{\asymptoticVariance} \asymptoticVariance_\penalizer^{-1}(i, j).
  \end{align*}
  Due to Lemma~\ref{lemma: uniformly covariance estimation}, we have $\Vert \Delta_{\asymptoticVariance} \Vert = O_{\prec}(\nsize^{-5/2} \sparsityParam^{-3/2})$. Applying Lemma~\ref{lemma: bounds of statistic center}, we obtain
  \begin{align*}
      \Vert \Delta_{\asymptoticVariance}' \Vert & \le \left(1 - \Vert \Delta_{\asymptoticVariance} \Vert \cdot \Vert \asymptoticVariance_\penalizer^{-1}(i, j) \Vert \right)^{-1} \Vert \asymptoticVariance^{-1}_\penalizer(i, j)\Vert^2 \Vert \Delta_{\asymptoticVariance} \Vert \\
      & = O(1) \cdot O(\nsize^4 \sparsityParam^2) \cdot O_{\prec}(\nsize^{-5/2} \sparsityParam^{-3/2}) 
      = O_{\prec}(\nsize^{3/2} \sparsityParam^{1/2}).
  \end{align*}
  Substituting the above into~\eqref{eq:bound_test_2} and applying~\eqref{eq: displacement 3 statistic difference lemma}, we get
  \begin{align*}
      |\estimator[\equalityStatistic]^\penalizer_{ij} - \equalityStatistic^\penalizer_{ij}| = O_{\prec}(\sqrt{\nsize \sparsityParam}) \cdot \Delta_{ij}^2.
  \end{align*}
  With probability $1 - \nsize^{-\varepsilon}$ this term is less than $\nsize^{1- 1/12} \sparsityParam \Delta^2_{ij}$ for any $\varepsilon$, provided $\sparsityParam > \nsize^{-1/3}$ and $\nsize$ is large enough. Thus, the lemma follows.
\end{proof}

The result of next lemma ensures that the proposed method allows to select the set of vertices that contains all the pure nodes and does not contain many non-pure ones.
\begin{lemma}
\label{lemma: pure set approximation}
  Assume that Conditions~\ref{cond: nonzero B elements}-\ref{cond: theta distribution-b} hold and SPA chooses an index $i_k$, then for each $\varepsilon$ there is $n_0$ such that for all $\nsize > n_0$ the following holds with probability at least $1 - \nsize^{-\varepsilon}$: $t_{\nsize} = C(\varepsilon) \log \nsize$ ensures that the set $\pureNodesSet_k$ is a subset of $\estimator[\pureNodesSet]_k = \{j \mid \estimator[\equalityStatistic]^\penalizer_{i_k j} \le t_{\nsize} \}$, and $\estimator[\pureNodesSet]_k \setminus \pureNodesSet_k$  has cardinality at most $C'(\varepsilon) \nsize^{\alpha/2}$. Moreover, for any $j \in \estimator[\pureNodesSet]_k$, we have the following:
  \begin{align*}
      \Vert \nodeCommunityMatrix_j - \ev_k \Vert \le \Tilde{C}(\varepsilon) \sqrt{\frac{\log \nsize}{\nsize \sparsityParam}}.
  \end{align*}
\end{lemma}

\begin{proof}
  According to Lemma~\ref{lemma: statistic difference}, a set $\{j \mid \estimator[\equalityStatistic]_{i_k j} \le \threshold_\nsize\}$ contains
  \begin{align*}
    \left\{j \mid \equalityStatisticCenter^\penalizer_{i_k j} \le \threshold_{\nsize} - \delta_1(\varepsilon) \sqrt{\nsize \sparsityParam \log \nsize} \Vert \nodeCommunityMatrix_{i_k} - \nodeCommunityMatrix_j \Vert_2 - \delta_2(\varepsilon) \log \nsize - \nsize^{1 - 1/12} \sparsityParam \Vert \nodeCommunityMatrix_{i_k} - \nodeCommunityMatrix_j \Vert^2 \right\}.
  \end{align*}
  with probability at least $1 - \nsize^{-\varepsilon}$. Due to Lemma~\ref{lemma: bounds of statistic center}, this set contains
  \begin{align}
  \label{eq: lemma pure set est, upper set}
     \left\{j \mid C \Vert \nodeCommunityMatrix_{i_k} - \nodeCommunityMatrix_j \Vert_2^2 \nsize \sparsityParam \le \threshold_\nsize - \delta_1(\varepsilon) \sqrt{\nsize \sparsityParam \log \nsize} \Vert \nodeCommunityMatrix_{i_k} - \nodeCommunityMatrix_j \Vert_2 - \delta_2(\varepsilon) \log \nsize \right\},
  \end{align}
  for some constant $C$. Here we use $\nsize^{1 - 1/12} \sparsityParam \le \nsize \sparsityParam$ for large enough $\nsize$. Since $\sigma_{\min}(\basisMatrix) \ge C \sqrt{\nsize}$ due to Lemma~\ref{lemma: F rows tensor product} and $\probEigenvectors = \nodeCommunityMatrix \basisMatrix$, Lemma~\ref{lemma: spa selection} guarantees that there is a constant $\delta_3(\varepsilon)$ such that 
  \begin{align*}
      \Vert \nodeCommunityMatrix_{i_k} - \ev_k \Vert_2 \le \frac{1}{\sigma_{\min}(\basisMatrix)} \Vert \probEigenvectors_{i_k} - \ev_k^\T \basisMatrix \Vert_2 \le \delta_3(\varepsilon) \sqrt{\log \nsize / (\nsize \sparsityParam)}
  \end{align*} 
  with probability $\nsize^{-\varepsilon}$. Thus, set~\eqref{eq: lemma pure set est, upper set} contains $\pureNodesSet_k$ if
  \begin{align*}
      C \delta_3(\varepsilon) \log \nsize \le \threshold_\nsize - \delta_1(\varepsilon) \cdot \delta_3(\varepsilon) \log \nsize - \delta_2(\varepsilon) \log \nsize.
  \end{align*}
  Choose $\threshold_\nsize = \left\{\bigl(C + \delta_1(\varepsilon)\bigr) \delta_3(\varepsilon) + \delta_2(\varepsilon) \right\} \log \nsize$, then the pure node set $\pureNodesSet_k$ is contained in set~\eqref{eq: lemma pure set est, upper set} with probability $1 - 2 \nsize^{-\varepsilon}$. Similarly, we have
  \begin{align}
  \label{eq: lemma pure set est, remainder}
    \{j \mid \estimator[\equalityStatistic]_{i_k j} \le \threshold_\nsize \} \subset \left\{j \mid C' \Vert \nodeCommunityMatrix_{i_k} - \nodeCommunityMatrix_j \Vert_2^2 \nsize \sparsityParam \le \threshold_\nsize + \delta_1(\varepsilon) \sqrt{\nsize \sparsityParam \log \nsize} \Vert \nodeCommunityMatrix_{i_k} - \nodeCommunityMatrix_j \Vert_2 + \delta_2(\varepsilon) \log \nsize \right\} 
  \end{align}
  for some other constant $C'$. Since
  \begin{align*}
      \Vert \nodeCommunityMatrix_{j} - \ev_k \Vert_2 - \Vert \nodeCommunityMatrix_{i_k} - \ev_k \Vert_2 \le \Vert \nodeCommunityMatrix_j - \nodeCommunityMatrix_{i_k} \Vert_2 & \le \Vert \nodeCommunityMatrix_j - \ev_k \Vert_2 + \Vert \nodeCommunityMatrix_{i_k} - \ev_k \Vert,  \\
      \Vert \nodeCommunityMatrix_{j} - \ev_k \Vert_2 - \delta_3(\varepsilon) \sqrt{\frac{\log \nsize}{\nsize \sparsityParam}} \le \Vert \nodeCommunityMatrix_j - \nodeCommunityMatrix_{i_k} \Vert_2 & \le \Vert \nodeCommunityMatrix_j - \ev_k \Vert_2 + \delta_3(\varepsilon) \sqrt{\frac{\log \nsize}{\nsize \sparsityParam}},
  \end{align*}
  set~\eqref{eq: lemma pure set est, remainder} belongs to a larger set
  \begin{align*}
    S 
    & = \left\{
      j 
      \mid
       C' \Vert \nodeCommunityMatrix_{j} - \ev_k \Vert_2^2 \nsize \sparsityParam \le \delta_4(\varepsilon) \log \nsize + \delta_5(\varepsilon) \sqrt{\nsize \sparsityParam \log \nsize} \Vert \nodeCommunityMatrix_j - \ev_k \Vert_2
    \right\}
  \end{align*}
  with probability at least $1 - 2 \nsize^{-\varepsilon}$. Hence, if $j \in S$, then
  \begin{equation*}
        \Vert \nodeCommunityMatrix_j - \ev_k \Vert_2 \le  \frac{\sqrt{\delta_5^2(\varepsilon) \nsize \sparsityParam \log \nsize + 4 C' \delta_4(\varepsilon) \nsize \sparsityParam \log \nsize} - \delta_5(\varepsilon) \sqrt{\nsize \sparsityParam \log \nsize}}{2 C' \nsize \sparsityParam}
        \le \delta_6(\varepsilon) \sqrt{\frac{\log \nsize}{\nsize \sparsityParam}}.
  \end{equation*}
  Condition~\ref{cond: theta distribution-b} ensures that $|S \setminus \pureNodesSet_k| \le C_{\delta_6} \nsize^{\alpha/2}$, and that concludes the proof.
\end{proof}

\subsection{Averaging over selected nodes}
\begin{lemma}
  \label{lemma: averaging lemma}
    Define
    \begin{align*}
      \estimator[\basisMatrix]_k & = \frac{1}{|\estimator[\pureNodesSet]_k|} \sum_{j \in \estimator[\pureNodesSet]_k} \tilde \probEigenvectors_{ik}.
    \end{align*}
    Then under Conditions~\ref{cond: nonzero B elements}-\ref{cond: theta distribution-b}, for any $\varepsilon$ there exist are constants $C_1(\varepsilon), C_2(\varepsilon)$ such that for $\threshold_n = C_1(\varepsilon) \log \nsize$, $C_{\basisMatrix} = C_2(\varepsilon)$, and $\nsize > \nsize_0(\varepsilon)$ we have
    \begin{align*}
      \PP \left( \min_{\permutationMatrix \in \mathbb{S}_\nclusters} \Vert \estimator[\basisMatrix] - \basisMatrix \permutationMatrix^\T \Vert_\F \ge \frac{C_{\basisMatrix} \sqrt{\log \nsize}}{\nsize^{1 + \alpha/2}\sqrt{\sparsityParam}}\right) \le \nsize^{-\varepsilon}.
    \end{align*}
  \end{lemma}
  
  \begin{proof}
    Due to Lemma~\ref{lemma: pure set approximation}, we can choose $t_n = C_1(\varepsilon) \log \nsize$ such that with probability $1 - n^{-\varepsilon} / 4$ we have the following:
      \begin{align}
          (i) \quad & \pureNodesSet_k \subset \estimator[\pureNodesSet]_k; \label{eq: pure nodes are subset of approximation} \\
          (ii) \quad & \Vert \nodeCommunityMatrix_j - \ev_k \Vert \le C(\varepsilon) \sqrt{\frac{\log \nsize}{\nsize \sparsityParam}}; \label{eq: deviation of the pure nodes} \\
          (iii) \quad & |\estimator[\pureNodesSet]_k \setminus \pureNodesSet_k | \le C'(\varepsilon) \log^\eta \nsize. \label{eq: approximation error}
      \end{align}
  In the proof, we assume that \textit{(i)-(iii)} holds. Additionally, we will use $t_k = \Omega(\nsize \sparsityParam)$, which is guaranteed by Lemmas~\ref{lemma: t_k is well-definied} and~\ref{lemma: eigenvalues asymptotics}.
  
  Due to~\eqref{eq: pure nodes are subset of approximation}, we have the decomposition
    \begin{align}
    \label{eq: averaging lemma, pure node set decomposition}
      \frac{1}{|\estimator[\pureNodesSet]_k|} 
      \sum_{j \in \estimator[\pureNodesSet]_k} 
        \tilde \probEigenvectors_{j} 
      & = \basisMatrix_k + \frac{1}{|\estimator[\pureNodesSet]_k|} \sum_{j \in \pureNodesSet_k} ( \tilde \probEigenvectors_j - \basisMatrix_k)
      + 
      \frac{1}{|\estimator[\pureNodesSet]_k|} 
          \sum_{j \in \estimator[\pureNodesSet_k] \setminus \pureNodesSet_k} 
              (\tilde  \probEigenvectors_j - \basisMatrix_k).
    \end{align}
    We start with analysis of the third term. Due to Lemma~\ref{lemma: eigenvector debiasing}, we have
    \begin{align*}
        \tilde \probEigenvectors_i = \probEigenvectors_i + \ev_i^\T \displaceMatrix \probEigenvectors \meanEigs^{-1} + \ev_i^\T (\displaceMatrix^2 - \EE \displaceMatrix^2) \probEigenvectors \meanEigs^{-2} + O_{\ell} \left( \sqrt{\frac{\log \nsize}{\nsize^3 \sparsityParam}} \right).
    \end{align*}
    Since $\probEigenvectors_i = \nodeCommunityMatrix_i \basisMatrix$ and $\basisMatrix_k = \ev_k^\T \basisMatrix$, for any $j \in \estimator[\pureNodesSet]_k \setminus \pureNodesSet_k$, we have
    \begin{align*}
        \tilde \probEigenvectors_j - \basisMatrix_k = (\nodeCommunityMatrix_j - \ev_k) \basisMatrix + \ev_i^\T \displaceMatrix \probEigenvectors \meanEigs^{-1} + \ev_i^\T (\displaceMatrix^2 - \EE \displaceMatrix^2) \probEigenvectors \meanEigs^{-2} + O_{\ell} \left( \sqrt{\frac{\log \nsize}{\nsize^3 \sparsityParam}} \right).
    \end{align*}
    Due to Lemma~\ref{lemma: F rows tensor product}, we have $\Vert \basisMatrix \Vert = O(1/ \sqrt{n})$. Together with~\eqref{eq: deviation of the pure nodes}, it implies
    \begin{align*}
        \Vert \tilde \probEigenvectors_j - \basisMatrix_k \Vert & \le \Vert \nodeCommunityMatrix_j - \ev_k \Vert \Vert \basisMatrix \Vert + \Vert \ev_j^\T \displaceMatrix \probEigenvectors \meanEigs^{-1} \Vert + \Vert \ev_j^\T (\displaceMatrix^2 - \EE \displaceMatrix^2) \probEigenvectors \meanEigs^{-2} \Vert + O_{\ell} \left( \sqrt{\frac{\log \nsize}{\nsize^3 \sparsityParam}} \right) \\
        & \le O \left ( \sqrt{\frac{\log \nsize}{\nsize \sparsityParam}} \right )\cdot O \left ( \frac{1}{\sqrt{n}} \right ) + \Vert \ev_j^\T \displaceMatrix \probEigenvectors \meanEigs^{-1} \Vert + \Vert \ev_j^\T (\displaceMatrix^2 - \EE \displaceMatrix^2) \probEigenvectors \meanEigs^{-2} \Vert  \\
        & \quad + O_{\ell} \left( \sqrt{\frac{\log \nsize}{\nsize^3 \sparsityParam}} \right).
    \end{align*}
    Due to Lemma~\ref{lemma: eigenvectors max norm}, we have $\Vert \uv_k \Vert_{\infty} = O(1/\sqrt{n})$. Therefore, Lemmas~\ref{lemma: log estimate vector difference} and~\ref{lemma: power deviation} imply
    \begin{align*}
      \Vert \ev_j^\T \displaceMatrix \meanEigs^{-1} \Vert & \le \sum_{k \in [\nclusters]} \frac{1}{t_k} |\ev_j^\T \displaceMatrix \uv_k| = O_\ell \left (\frac{\sqrt{\sparsityParam \log n}}{\nsize \sparsityParam}\right ) \\
      \Vert \ev_j^\T (\displaceMatrix^2 - \EE \displaceMatrix^2) \probEigenvectors \meanEigs^{-2} \Vert & \le \sum_{k \in [\nclusters]} \frac{1}{t_k^2} |\ev_j^\T (\displaceMatrix^2 - \EE \displaceMatrix^2) \uv_k | = O_\prec \left ( \frac{1}{(\nsize \sparsityParam)^{3/2}}\right ).
    \end{align*}
    We have $O_{\prec}(n^{-1/2} / \sparsityParam) = O_{\ell}(1)$ due to Condition~\ref{cond: sparsity param bound}. Thus, for any $j \in \estimator[\pureNodesSet]_k \setminus \pureNodesSet_k$, we have
    \begin{align*}
      \Vert \tilde \probEigenvectors_j - \basisMatrix_k \Vert \le O \left ( \frac{\sqrt{\log \nsize}}{\nsize \sparsityParam} \right )  + O_\ell \left ( \frac{\sqrt{\log \nsize}}{\nsize \sqrt{\sparsityParam}}\right ) = O_\ell \left ( \frac{\sqrt{\log \nsize}}{\nsize \sqrt{\sparsityParam}} \right ).
    \end{align*}
    Therefore, with probability $1 - n^{-\varepsilon}/2$, the third term of~\eqref{eq: averaging lemma, pure node set decomposition} is at most
    \begin{align}
      \frac{1}{|\estimator[\pureNodesSet]_k|} \sum_{j \in \estimator[\pureNodesSet]_k \setminus \pureNodesSet_k} \frac{C(\varepsilon) \sqrt{\log \nsize}}{\nsize \sqrt{\sparsityParam}} \le \frac{C(\varepsilon) |\estimator[\pureNodesSet]_k \setminus \pureNodesSet_k|}{|\pureNodesSet_k|} \cdot \frac{\sqrt{\log \nsize}}{\nsize \sqrt{\sparsityParam}} \le \frac{C'(\varepsilon) \nsize^{\alpha/2} \sqrt{\log \nsize}}{\nsize^{1 + \alpha} \sqrt{\sparsityParam}}, 
    \label{eq: averaging thrid term bound}
    \end{align}
    where we used~\eqref{eq: approximation error} and Condition~\ref{cond: theta distribution-a}.

    Next, we analyze the second term of~\eqref{eq: averaging lemma, pure node set decomposition}. If $j \in \pureNodesSet_k$, then $\probEigenvectors_j = \basisMatrix_{k}$. Hence, Lemma~\ref{lemma: eigenvector debiasing} implies
    \begin{align*}
      \frac{1}{|\pureNodesSet_k|} 
      \sum_{j \in \pureNodesSet_k} 
      \tilde{\probEigenvectors}_{j} 
      = 
      \basisMatrix_k 
      + 
      \frac{1}{\sqrt{|\pureNodesSet_k|}} 
      \mathbf{r}^{\T}
      \displaceMatrix \probEigenvectors \meanEigs^{-1}
      + 
      \frac{1}{\sqrt{|\pureNodesSet_k|}} 
      \mathbf{r}^{\T} (\displaceMatrix^2 - \EE \displaceMatrix^2) \probEigenvectors \meanEigs^{-2}
      +
      O_\ell \left( \sqrt{\frac{\log \nsize}{\nsize^3 \sparsityParam}} \right)
    \end{align*}
    for a unit vector $\mathbf{r} = \frac{1}{\sqrt{|\pureNodesSet_k|}} \sum_{j \in \pureNodesSet_k} \ev_j$. Finally, applying Lemma~\ref{lemma: log estimate vector difference} and Lemma~\ref{lemma: power deviation}, we derive
    \begin{align*}
      \mathbf{r}^{\T} \displaceMatrix \probEigenvectors \meanEigs^{-1} & = O_\ell \left(\sqrt{\frac{\log \nsize}{\nsize^2 \sparsityParam}} \right), \\
      \mathbf{r}^{\T} (\displaceMatrix^2 - \EE \displaceMatrix^2) \probEigenvectors \meanEigs^{-2} & = O_\prec \left(\frac{\sqrt{\nsize \sparsityParam}}{\nsize^2 \sparsityParam^2} \right) = \frac{1}{n \sqrt{\sparsityParam}} \cdot O_\prec \left ( \frac{1}{n^{1/2} \sparsityParam} \right ) = \frac{O_\ell(1)}{\nsize \sqrt{\sparsityParam}},
    \end{align*}
    where we used Condition~\ref{cond: sparsity param bound} to obtain the last inequality.
    Consequently, with probability $1 - n^{-\varepsilon}/2$, we have
    \begin{align*}
        \left | \frac{1}{|\estimator[\pureNodesSet]_k|} \sum_{j \in \pureNodesSet_k} (\tilde \probEigenvectors_j - \basisMatrix_k) \right | \le \frac{C''(\varepsilon) |\pureNodesSet_k|}{|\estimator[\pureNodesSet]_k|} \cdot \frac{\sqrt{\log \nsize}}{\nsize^{1 + \alpha/ 2} \sqrt{\sparsityParam}} \le \frac{C''(\varepsilon) \sqrt{\log \nsize}}{\nsize^{1 + \alpha/2} \sqrt{\sparsityParam}},
    \end{align*}
    where we used~\eqref{eq: pure nodes are subset of approximation}.
    Finally, we combine the above with bound~\eqref{eq: averaging thrid term bound} and substitute the result into~\eqref{eq: averaging lemma, pure node set decomposition}, establishing the lemma.
  \end{proof}

\subsection{Estimation of the number of communities}
\label{sec:number_communities}
\begin{lemma}
  \label{lemma: estimation of nclusters}
  Suppose Condition~\ref{cond: theta distribution-a} holds. Then, we have $\estimator[\nclusters] = \nclusters$ with probability $\nsize^{-\Omega(\log \nsize)}$.
  \end{lemma}
  
  \begin{proof}
  Note that for any indices $j \in [\nsize]$ we have 
      \begin{align}
      \label{eq: spectral displacement bound}
        |\lambda_j(\probMatrix) - \lambda_j(\adjacencyMatrix)| \le \Vert \displaceMatrix \Vert
      \end{align}
      due to Weyl's inequality. Since $\lambda_j(\probMatrix) = 0$ for $j > \nclusters$, we have $\max_{j > \nclusters} |\lambda_j(\adjacencyMatrix)| \le \Vert \displaceMatrix \Vert$. Let us bound the norm of $\displaceMatrix$ via the matrix Bernstein inequality. Decompose
      \begin{align*}
        \displaceMatrix = \sum_{1 \le i \le j \le \nsize} \displaceMatrix_{i j} (\ev_i \ev_j^{\T} + \ev_j \ev_i^{\T}) \cdot \frac{2 - \delta_{i j}}{2}.
      \end{align*}
      and apply Lemma~\ref{lemma: matrix bernstein inequality} for the summands. We obtain
      \begin{align*}
        \PP \left( 
          \Vert \displaceMatrix \Vert \ge t
        \right)
        \le \exp \left(-\frac{t^2 / 2}{\sigma^2 + \frac{1}{3} t} \right).
      \end{align*}
      where 
      \begin{equation*}
        \sigma^2
        = 
        \left \Vert 
          \sum_{1 \le i \le j \le \nsize} \EE \displaceMatrix_{ij}^2 (\ev_i \ev_j^{\T} + \ev_j \ev_i^{\T}) \frac{(2 - \delta_{ij})^2}{4}
        \right \Vert
        \le
        \left \Vert \diag \left( \sum_{t = 1}^{\nsize} \probMatrix_{i t} (1 - \probMatrix_{i t}) \right)_{i = 1}^{\nsize} \right \Vert
        \le
        \max_{i \in [\nsize]} \sum_{t = 1}^{\nsize} \probMatrix_{i t}.
      \end{equation*}
      Thus,
      \begin{align*}
        \PP \left(
          \Vert \displaceMatrix \Vert \ge \max_{i} \sqrt{\sum_{t = 1}^{\nsize} \probMatrix_{it}} \log \nsize
        \right) = \nsize^{-\Omega(\log \nsize)}.
      \end{align*}
      Meanwhile,
      \begin{align*}
        & \quad \PP \left(
          \sum_{t = 1}^{\nsize} \adjacencyMatrix_{i t} \le \frac{1}{2} \sum_{t = 1}^{\nsize} \probMatrix_{i t}
        \right)
        = 
        \PP \left(
          \sum_{t = 1}^{\nsize} (\adjacencyMatrix_{i t} - \probMatrix_{i t}) \le - \frac{1}{2} \sum_{t = 1}^{\nsize} \probMatrix_{i t}
        \right ) \\
        &=
        \PP \left(
          \sum_{t = 1}^{\nsize} (\probMatrix_{i t} - \adjacencyMatrix_{i t}) \ge \frac{1}{2} \sum_{t = 1}^{\nsize} \probMatrix_{i t}
        \right) 
        \le 
        \exp \left( 
          -\frac{
            \left [\frac{1}{2} \sum_{t = 1}^{\nsize} \probMatrix_{i t} \right ]^2
          }{
            \sparsityParam \nsize + \frac{1}{3} \cdot \frac{1}{2} \sum_{t = 1}^{\nsize} \probMatrix_{i t}
          }
        \right)
        =
        \exp \bigl(-\Omega(\nsize \sparsityParam)\bigr).
      \end{align*}
      Consequently,
      \begin{align*}
        \PP \left( \Vert \displaceMatrix \Vert \ge 2 \max_{i \in [\nsize]} \sqrt{\sum_{t = 1}^{\nsize} \adjacencyMatrix_{i t} \log^2 \nsize}  \right) \le \nsize^{-\Omega(\log \nsize)}.
      \end{align*}
      Hence, combining the above with~\eqref{eq: spectral displacement bound}, we obtain that
      \[
        \estimator[\nclusters] = \min_{j} \left\{\lambda_j(\adjacencyMatrix) \ge 2 \max_{i \in [\nsize]} \sqrt{\sum_{t = 1}^{\nsize} \adjacencyMatrix_{i t} \log^2 \nsize} \right\}
      \]
      is at most $\nclusters$ with probability $\nsize^{-\Omega(\log \nsize)}$. Due to Lemma~\ref{lemma: eigenvalues asymptotics}, we have $\lambda_\nclusters(\probMatrix) = \Theta(\nsize \sparsityParam)$ and, therefore, 
      \begin{align*}
        \PP \left( 
          2 \sqrt{\sum_{t = 1}^{\nsize} \adjacencyMatrix_{i t} \log^2 \nsize} \ge \lambda_{\nclusters} - \Vert \displaceMatrix \Vert
        \right)
        =
        \exp \bigl(-\Omega(\nsize)\bigr).
      \end{align*}
      Consequently, $\estimator[\nclusters] = \nclusters$ with probability $\nsize^{-\Omega(\log \nsize)}$.
  \end{proof}

\section{Proof of Theorem~\ref{theorem: lower bound with almost no pure nodes}}
\label{section: proof of the genral lower bound}
\def\ttb{\mathtt{b}}

We employ standard approach based on hypotheses testing.

\subsection{Additional notation}
  For this section, we introduce additional notation.
  \begin{itemize}
    \item Let $\Omega$ be a set of $\{0, 1\}$-valued vectors $\omega$ indexed by a finite set $\mathcal{X}$, i.e. $\Omega = \{\omega_{x} \mid x \in \mathcal{X} \}$. Then the Hamming distance $d_H(\omega, \omega')$ between two elements $\omega, \omega'$ of $\Omega$ is defined as follows:
    \begin{align*}
        d_H(\omega_1, \omega_2) = |\{ x \in \mathcal{X} \mid \omega_x \neq \omega_x'\}|.
    \end{align*}
    \item For two probability distributions $\PP_1, \PP_2$, we denote by $\KL(\PP_1 \Vert \PP_2)$ the Kullback--Leibler divergence (or simply KL-divergence) between them.
    \item For a function $f\colon X \to Y$ and a subset $X' \subset X$ , we define the image of $X'$ as follows:
    \begin{align*}
        f(X') = \{ f(x) \mid x \in X'\}.
    \end{align*}
    Additionally, if $f(X')$ is a set of matrices and $\mathbf{Y}$ is a matrix of the suitable shape, then
    \begin{align*}
        \mathbf{Y} f(X') & = \{\mathbf{Y} \mathbf{X} \mid \mathbf{X} \in f(X')\}, \\
        f(X') \mathbf{Y} & = \{\mathbf{X} \mathbf{Y} \mid \mathbf{X} \in f(X') \}.
    \end{align*}
  \end{itemize}

\subsection{Permutation-resistant code}
  Let $\omega$ be a $\{0, 1\}$-vector indexed by sets $\{k, k'\} \in \binom{[\nclusters]}{2}$. Define the set of such vectors by $\Omega$, $|\Omega| = 2^{\binom{\nclusters}{2}}$. Let $\probeMatrixFunction (\omega)$ be a matrix-valued function defined as follows:
  \begin{align*}
    \probeMatrixFunction_{k k'}(\omega) = \begin{cases}
        \frac{1}{4} + \omega_{\{k, k'\}} \mathtt{b}^\omega_{\{k, k'\}} \cdot \frac{\mu}{\nsize}, & k \neq k', \\
        \frac{1}{2}, & k = k',
    \end{cases}
  \end{align*}
  where $\ttb^\omega_{S} \in \{-1, 1\}, S \in \binom{[K]}{2},$ are signs chosen to minimize $\left | \sum_{S \in \binom{[K]}{2}} \omega_{S} \ttb^\omega_S  \right |$.
  We specify $\mu$ later. In what follows, we define a family of matrices $\mathcal{B}$ required for application of Lemma~\ref{lemma: fano lemma} as an image $\probeMatrixFunction(\Omega'')$ for some subset $\Omega'' \subset \Omega$. First, we satisfy the assumption of Lemma~\ref{lemma: fano lemma} on the semi-distance.
  
  Let $\Omega'$ be the subset of $\Omega$ obtain from Lemma~\ref{lemma: varshamov-gilbert bound}. Then, for any distinct $\omega, \omega' \in \Omega'$, we have
  \begin{align*}
    d_H(\omega, \omega') \ge \frac{1}{8} \binom{\nclusters}{2} \quad \text{and} \quad |\Omega'| \ge 1 + 2^{\frac 1 8 \binom{\nclusters}{2}}.
  \end{align*}

  Clearly, the map $\probeMatrixFunction\colon \Omega \to [0, 1]^{\nclusters \times \nclusters}$ is injective, i.e. there exists a map $\probeMatrixFunction^{-1}\colon \probeMatrixFunction(\Omega) \to \Omega$ such that $\probeMatrixFunction^{-1} (\probeMatrixFunction(\omega)) = \omega$. Next, the set $\probeMatrixFunction(\Omega)$ is invariant under permutations, i.e.
  \begin{align*}
    \permutationMatrix \probeMatrixFunction(\Omega) \permutationMatrix^\T = \probeMatrixFunction(\Omega)
  \end{align*}
  for any permutation matrix $\permutationMatrix \in \mathbb{S}_\nclusters$. 

  We can express $\Vert \permutationMatrix \probeMatrixFunction(\omega_1) \permutationMatrix^\T - \probeMatrixFunction(\omega_2) \Vert_\F$ in terms of the Hamming distance $$d_H \left (\probeMatrixFunction^{-1}(\permutationMatrix \probeMatrixFunction(\omega_1) \permutationMatrix^\T), \probeMatrixFunction(\omega_2) \right ).$$ In the following lemma, we construct a subset of $\Omega''$ such that for any $\omega_1, \omega_2 \in \Omega''$ the Hamming distance $d_H \left (\probeMatrixFunction^{-1}(\permutationMatrix \probeMatrixFunction(\omega_1) \permutationMatrix^\T), \probeMatrixFunction(\omega_2) \right )$ is large.

  \begin{lemma}
  \label{lemma: factorization}
    There exists a set $\Omega'' \subset \Omega$ such that
    \begin{itemize}
      \item $\zero \in \Omega''$,
      \item for any distinct $\omega_1, \omega_2 \in \Omega''$, we have
      \begin{align*}
          \min_{\permutationMatrix \in \mathbb{S}_\nclusters} d_H(\probeMatrixFunction^{-1} \left( \permutationMatrix \probeMatrixFunction(\omega_1) \permutationMatrix^\T \right), \omega_2) \ge \frac{1}{17} \binom{\nclusters}{2} - 2,
      \end{align*}
      \item any $\omega \in \Omega''$ has even number of ones;
      \item and it holds that
      \begin{align*}
          |\Omega''| \ge 1 + 2^{\frac{1}{8}\binom{\nclusters}{2}} / |\mathbb{S}_\nclusters|.
      \end{align*}
    \end{itemize}
  \end{lemma}

  \begin{proof}
    Define a map $T_{\permutationMatrix}\colon \Omega \to \Omega$ as follows:
    \begin{align*}
        T_{\permutationMatrix}(\omega) = \probeMatrixFunction^{-1} \left( \permutationMatrix \probeMatrixFunction(\omega) \permutationMatrix^\T \right).
    \end{align*}
    Additionally, define the set $\mathcal{O}_\omega$ as
    \begin{align*}
        \mathcal{O}_\omega = \left \{ \omega' \mid \exists \permutationMatrix \in \mathbb{S}_\nclusters \text{ s.t. } d_{H}(T_{\permutationMatrix}(\omega), \omega') \le \frac{1}{17} \binom{\nclusters}{2} \right \}.
    \end{align*}
    We claim that for any $\omega \in \Omega'$ we have
    \begin{align}
    \label{eq: intersection with Omega}
        |\mathcal{O}_\omega \cap \Omega'| \le \nclusters!.
    \end{align}
    Indeed, if $|\mathcal{O}_\omega \cap \Omega'| > \nclusters!$ then there exists a permutation $\permutationMatrix_0$ such that $d_H(\omega_1, T_{\permutationMatrix_0}(\omega)) \le \frac{1}{17} \binom{\nclusters}{2}$ and $d_H(\omega_2, T_{\permutationMatrix_0}(\omega)) \le \frac{1}{17} \binom{\nclusters}{2}$ for two distinct $\omega_1, \omega_2 \in \Omega'$. By the triangle inequality, that implies $d_H(\omega_1, \omega_2) \le \frac{2}{17} \binom{\nclusters}{2}$ which contradicts the definition of $\Omega'$.

    We construct a set $\tilde \Omega$ iteratively by the following procedure.
    \begin{algorithmic}[1]
        \State Set $\widehat{\Omega} = \Omega' \setminus \{\mathbf{0}\}$, $\tilde \Omega = \varnothing$
        \Repeat
        \State Choose $\omega \in \widehat{\Omega}$
        \State $\tilde \Omega := \tilde \Omega \cup \{\omega\}$
        \State $\widehat{\Omega} := \widehat{\Omega} \setminus \mathcal{O}_\omega$
        \Until{$\widehat{\Omega} = \varnothing$}
        \State $\tilde \Omega:= \tilde \Omega \cup \{\mathbf{0}\}$
    \end{algorithmic}
    Due to~\eqref{eq: intersection with Omega}, the loop will make at least $2^{\frac{1}{8}\binom{\nclusters}{2}} / |\mathbb{S}_\nclusters|$ iteration. Thus, we have
    \begin{align*}
        |\tilde \Omega| \ge 1 + 2^{\frac{1}{8}\binom{\nclusters}{2}} / |\mathbb{S}_\nclusters|.
    \end{align*}
    We only should check that for two distinct $\omega_1, \omega_2 \in \tilde \Omega$ we have
    \begin{align*}
        \min_{\permutationMatrix \in \mathbb{S}_\nclusters} d_H(T_{\permutationMatrix}(\omega_1), \omega_2) \ge \frac{1}{17} \binom{\nclusters}{2}.
    \end{align*}
    Assume that the opposite holds. Then, $\omega_1 \in \mathcal{O}_{\omega_2}$ and $\omega_2 \in \mathcal{O}_{\omega_1}$. If $\omega_1, \omega_2$ are non-zero that is impossible by the construction of $\tilde \Omega$. Without loss of generality, assume that $\omega_1 = \mathbf{0}$. Then, for any $\permutationMatrix \in \mathbb{S}_\nclusters$, we have
    \begin{align*}
        d_H(T_{\permutationMatrix}(\omega_1), \omega_2) = d_H(\omega_1, \omega_2) \ge \frac{1}{8} \binom{\nclusters}{2}
    \end{align*}
    by the definition of $\Omega'$, the contradiction.

    Then, we obtain $\Omega''$ from $\tilde \Omega$ as follows. For each $\omega \in \tilde \Omega$, we change $\omega_{\{K - 1, K\}}$ to $1 - \omega_{\{K - 1, K\}}$ if the number of ones in $\omega$ is odd. For any distinct $\omega_1, \omega_2 \in \tilde \Omega$, it reduces the quantity $
        \min_{\permutationMatrix \in \mathbb{S}_\nclusters} d_H(T_{\permutationMatrix}(\omega_1), \omega_2)
    $
    by two at most.
  \end{proof}

\subsection{Bounding KL-divergence} Next, for each $\constCommunityMatrix \in \mathcal{B} = \{\probeMatrixFunction(\omega) \mid \omega \in \Omega''\}$ we construct the same matrix of memberships $\nodeCommunityMatrix$. For each community, it has $\max\{1, \lfloor n^{\alpha}   / K\rfloor\}$ pure nodes. The other nodes have memberships equally distributed between communities: $\nodeWeights_i = \ones / \nclusters$ for each $i \not \in \pureNodesSet$. Thus, we obtain $|\Omega''|$ matrices of connection probabilities $\probMatrix^\omega = \rho \nodeCommunityMatrix \probeMatrixFunction(\omega) \nodeCommunityMatrix^\T$, $\omega \in \Omega''$. The induced distribution on graphs we define by $\PP_\omega$.

\begin{lemma}
    \label{lemma: KL divergence bound for few pure nodes}
    We have $\KL(\PP_\omega \Vert \PP_0) \le 32 \rho K \mu^2 / n^{1 - \alpha}$.
\end{lemma}

\begin{proof}
We bound the KL-divergence as follows:
\begin{align*}
    \KL(\PP_\omega \Vert \PP_0) 
    & = \sum_{1 \le i \le j \le n} 
        \KL (Bern(\probMatrix^\omega_{ij}) \Vert Bern(\probMatrix^{(0)}_{ij}) ) \\
    & \le \sum_{1 \le i \le j \le n} \frac{(\probMatrix^\omega_{ij} - \probMatrix^{(0)}_{ij} )^2}{\probMatrix^{(0)}_{ij}} + \frac{(\probMatrix^\omega_{ij} - \probMatrix^{(0)}_{ij} )^2}{1 - \probMatrix^{(0)}_{ij}},
\end{align*}
where we apply the fact that KL-divergence does not exceed chi-square divergence.

Since $\probMatrix^{(0)}_{ij}$ is some convex combination of entries of $\rho \probeMatrixFunction(\zero)$, we have $\probMatrix^{(0)}_{ij} \in [\rho/4; \rho/2]$. Thus, both $\probMatrix^{(0)}_{ij}$ and $1 - \probMatrix^{(0)}_{ij}$ are at least $\rho/4$, and
\begin{align*}
    \KL(\PP_\omega \Vert \PP_0) \le \frac{8}{\rho} \sum_{1 \le i \le j \le n} \left ( \probMatrix_{ij}^\omega - \probMatrix^{(0)}_{ij} \right )^2
\end{align*}
holds.

We distinguish three cases: $i, j \in \pureNodesSet$, only one of $i, j$ in $\pureNodesSet$, and both $i, j$ are not pure. If $i, j \in \pureNodesSet$, then, we have for some $k, k'$:
\begin{align*}
    \left (\probMatrix^\omega_{ij} - \probMatrix^{(0)}_{ij} \right )^2 = \rho^2 (\ev_k^\T (\probeMatrixFunction(\omega) - \probeMatrixFunction(\zero)) \ev_{k'})^2 
    \le \frac{\mu^2 \rho^2}{n^2}.
\end{align*}
We obtain the same bound if only one of $i, j$ in $\pureNodesSet$. If both $i, j$ are not pure, then $\nodeWeights_i = \nodeWeights_j = \ones / K$ by the construction, and
\begin{align*}
    \left (\probMatrix^\omega_{ij} - \probMatrix^{(0)}_{ij} \right )^2 = \left ( \ones^\T (\probeMatrixFunction(\omega) - \probeMatrixFunction(\zero)) \ones \right)^2 / K^4 = \frac{\mu^2}{n^2 K^4}\left ( \sum_{S \in \binom{[K]}{2}} \omega_S \ttb_S^\omega \right )^2 = 0,
\end{align*}
since $\omega \in \Omega''$ has the odd number of ones, and $\ttb^\omega \in \{-1,  1\}^{\binom{K}{2}}$ was chosen to minimize $\left | \sum_{S} \omega_S \ttb_S^\omega \right |$, which minimum is clearly zero. Hence, we have
\begin{align*}
    \KL(\PP_\omega, \PP_0) & \le \frac{8}{\rho}\sum_{i, j \in \mathcal{P}} \left ( \probMatrix_{ij}^\omega - \probMatrix^{(0)}_{ij}\right )^2 + \frac{16}{\rho} \sum_{i \in \pureNodesSet, j \not \in \pureNodesSet}\left ( \probMatrix_{ij}^\omega - \probMatrix^{(0)}_{ij}\right )^2 + \frac{8}{\rho}\sum_{i, j \not \in \pureNodesSet}  \left ( \probMatrix_{ij}^\omega - \probMatrix^{(0)}_{ij}\right )^2 \\
    & \le \frac{16 \rho^2 \mu^2}{\rho n^2} (\max\{K^2, n^{2 \alpha}\} + K n^{1 + \alpha} ) \le \frac{32 K \rho \mu^2}{n^{1 - \alpha}}. 
\end{align*}
\end{proof}

\subsection{Proof of Theorem~\ref{theorem: lower bound with almost no pure nodes}}

We distinguish two cases. The first one is when $\nclusters \ge 512$, and the second one is when $2 \le K \le 511$. For a reminder, we have defined $T_{\permutationMatrix} = \probeMatrixFunction^{-1} (\permutationMatrix \probeMatrixFunction(\omega) \permutationMatrix^\T),
  \mathtt{T}(\omega) = \probeMatrixFunction(\omega) - \probeMatrixFunction(\zero)$.

\textbf{Case 1. Suppose that $\nclusters \ge 512$.} 
Let $\Omega''$ be the set obtained from Lemma~\ref{lemma: factorization}. We define the desired set $\BC$ as follows:
  \begin{align*}
    \BC = \probeMatrixFunction(\Omega'').
  \end{align*}
  Since $\probeMatrixFunction(\cdot)$ is injection, we have
  \begin{align*}
    |\BC| \ge 1 + 2^{\frac{1}{8} \binom{\nclusters}{2}} / |\mathbb{S}_\nclusters|.
  \end{align*}


  First, we bound $\min_{\permutationMatrix \in \mathbb{S}_\nclusters} \Vert \permutationMatrix (\sparsityParam \constCommunityMatrix_1) \permutationMatrix^\T - \sparsityParam \constCommunityMatrix_2 \Vert_\F$ for two distinct $\constCommunityMatrix_1, \constCommunityMatrix_2 \in \BC$. Let $\omega_1, \omega_2$ be such that $\constCommunityMatrix_i = \probeMatrixFunction(\omega_i)$ for each $i \in \{1, 2\}$. We have
  \begin{align*}
    \Vert \permutationMatrix \constCommunityMatrix_1 \permutationMatrix^\T - \constCommunityMatrix_2 \Vert_\F^2 & = \frac{\mu^2}{\nsize^2}
        \Vert 
            \permutationMatrix \mathtt{T}(\omega_1) \permutationMatrix^\T - \mathtt{T}(\omega_2)
        \Vert_\F^2 \nonumber \\
    & = \frac{\mu^2}{\nsize^2}
    \Vert 
        \mathtt{T}(T_{\permutationMatrix}(\omega_1)) - \mathtt{T}(\omega_2)
    \Vert_\F^2 \nonumber \\
    & =
    \frac{2 \mu^2}{\nsize^2} d_H(T_{\permutationMatrix}(\omega_1), \omega_2).
  \end{align*}
  Due to Lemma~\ref{lemma: factorization}, we have
  \begin{align}
  \label{eq: lower bound risk permuted for few pure nodes}
    \min_{\permutationMatrix \in \mathbb{S}_\nclusters}
    \Vert \permutationMatrix \constCommunityMatrix_1 \permutationMatrix^\T - \constCommunityMatrix_2 \Vert_\F \ge  \frac{\mu}{n} \sqrt{2 \left (\frac{1}{17} \binom{K}{2} - 2 \right )} \ge \frac{\mu K}{n\sqrt{34}},
  \end{align}
  where we use $\binom{K}{2} \ge 34$.
  We apply Lemma~\ref{lemma: fano lemma} with $\alpha = 1/16$. Due to Lemma~\ref{lemma: KL divergence bound for few pure nodes}, we should choose $\mu$ such that
  \begin{align*}
    32 \mu^2 K \sparsityParam / n^{1 - \alpha} \le \frac{1}{16} \log \frac{2^{\frac{1}{8} \binom{\nclusters}{2}}}{|\mathbb{S}_\nclusters|}.
  \end{align*}
  Since $\nclusters \ge 512$, we have
  \begin{align*}
    \log_2 \frac{2^{\frac{1}{8} \binom{\nclusters}{2}}}{|\mathbb{S}_\nclusters|} \ge \frac{1}{8} \binom{\nclusters}{2} - \nclusters \log_2 \nclusters \ge \frac{1}{16} \binom{\nclusters}{2} + \nclusters \left( \frac{\nclusters - 1}{32} - \log_2 \nclusters\right) \ge \frac{1}{16} \binom{\nclusters}{2}.
  \end{align*}
  Hence, $|\BC| \ge 1 + 2^{\binom{\nclusters}{2} / 16}$, and it is enough to satisfy the following inequality:
  \begin{align*}
    32 \mu^2  K\sparsityParam / n^{1 - \alpha} \le \frac{\log 2}{256} \binom{\nclusters}{2}.
  \end{align*}
  We choose $\mu = \sqrt{n^{1 - \alpha} \nclusters  / (150 \rho)}$. We substitute $\mu$ to~\eqref{eq: lower bound risk permuted for few pure nodes}, then apply Lemma~\ref{lemma: fano lemma}, and obtain the result.

  \textbf{Case 2. Suppose that $2 \le K \le 511$.} Choose $\omega \in \Omega$ such that
  \begin{align*}
    \sum_{S \in \binom{[K]}{2}} \omega_{S} \ge \frac{K^2}{4}
  \end{align*}
  and $\sum_{S \in \binom{K}{2}} \omega_S \ttb_S^\omega = 0$.
  Then, we have
  \begin{align*}
    \Vert \probeMatrixFunction(\zero) - \permutationMatrix \probeMatrixFunction(\omega) \permutationMatrix^\T \Vert_{\F} = \frac{\mu}{\nsize} \Vert \permutationMatrix \mathtt{T}(\omega) \permutationMatrix^\T \Vert_{\F} = \frac{\mu}{\nsize} \sqrt{2 \sum_{S \in \binom{[K]}{2}} \omega_S} \ge \frac{ \mu K}{\sqrt{2} \nsize}.
  \end{align*}
  Let $\PP_\omega, \PP_0$ be distributions defined by the matrices of connection probabilities $\sparsityParam \nodeCommunityMatrix_0 \probeMatrixFunction(\omega) \nodeCommunityMatrix_0^\T,$ \newline $ \sparsityParam \nodeCommunityMatrix_0 \probeMatrixFunction(\zero) \nodeCommunityMatrix_0^\T$ respectively. Then, due to Lemma~\ref{lemma: KL divergence bound for few pure nodes}, we have
  \begin{align*}
    \KL(\PP_{\omega}, \PP_0) \le 32 \mu^2 K \sparsityParam / n^{1 - \alpha}.
  \end{align*}
  Define $\BC = \{\probeMatrixFunction(\omega), \probeMatrixFunction(\zero)\}$. We choose $\mu = (n^{1 - \alpha} \nclusters / (10 \cdot 511^2 \cdot \rho))^{1/2}$. Since $K \le 511$, we have $\KL(\PP^{\omega}, \PP_0) \le 3.2$. Next, we apply Lemma~\ref{lemma: minimax bounds for 2 hypotheses}, and obtain
  \begin{align*}
    \inf_{\estimator[\communityMatrix]} \sup_{\constCommunityMatrix \in \BC} \PP \left ( 
        \min_{\permutationMatrix \in {\mathbb{S}_\nclusters}} \Vert \estimator[\communityMatrix] - \sparsityParam \permutationMatrix \constCommunityMatrix \permutationMatrix^\T \Vert_{\F} \ge \frac{1}{3066} \sqrt{\frac{\rho K^3}{n^{1 + \alpha}}}
    \right ) \ge \frac{1}{4} e^{-3.2}.
  \end{align*}

\section{Proof of Theorem~\ref{theorem: lower bound}}
\label{section: proof of the lower bound under conditions}
\subsection{Constructing hypotheses}

The goal of this section is to construct two distributions $\PP_0$ and $\PP_1$ that satisfies Conditions~\ref{cond: nonzero B elements}-\ref{cond: theta distribution-b} and have small KL-divergence. Suppose that distributions $\PP_0$ and $\PP_1$ are determined by community matrices $\communityMatrix_0 = \sparsityParam \constCommunityMatrix_0$, $\communityMatrix_1 = \sparsityParam \constCommunityMatrix_1$ and memebrship matrices $\nodeCommunityMatrix_0$ and $\nodeCommunityMatrix_1$.

The most restrictive condition is $\lambda_K(\nodeCommunityMatrix^\top \nodeCommunityMatrix) = \Omega(n)$. To satisfy it, we divide all $n$ nodes into four types:
\begin{enumerate}
    \item \label{item: type-1}$\lfloor n^\alpha / 4096 \rfloor$ pure nodes, that belong to the first community;
    \item \label{item: type-2} $\lfloor n^\alpha / 4096 \rfloor$ pure nodes, that belong to the second community;
    \item \label{item: type-3} $\lfloor n/2 -  \lfloor n^\alpha /4096 \rfloor \rfloor$ nodes, that have a memebership vector $\nodeWeights_1$ (which is different for $\PP_0$ and $\PP_1$);
    \item \label{item: type-4} $\lceil n/2 - \lfloor n^\alpha / 4096 \rfloor \rceil$ nodes, that have a membership vector $\nodeWeights_2$ (which is different for $\PP_0$ and $\PP_1$).
\end{enumerate}
To satisfy $\lambda_2 (\nodeCommunityMatrix^\top \nodeCommunityMatrix) = \Omega(n)$ it is enough to ensure that vectors $\nodeWeights_1$ and $\nodeWeights_2$ are independent. 

In the case of the distribution $\PP_0$ we set $\nodeWeights_1 = \nodeWeights_1^{0} = (1/4, 3/4)$ and $\nodeWeights_2 = \nodeWeights_2^0 = (3/4, 1/4)$. In the case of the distribution $\PP_1$, we introduce a real number $\eta$, and set $\nodeWeights_1 = \nodeWeights_1^0 + \eta (-1, 1) = (1/4 - \eta, 3/4 + \eta)$ and $\nodeWeights_2 = \nodeWeights_2^0 + \eta (-1, 1) = (3/4 - \eta, 1/4 + \eta)$.

  We can provide a sufficient upper bound on KL-divergence $\KL(\PP_1 \Vert \PP_0)$, if the following 3 equations are satisfied:
\begin{align}
\label{eq: system on theta}
\left ( \nodeWeights_k^0 + \eta \begin{pmatrix}
    -1 \\
    1
\end{pmatrix} \right )^\T \constCommunityMatrix_1 \left ( \nodeWeights_{k'}^0 + \eta \begin{pmatrix}
    -1 \\
    1
\end{pmatrix}  \right ) - ( \nodeWeights_k^0)^\T \constCommunityMatrix_0 \nodeWeights_{k'}^0 = 0 \quad \text{ for all } k, k' \in [2].
\end{align}
Set
\begin{align*}
    \constCommunityMatrix_0 = \begin{pmatrix}
        1/2& 1/4 \\
        1/4 & 1/2
    \end{pmatrix}.
\end{align*}
Note that the system~\eqref{eq: system on theta} is linear in $ \constCommunityMatrix_0 - \constCommunityMatrix_1$. To rewrite it in the matrix form, we define a vector 
\begin{align*}
    \bv = \begin{pmatrix}
        (\constCommunityMatrix_1 - \constCommunityMatrix_0)_{11} \\
        (\constCommunityMatrix_1 - \constCommunityMatrix_0)_{12} \\
        (\constCommunityMatrix_1 - \constCommunityMatrix_0)_{22}
    \end{pmatrix}.
\end{align*}
Then, the system~\eqref{eq: system on theta} can be restated as follows:
\begin{align*}
    \left ( A_0 + \eta A_1 + \eta^2 A_2 \right ) \bv = \eta \cdot \begin{pmatrix}
        -1/4 - \eta/2\\
        0 \\
        1/4 - \eta/2
    \end{pmatrix},
\end{align*}
where we denote
\begin{align*}
    A_0 = \frac{1}{16} \begin{pmatrix}
        1 & 6 & 9 \\
        3 & 10 & 3\\
        9 & 6 & 1
    \end{pmatrix}, \quad A_1 = \frac{1}{2} \begin{pmatrix}
        - 1 & -2 & 3\\
        -2 & 0 & 2\\
        - 3 & 2 & 1
    \end{pmatrix}, \quad A_2 = \begin{pmatrix}
        1 & -2 & 1\\
        1 & - 2 & 1\\
        1 & -2 & 1
    \end{pmatrix}.
\end{align*}
We obtain
\begin{align*}
    \bv = \eta \cdot (A_0 + \eta A_1 + \eta^2 A_2)^{-1} \begin{pmatrix}
        -1/4 - \eta/2\\
        0 \\
        1/4 - \eta/2
    \end{pmatrix}.
\end{align*}
In particular, we have
\begin{align*}
    \frac{\eta/4 \cdot \sqrt{2 + 8 \eta^2}}{\Vert A_0 \Vert + \eta \Vert A_1 \Vert + \eta^2 \Vert A_1 \Vert} \le \Vert \bv \Vert \le \frac{\eta/4 \cdot \sqrt{2 + 8 \eta^2}}{\sigma_{\min}(A_0) - \eta \Vert A_1 \Vert - \eta^2 \Vert A_2 \Vert}.
\end{align*}
Using $1/5 \le \sigma_{\min}(A_0) \le \Vert A_0 \Vert \le 2$, $\Vert A_1 \Vert \le 3$, $\Vert A_2 \Vert \le 5$ and $\Vert \bv \Vert \le \Vert \constCommunityMatrix_1 - \constCommunityMatrix_0 \Vert_{\F} \le 2 \Vert \bv \Vert$, we get
\begin{align*}
    \frac{\eta/ 4 \cdot \sqrt{2 + 8 \eta^2}}{2 + 3 \eta + 5 \eta^2} \le \Vert \constCommunityMatrix_1 - \constCommunityMatrix_0 \Vert_{\F} \le \frac{\eta/2 \cdot \sqrt{2 + 8 \eta^2}}{1/5 - 3 \eta - 5 \eta^2}.
\end{align*}
We will choose the specific value of $\eta$ in the next section. From now, we assume that $\eta \le 1/100$, so we have
\begin{align}
\label{eq: bounds on the Frobenius norm of community matrix}
    \frac{\eta}{12} \le \Vert \constCommunityMatrix_1 - \constCommunityMatrix_0 \Vert_{\F} \le 10 \eta.
\end{align}
Note that for any permutation matrix $\permutationMatrix$, we have $\permutationMatrix \constCommunityMatrix_0 \permutationMatrix^\T  = \constCommunityMatrix_0$, so
\begin{align}
\label{eq: lower bound on the loss 2 communities case}
    \min_{\permutationMatrix \in \mathbb{S}_2} \Vert \communityMatrix_1 - \permutationMatrix \communityMatrix_0 \permutationMatrix^\T \Vert_\F = \rho \Vert \constCommunityMatrix_1 - \constCommunityMatrix_0 \Vert_\F \ge \frac{\rho \eta}{12}.
\end{align}

\subsection{Bounding KL-divergence} Next, we bound the KL-divergence $\KL(\PP_1 \Vert \PP_0)$ between $\PP_0$ and $\PP_1$. We define $\probMatrix^0 = \sparsityParam \nodeCommunityMatrix_0 \constCommunityMatrix_0 \nodeCommunityMatrix_0^\T$ and $\probMatrix^1 = \sparsityParam \nodeCommunityMatrix_1 \constCommunityMatrix_1 \nodeCommunityMatrix_1^\T$. We have
\begin{align*}
    \KL (\PP_1 \Vert \PP_0) & \le  \sum_{1 \le i \le j \le n} 
        \KL (Bern(\probMatrix^1) \Vert Bern(\probMatrix^0_{ij}) ) \\
    & \le \sum_{1 \le i \le j \le n} \frac{(\probMatrix^1_{ij} - \probMatrix^{0}_{ij} )^2}{\probMatrix^{0}_{ij}} + \frac{(\probMatrix^1_{ij} - \probMatrix^{0}_{ij} )^2}{1 - \probMatrix^{0}_{ij}},
\end{align*}
where we used the fact that the KL-divergence does not exceed chi-square divergence. Note that elements of $\probMatrix^0$ are convex combinations of elements of $\rho \constCommunityMatrix_0$. Therefore, for each $i, j$ we have $\probMatrix^0_{ij} \in [\rho/4; \rho/2]$. It yields
\begin{align}
\label{eq: KL sum for 2 commnitites}
    \KL (\PP_1 \Vert \PP_0) \le \frac{8}{\rho} \sum_{1 \le i \le j \le n} (\probMatrix^1_{ij} - \probMatrix^{0}_{ij} )^2.
\end{align}
In the previous section, we divided all nodes into four types \ref{item: type-1}-\ref{item: type-4}. We denote the set of nodes belonging to type $\ell$ by $\mathcal{T}_\ell$. Next, we decompose the sum~\eqref{eq: KL sum for 2 commnitites} into $16$ summands, each summand corresponds to one pair of types:
\begin{align*}
    \KL (\PP_1 \Vert \PP_0) \le \frac{8}{\sparsityParam} \sum_{\ell, \ell'} \sum_{i \in \mathcal{T}_\ell, j \in \mathcal{T}_{\ell'}} (\probMatrix^1_{ij} - \probMatrix^{0}_{ij} )^2.
\end{align*}
If either $i$ or $j$ belongs to types~\ref{item: type-1}-\ref{item: type-2}, using~\eqref{eq: bounds on the Frobenius norm of community matrix}, we bound 
\begin{align*}
    (\probMatrix^1_{ij} - \probMatrix^{0}_{ij} )^2 \le \rho^2 \Vert \constCommunityMatrix_1 - \constCommunityMatrix_0 \Vert^2_\F \le 100 \rho^2 \eta^2 
\end{align*}
Next, we consider the case one $i \in \mathcal{T}_\ell$ for $\ell \in \{3, 4\}$ and $j \in \mathcal{T}_{\ell'}$ for $\ell' \in \{3, 4\}$. Then, we have
\begin{align*}
    \probMatrix^1_{ij} - \probMatrix^0_{ij} = \rho \left [ \left (\nodeWeights_{\ell - 2}^0 + \eta \begin{pmatrix}
        -1 \\
        1
    \end{pmatrix}\right )^\T \constCommunityMatrix_1 \left (\nodeWeights_{\ell' - 2}^0 + \eta  \begin{pmatrix}
        -1 \\
        1
    \end{pmatrix}\right ) - (\nodeWeights^0_{\ell - 2})^\T \constCommunityMatrix_0 \nodeWeights_{\ell' - 2}^0 \right ] = 0,
\end{align*}
since the system~\eqref{eq: system on theta} is satisfied by construction of $\constCommunityMatrix_1$. Thus, we have
\begin{align*}
    \KL (\PP_1 \Vert \PP_0) \le \frac{16 \rho^2}{\rho} \cdot n \cdot \frac{n^\alpha}{20} \cdot 100 \eta^2 = 80 \rho n^{1 + \alpha} \eta^2.
\end{align*}
We set $\eta = (80 \rho n^{1 + \alpha})^{-1/2}$, which is less than $1/100$ provided $\rho n^{1 + \alpha}$ is larger than some constant. Due to~\eqref{eq: lower bound on the loss 2 communities case}, it yields
\begin{align*}
    \min_{\permutationMatrix \in \mathbb{S}_2} \Vert \communityMatrix_1 - \permutationMatrix \communityMatrix_0 \permutationMatrix^\T\Vert_{\F} \ge \frac{\rho \eta}{12} \ge \frac{1}{12 \cdot \sqrt{80}} \sqrt{\frac{\rho}{n^{1 + \alpha}}} \ge \frac{1}{108} \sqrt{\frac{\rho}{n^{1 + \alpha}}}.
\end{align*}
Note that for this choice of $\eta$, we have $\KL(\PP_1 \Vert \PP_0) \le 1$. Applying Lemma~\ref{lemma: minimax bounds for 2 hypotheses}, we deduce the lower bound stated in Theorem~\ref{theorem: lower bound}. It remains to check that properties~\ref{theorem lower bound, property i}-\ref{theorem lower bound, property iv} are satisfied.

\subsection{Checking the properties}

The matrix $\constCommunityMatrix_0$ has singular values $3/4$ and $1/4$. Next, we may bound the singular numbers of $\constCommunityMatrix_1$ by $\sigma_{2}(\constCommunityMatrix_0) - \Vert \constCommunityMatrix_1 - \constCommunityMatrix_0 \Vert$, which is at most $1/4 - 10 \eta \ge 1/8$ due to~\eqref{eq: bounds on the Frobenius norm of community matrix}, so property~\ref{theorem lower bound, property i} is satisfied.

Next, we check the diverging spiked eigenvalue property of $\probMatrix_\ell$, $\ell \in \{0, 1\}$. We start from the matrix $\probMatrix_0$, and decompose it as follows. Set $m = \lfloor n/2 - \lfloor n^\alpha /4096 \rfloor \rfloor$. Let $\ones_m$ be a vector of length $m$ which entries are equal to 1. Then, we represent the matrix $\probMatrix_0$ as the following sum:
\begin{align*}
    \probMatrix_0 = \left [\begin{pmatrix}
        (\nodeWeights_1^0)^\T \communityMatrix_0 \nodeWeights_1^0 & (\nodeWeights_1^0)^\T \communityMatrix_0 \nodeWeights_2^0 \\
        (\nodeWeights_2^0)^\T \communityMatrix_0 \nodeWeights_1^0 & (\nodeWeights_2^0)^\T \communityMatrix_0 \nodeWeights_2^0
    \end{pmatrix} \otimes \ones_m \ones^\T_m \right ] \oplus \mathbf{O}_{n - 2 m} + \mathbf{R},
\end{align*}
where we grouped elements $i, j \in \mathcal{T}_3 \cup \mathcal{T}_4$ in up-left corner, and $\mathbf{O}_{n - 2 m}$ is a $(n - 2m) \times (n - 2m)$ matrix consisting of zeros and $\mathbf{R}$ is some matrix with non-zero values either in the last $n - 2m$ columns or in the last $n - 2m$ rows. Therefore, $\Vert \mathbf{R} \Vert \le \Vert \mathbf{R} \Vert_{\F} \le \rho \sqrt{2 n (n - 2m)} \le \rho \sqrt{n^{1 + \alpha} /2048 + 2 n}$. Then, we compute the singular values of the matrix
\begin{align*}
    \begin{pmatrix}
        (\nodeWeights_1^0)^\T \communityMatrix_0 \nodeWeights_1^0 & (\nodeWeights_1^0)^\T \communityMatrix_0 \nodeWeights_2^0 \\
        (\nodeWeights_2^0)^\T \communityMatrix_0 \nodeWeights_1^0 & (\nodeWeights_2^0)^\T \communityMatrix_0 \nodeWeights_2^0
    \end{pmatrix} = \rho\cdot \begin{pmatrix}
        13/32 & 11/32 \\
        11/32 & 13/32\\
    \end{pmatrix},
\end{align*}
which are $3\rho/4$ and $\rho/16$. We have, provided $n$ is larger than some constant,
\begin{align*}
    \sigma_1(\probMatrix_0) & \ge \frac{3\rho m}{4} - \Vert \mathbf{R} \Vert \ge \frac{3 \rho m}{4} - \frac{\rho n}{32} \ge \frac{3\rho n}{16} - \frac{\rho n}{32} \ge \frac{5\rho n}{16}, \\
    \sigma_2(\probMatrix_0) & \le \frac{\rho m}{16} + \Vert \mathbf{R} \Vert \le \frac{\rho m}{16} + \frac{\rho n}{32} \le \frac{\rho n}{16}.
\end{align*}
Hence, we have $\sigma_1(\probMatrix_0) / \sigma_2(\probMatrix_0) \ge 5$. Similarly, using $\Vert \constCommunityMatrix_1 - \constCommunityMatrix_0 \Vert_\F \le 10 \eta$ from~\eqref{eq: bounds on the Frobenius norm of community matrix}, we get
\begin{align*}
    \sigma_1(\probMatrix_1) & \ge (3/4 - 20 \eta) \rho m - \Vert \mathbf{R} \Vert \ge \frac{\rho n}{4}, \\
    \sigma_2(\probMatrix_1) & \le (1/16 + 20 \eta) \rho m + \Vert \mathbf{R} \Vert \le \frac{7 \rho n}{32},
\end{align*}
so we have $\sigma_1(\probMatrix_1) / \sigma_2(\probMatrix_1) \ge 8/7$, and the first part of property~\ref{theorem lower bound, property ii} holds. To establish the second part, we note that
\begin{align*}
    \max_j \sum_{i = 1}^\nsize \probMatrix_{ij} \left ( 1 - \probMatrix_{ij} \right) \ge \frac{13 \rho m}{32} \cdot \frac{1}{2} \ge \frac{\rho \nsize}{16}.
\end{align*}

Then, we move on the proof of property~\ref{theorem lower bound, property iii}. The lower bound on $|\pureNodesSet_k|$ holds by construction. Then, we prove the lower bound on the second eignevalue of $\nodeCommunityMatrix^\T \nodeCommunityMatrix$. We have
\begin{align*}
    \nodeCommunityMatrix^\T \nodeCommunityMatrix \succeq \sum_{i \in \mathcal{T}_3} \nodeWeights_i \nodeWeights_i^\T + \sum_{i \in \mathcal{T}_4} \nodeWeights_i \nodeWeights_i^\T \succeq \min\{|\mathcal{T}_3|, |\mathcal{T}_4|\} \left ( \bm{\alpha} \bm{\alpha}^\T + \bm{\beta} \bm{\beta}^\T  \right ),
\end{align*}
where $\bm \alpha = \nodeWeights_i$ for any $i \in \mathcal{T}_3$ and $\bm \beta = \nodeWeights_j$ for any $j \in \mathcal{T}_4$, since vectors from $\mathcal{T}_3$ are the same as well as vectors from $\mathcal{T}_4$ for both models $\PP_0$ and $\PP_1$. 
Note that both $|\mathcal{T}_3|$ and $|\mathcal{T}_4|$ has size linear in $n$, so it is enough to check that the matrix $\bm \alpha \bm \alpha^\T + \bm \beta \bm \beta^\T$ has the least singular value bounded below by some constant.

For the model $\PP_0$, we have $\bm \alpha \bm \alpha^\T + \bm \beta \bm \beta^\T$ equals to the following matrix
\begin{align*}
    \bm \alpha \bm \alpha^\T + \bm \beta \bm \beta^\T = (\nodeWeights_1^0) (\nodeWeights_1^0)^\T + (\nodeWeights_2^0) (\nodeWeights_2^0)^\T = \begin{pmatrix}
        5/8 & 3/8 \\
        3/8 & 5/8
    \end{pmatrix},
\end{align*}
which least singular value equals $1/4$. For the model $\PP_1$, we have
\begin{align*}
    \bm \alpha \bm \alpha^\T + \bm \beta \bm \beta^\T = \begin{pmatrix}
        5/8 - 2 \eta  + 2 \eta^2 & 3/8 - 2 \eta^2 \\
        3/8 - 2 \eta^2 & 5/8 + 2 \eta + 2 \eta^2
    \end{pmatrix}.
\end{align*}
Applying Weyl's inequality, we get $\sigma_{\min} (\bm \alpha \bm \alpha^\T + \bm \beta \bm \beta^\T) \ge 1/4 - 2 \eta - 4 \eta^2$. Since $\eta \le 1/100$, the latter is at least $21/100$, so $\sigma_{\min}(\nodeCommunityMatrix^\T \nodeCommunityMatrix) \ge C n$ for some absolute constant $C$, and the property~\ref{theorem lower bound, property iii} holds. 

Finally, we verify property~\ref{theorem lower bound, property iv}. We claim that
\begin{align*}
    \sum_{j \not \in \pureNodesSet_k} \indicator[\Vert \nodeWeights_j - \ev_k \Vert \le \delta \sqrt{\frac{\log \nsize}{\nsize \sparsityParam}}] = 0,
\end{align*}
provided $n$ is larger than some function of $\delta$. This clearly holds by the construction of membership vectors. Thus, for any $n$, we can bound
\begin{align*}
    \sum_{j \not \in \pureNodesSet_k} \indicator[\Vert \nodeWeights_j - \ev_k \Vert \le \delta \sqrt{\frac{\log \nsize}{\nsize \sparsityParam}}] \le C(\delta),
\end{align*}
where $C(\delta)$ is some constant depending on $\delta$ only.

\section{Tools and supplementary lemmas for Theorem~\ref{theorem: main result}}

\subsection{Supplementary lemmas}
\subsubsection{Efficient estimation of eigenvalues}

\begin{lemma}
    \label{lemma: second order term estimation}
  Under Conditions~\ref{cond: nonzero B elements}-\ref{cond: theta distribution-a}, we have
  \begin{align*}
    \frac{\uv_k^{\T} \EE \diagAdjecencyMatrix \uv_{k'}}{t_k^2}
    -
    \frac{\widehat{\uv}_k^{\T} \diagAdjecencyMatrix \widehat{\uv}_{k'}}{ \adjacencyEigenvalues^{2}_{k k}} = O_\ell \left( \sqrt{\frac{\log \nsize}{\nsize^3 \sparsityParam^3}}\right),
  \end{align*}
  for any not necessarily distinct $k, k'$.
\end{lemma}

\begin{proof}
  We decompose the initial difference in the following way:
  \begin{align*}
    \frac{\uv_k^{\T} \EE \diagAdjecencyMatrix \uv_{k'}}{t_k^2}
    -
    \frac{\widehat{\uv}_k^{\T} \diagAdjecencyMatrix \widehat{\uv}_{k'}}{ \adjacencyEigenvalues^{2}_{k k}}
    & =
    \left(\frac{\uv_k^{\T} \EE \diagAdjecencyMatrix \uv_{k'}}{t_k^2}
    -
    \frac{\widehat{\uv}_k^{\T} \EE \diagAdjecencyMatrix \widehat{\uv}_{k'}}{t_k^2} \right)
    +
    \left(\frac{\widehat{\uv}_k^{\T} \EE \diagAdjecencyMatrix \widehat{\uv}_{k'}}{t_k^2}
    -
    \frac{\widehat{\uv}_k^{\T} \diagAdjecencyMatrix \widehat{\uv}_{k'}}{t_k^2} \right) \\
    & \quad +
    \left(
      \frac{\widehat{\uv}_k^{\T} \diagAdjecencyMatrix \widehat{\uv}_{k'}}{t_k^2}
      -
      \frac{\widehat{\uv}_k^{\T} \diagAdjecencyMatrix \widehat{\uv}_{k'}}{ \adjacencyEigenvalues^{2}_{k k}}
    \right) \\
    & = : \Delta_1 + \Delta_2 + \Delta_3.
  \end{align*}
  We analyze each term separately. First, from Lemma~\ref{lemma: adj_eigenvectors_displacement}, we have
  \begin{align*}
    \adjacencyEigenvectors_{ik} \adjacencyEigenvectors_{ik'} & = \probEigenvectors_{ik} \probEigenvectors_{i k'} + \probEigenvectors_{i k'} \frac{\displaceMatrix_i \uv_k}{t_k} + \probEigenvectors_{i k}\frac{\displaceMatrix_i \uv_{k'}}{t_{k'}} + (\probEigenvectors_{ik} + \probEigenvectors_{i k'} ) \cdot O_{\prec} \Bigl ( \frac{1}{\nsize \sparsityParam \sqrt{\nsize}} \Bigr ) \\
    & \quad + \frac{\displaceMatrix_i \uv_k}{t_k}  \cdot \frac{\displaceMatrix_i \uv_{k'}}{t_{k'}}.
  \end{align*}
  Since $\probEigenvectors_{i k}, \probEigenvectors_{i k'} = O(\nsize^{-1/2})$ due to Lemma~\ref{lemma: eigenvectors max norm} and $t_k^{-1} \displaceMatrix_i \uv_k, t_{k'}^{-1} \displaceMatrix_i \uv_{k'} = O_\ell (\sqrt{\sparsityParam \log \nsize})$ due to Lemma~\ref{lemma: log estimate vector difference}, we get
  \begin{align*}
    \sum_{i = 1}^\nsize (\EE \diagAdjecencyMatrix_{ii}) (\adjacencyEigenvectors_{ik} \adjacencyEigenvectors_{ik'} - \probEigenvectors_{ik} \probEigenvectors_{i k'}) 
    & =
    \frac{1}{t_k} 
    \sum_{i = 1}^\nsize \probEigenvectors_{i k'} (\EE \diagAdjecencyMatrix_{ii}) \displaceMatrix_i \uv_k 
    \\
    & \quad +
    \frac{1}{t_{k'}} \sum_{i = 1}^\nsize \probEigenvectors_{i k} (\EE \diagAdjecencyMatrix_{ii}) \displaceMatrix_i \uv_{k'}
    + O_\prec \left( 1 \right).
  \end{align*}
  Let us analyze the first term of the right-hand side:
  \begin{align*}
      \sum_{i = 1}^\nsize \probEigenvectors_{i k'} (\EE \diagAdjecencyMatrix_{i i}) \displaceMatrix_i \uv_{k} = 2 \sum_{i = 1}^\nsize \sum_{j \le i} \left( \probEigenvectors_{i k'} \probEigenvectors_{j k} (\EE \diagAdjecencyMatrix_{ii}) \displaceMatrix_{i j} + \probEigenvectors_{i k} \probEigenvectors_{j k'} (\EE \diagAdjecencyMatrix_{j j}) \displaceMatrix_{j i} \right) \left( 1 - \frac{\delta_{i j}}{2} \right).
  \end{align*} 
  Here $\delta_{ij}$ is the Kronecker symbol. The double sum consists of $\binom{\nsize + 1}{2}$ mutually independent random variables and, thus, the Bernstein inequality can be applied. Bounding $\EE \diagAdjecencyMatrix_{ii}$, $\probEigenvectors_{j k}$ and $\Var \displaceMatrix_{i j}$ by $\nsize \sparsityParam$, $C_{\probEigenvectors} \nsize^{-1/2}$ and $\sparsityParam$ respectively, we observe
  \begin{align*}
    \sum_{i = 1}^\nsize \probEigenvectors_{i k'} (\EE \diagAdjecencyMatrix_{ii}) \displaceMatrix_i \uv_k = O_\ell \left( 
      \sqrt{\nsize^2 \sparsityParam^3 \log \nsize}
    \right).
  \end{align*}
  Analogously,
  \begin{align*}
       \sum_{i = 1}^\nsize \probEigenvectors_{i k} (\EE \diagAdjecencyMatrix_{ii}) \displaceMatrix_i \uv_{k'} = O_\ell \left( 
      \sqrt{\nsize^2 \sparsityParam^3 \log \nsize}
    \right).
  \end{align*}
  Consequently, $\Delta_1 = O_{\prec} \Bigl ( \sqrt{\rho \log \nsize} / (\nsize^2 \sparsityParam^2)\Bigr ) = O_\ell \left( \sqrt{\frac{\log \nsize}{\nsize^3 \sparsityParam^3}}\right)$. Second, we estimate $\Delta_2$. Note that
  \begin{align*}
    \EE \diagAdjecencyMatrix_{ii} - \diagAdjecencyMatrix_{ii} =
    \sum_{j = 1}^\nsize
    (
      \probMatrix_{i j} - \adjacencyMatrix_{ij} 
    ) = O_\ell ( \sqrt{\nsize \sparsityParam \log \nsize}),
  \end{align*}
  since this sum consists of bounded random variables again and, whence, its order can be established via the Bernstein inequality. Thus,
  \begin{align*}
    \frac{
        \estimator[\uv]_{k}^\T
        (\EE \diagAdjecencyMatrix - \diagAdjecencyMatrix) 
        \estimator[\uv]_{k'}
    }{t_k^2} = t_k^{-2} \sum_{i = 1}^\nsize \adjacencyEigenvectors_{i k} \adjacencyEigenvectors_{i k'} \cdot O_\ell (\sqrt{\sparsityParam \nsize \log \nsize}).
  \end{align*}
  Due to Lemma~\ref{lemma: adj eigenvectors displacement} and Lemma~\ref{lemma: eigenvectors max norm}, we have $\adjacencyEigenvectors_{ik} = \probEigenvectors_{i k} + O_{\ell} \Bigl ( \sqrt{\frac{\log \nsize}{\nsize^2 \sparsityParam}} \Bigr ) = O_{\ell} (\nsize^{-1/2})$ under Condition~\ref{cond: sparsity param bound}. Hence, we get
  \begin{align*}
      \Delta_2 = O \Bigl(\frac{1}{\nsize^2 \sparsityParam^2} \Bigr) \cdot \nsize \cdot O_\ell \Bigl( \frac{1}{\nsize} \Bigr) \cdot O_\ell (\sqrt{\sparsityParam \nsize \log \nsize}) = O_{\ell} \Biggl( \sqrt{\frac{\log \nsize}{\nsize^3 \sparsityParam^3}} \Biggr).
  \end{align*}
  Finally, we bound $\Delta_3$. Using the same arguments as above, we obtain
  \begin{equation*}
    \estimator[\uv]_k^{\T} \diagAdjecencyMatrix \estimator[\uv]_k = \sum_{i = 1}^\nsize \adjacencyEigenvectors_{ik} \adjacencyEigenvectors_{i k'} (\EE \diagAdjecencyMatrix_{ii} + (\diagAdjecencyMatrix_{ii} - \EE \diagAdjecencyMatrix_{ii})
    = \nsize \cdot O_{\ell} \bigl( \nsize^{-1} \bigr) \cdot \bigl (O(\nsize \sparsityParam) + O_\ell (\sqrt{\sparsityParam \nsize \log \nsize}) \bigr )
    = O_{\ell} (\nsize \sparsityParam).
  \end{equation*}
  So, we get $\Delta_3 = O_{\ell} (\nsize \sparsityParam) \cdot (t_k^{-2} - \adjacencyEigenvalues_{kk}^{-2})$. According to Lemma~\ref{lemma: eigenvalues difference}, we have
  \begin{align*}
      \adjacencyEigenvalues_{kk} - t_k = \uv_k^\T \displaceMatrix \uv_{k'} + O_{\prec} (\nsize^{-1/2}),
  \end{align*}
  which is $O_\ell (\sqrt{\sparsityParam \log \nsize})$ due to Lemma~\ref{lemma: log estimate vector difference} and Condition~\ref{cond: sparsity param bound}. It implies
  \begin{align*}
      \Delta_3 = O_{\ell} (\nsize \sparsityParam) \cdot (t_k^{-2} - \adjacencyEigenvalues_{kk}^{-2})
                = O_{\ell} (\nsize \sparsityParam) \cdot t_k^{-2} \adjacencyEigenvalues_{kk}^{-2} (\adjacencyEigenvalues_{kk}^2 - t_k^2)
                = O_{\ell} (\nsize \sparsityParam) \cdot t_k^{-2} \adjacencyEigenvalues_{kk}^{-2} \cdot t_k \cdot O_{\ell} (\sqrt{\sparsityParam \log \nsize}).
  \end{align*}
  Since
  \begin{align*}
      \adjacencyEigenvalues^{-2}_{kk} = t_k^{-2} \Biggl ( 1 - \frac{O_\ell (\sqrt{\sparsityParam \log \nsize})}{t_k}\Biggr )^{-2} = t_k^{-2} \bigl(1 + o(1)\bigr),
  \end{align*}
  we get
  \begin{align*}
      \Delta_3 = O_{\ell} (\nsize \sparsityParam) \cdot t_k^{-3} \cdot O_{\ell} (\sqrt{\sparsityParam \log \nsize}) = O_{\ell} \Biggl ( \sqrt{\frac{\sparsityParam \log \nsize}{\nsize^4 \sparsityParam^4 }}\Biggr ) = O_{\ell} \Biggl ( \sqrt{\frac{\log \nsize}{\nsize^3 \sparsityParam^3}} \Biggr).
  \end{align*}
  That concludes the lemma.
\end{proof}

\begin{lemma}
\label{lemma: debaised eigenvalues behaviour}
  Under Conditions~\ref{cond: nonzero B elements}-\ref{cond: theta distribution-a} it holds
  \begin{align*}
    \lambda_k(\probMatrix) - \debiasedEigenvalues_{kk} = O_{\ell} (\sqrt{\sparsityParam \log \nsize}).
  \end{align*}
\end{lemma}

\begin{proof}
  By the definition of $t_k$ in~\eqref{eq: t_k definition},
  \begin{align*}
      1 + \lambda_k(\probMatrix) 
    \left\{ 
      \resolvent(\uv_k, \uv_k, t_k) 
      - 
      \resolvent(\uv_k, \probEigenvectors_{-k}, t_k) 
      [
        \probEigenvalues^{-1}_{-k} + \resolvent(\probEigenvectors_{-k}, \probEigenvectors_{-k}, t_k)
      ]^{-1}
      \resolvent(\probEigenvectors_{-k}, \uv_k, t_k)
    \right\} = 0.
  \end{align*}
  Applying asymptotics from Lemma~\ref{lemma: auxiliary variables expansion}, we observe
  \begin{align*}
    \resolvent(\uv_k, \probEigenvectors_{-k}, t_k) 
    [
      \probEigenvalues^{-1}_{-k} + \resolvent(\probEigenvectors_{-k}, \probEigenvectors_{-k}, t_k)
    ]^{-1}
    \resolvent(\probEigenvectors_{-k}, \uv_k, t_k) = O(t_k^{-2}) \cdot O(t_k) O(t_k^{-2}) = O(t_k^{-3}),
  \end{align*}
  and, consequently,
  \begin{align}
  \label{eq: t_k - lambda_k approximation}
    1 + \lambda_k(\probMatrix) \left \{ 
      - \frac{1}{t_k} - \frac{1}{t_k^{3}} \uv_k^{\T} \EE \displaceMatrix^2 \uv_k + O(t_k^{-5/2}) + O(t_k^{-3})
    \right \} & = 0, \nonumber \\
    t_k - \lambda_k(\probMatrix) - \frac{\lambda_k(\probMatrix)}{t_k} \cdot \frac{\uv_k^{\T} \EE \displaceMatrix^2 \uv_k}{t_k} & = O(t_k^{-1/2}).
  \end{align}
  Since $(\EE \displaceMatrix^2)_{ij} = \delta_{i j} \sum_{t} \probMatrix_{i t} (1 - \probMatrix_{i t}) = (\EE \diagAdjecencyMatrix)_{ij} + O(\sparsityParam^2 \nsize)$, we have
  \begin{align*}
    \frac{1}{t_k^2} \uv_k^{\T} \EE \displaceMatrix^2 \uv_k = \frac{1}{t_k^2} \uv_k^{\T} \EE \diagAdjecencyMatrix \uv_k + O(t_k^{-2}) \cdot O(\sparsityParam^2 \nsize).
  \end{align*}
  Substituting this into~\eqref{eq: t_k - lambda_k approximation}, we obtain
  \begin{align*}
    t_k - \lambda_k(\probMatrix) - \lambda_k(\probMatrix) \cdot \frac{\uv_k^{\T} \EE \diagAdjecencyMatrix \uv_k}{t_k^2} = O(\sparsityParam).
  \end{align*}
  The term $(\uv_k^{\T} \EE \diagAdjecencyMatrix \uv_k)/t_k^2$ can be efficiently estimated via Lemma~\ref{lemma: second order term estimation}. Thus,
  \begin{align*}
    t_k - \lambda_k(\probMatrix) \left [1 + \frac{\estimator[\uv]_k^{\T} \diagAdjecencyMatrix \estimator[\uv]_k}{\adjacencyEigenvalues_{kk}^2} \right ] = O(\sparsityParam).
  \end{align*}
  Meanwhile, due to Lemma~\ref{lemma: eigenvalues difference}, $\adjacencyEigenvalues_{kk} = t_k + \uv_k^{\T} \displaceMatrix \uv_k + O_{\prec}(\nsize^{-1/2})$. Lemma~\ref{lemma: log estimate vector difference} guarantees that $\uv_k^{\T} \displaceMatrix \uv_k = O_\ell (\sqrt{\sparsityParam \log \nsize})$. Thus, $t_k - \adjacencyEigenvalues_{k k} = O_\ell(\sqrt{\sparsityParam \log \nsize})$, and
  \begin{align*}
    \adjacencyEigenvalues_{k k} - \lambda_k(\probMatrix) \left [1 + \frac{\estimator[\uv]_k^{\T} \diagAdjecencyMatrix \estimator[\uv]_k}{\adjacencyEigenvalues_{kk}^2} \right ] = O_\ell (\sqrt{\sparsityParam \log \nsize}), \\
    \lambda_k(\probMatrix) = \left [ \frac{1}{\adjacencyEigenvalues_{k k}} + \frac{\estimator[\uv]_k^{\T} \diagAdjecencyMatrix \estimator[\uv]_k}{\adjacencyEigenvalues_{kk}^3}\right ]^{-1} + O_\ell (\sqrt{\sparsityParam \log \nsize}).
  \end{align*}
  By the definition of $\debiasedEigenvalues_{kk}$ the statement of the lemma holds.
\end{proof}
  
\subsubsection{Important properties of the equality statistic}
  \begin{lemma}
  \label{lemma: bounds of statistic center}
    Suppose that $a = \Theta(\nsize^{-2} \sparsityParam^{-1})$. Under Conditions~\ref{cond: nonzero B elements}-\ref{cond: eigenvalues divergency} there are such constants $C_1$, $C_2$ that 
    \begin{align*}
      \frac{C_1}{\nsize^2 \sparsityParam} 
      \leqslant 
      \lambda_{\min} \bigl(
        \asymptoticVariance(i, j) + \penalizer \identity
      \bigr)
      \leqslant
      \lambda_{\max} \bigl(
        \asymptoticVariance(i, j) + \penalizer \identity
      \bigr)
      \leqslant
      \frac{C_2}{\nsize^2 \sparsityParam}
    \end{align*}
    and such constants $C_1'$ and $C_2'$ that
    \begin{align*}
      C_1' \Vert \nodeCommunityMatrix_i - \nodeCommunityMatrix_j \Vert^2 
      \leqslant
      \frac{\equalityStatisticCenter_{ij}^\penalizer}{\nsize \sparsityParam}
      \leqslant
      C_2' \Vert \nodeCommunityMatrix_i - \nodeCommunityMatrix_j \Vert^2
    \end{align*}
  for any $i$ and $j$.
  \end{lemma}

  \begin{proof}
    Let us estimate eigenvalues of matrix $\asymptoticVariance(i, j)$. After some straightforward calculations we have
    \begin{align*}
      \asymptoticVariance (i, j) &= 
      \probEigenvalues^{-1}
      \probEigenvectors^{\T}
      \EE \left(
        \displaceMatrix_i
        -
        \displaceMatrix_j
      \right)^{\T}
      \left(
        \displaceMatrix_i
        -
        \displaceMatrix_j
      \right)
      \probEigenvectors
      \probEigenvalues^{-1}  \\
      & = 
      \probEigenvalues^{-1}
      \probEigenvectors^{\T}
      \left(
        \operatorname{diag} (
          \EE \displaceMatrix_i^2
          +
          \EE \displaceMatrix_j^2
        )
        - \EE \displaceMatrix_{ij}^2
        (\ev_i \ev_j^{\T} + \ev_j \ev_i^{\T})
      \right)
      \probEigenvectors
      \probEigenvalues^{-1}.
    \end{align*}
    The maximum eigenvalue can be estimated using a norm of the matrix:
    \begin{align*}
      \lambda_{\max} \bigl( 
        \asymptoticVariance (i, j) + \penalizer \identity
      \bigr)
      =
      \Vert
        \asymptoticVariance (i, j)
      \Vert + \penalizer
      \leqslant \penalizer + 
      \Vert
        \probEigenvalues^{-1}
      \Vert^2
      \Vert
        \probEigenvectors
      \Vert^2
      \left(
        \Vert 
          \operatorname{diag} (
            \EE \displaceMatrix_i^2
            +
            \EE \displaceMatrix_j^2
          )
        \Vert
        + 2 \EE \displaceMatrix_{ij}^2
      \right),
    \end{align*}
    \begin{align*}
      \lambda_{\max} \bigl(
        \asymptoticVariance(i, j)
      \bigr)
      \leqslant 
      \frac{
        4 \sparsityParam
      }{
        \lambda_{\nclusters}^2 (\probMatrix)
      } + O(\nsize^{-2} \sparsityParam^{-1}),
    \end{align*}
    since $\EE \displaceMatrix_{ij}^2 = \probMatrix_{ij} - \probMatrix_{ij}^2$. Since $\lambda_{\nclusters}(\probMatrix) = \Theta(\nsize \sparsityParam)$ due to Lemma~\ref{lemma: eigenvalues asymptotics}, we have
    \begin{align*}
        \lambda_{\max}(\asymptoticVariance(i, j) +\penalizer \identity) = O(\nsize^{-2} \sparsityParam^{-1})
    \end{align*}
    Clearly, $\asymptoticVariance(i, j)$ is non-negative. Thus, we get
    \begin{align*}
        \lambda_{\min}(\asymptoticVariance(i, j) + \penalizer \identity) \ge a = \Omega(\nsize^{-2} \sparsityParam^{-1}).
    \end{align*}
    
    Now we state
    \begin{align*}
      \equalityStatisticCenter_{ij}^\penalizer 
      \leqslant 
      \frac{1}{\lambda_{\min} \bigl(\asymptoticVariance(i, j) + \penalizer \identity\bigr)} 
      \Vert
        \probEigenvectors_i - \probEigenvectors_j
      \Vert^2
      \leqslant
      \frac{
        \sigma_{\max}^2 (\basisMatrix)
      }{
        \lambda_{\min} \bigl(\asymptoticVariance(i, j) + \penalizer \identity \bigr)
      }
      \Vert
        \nodeCommunityMatrix_i
        -
        \nodeCommunityMatrix_j
      \Vert^2.
    \end{align*}
    In the same way, we obtain
    \begin{align*}
      \equalityStatisticCenter_{ij}^\penalizer
      \geqslant
      \frac{
        \sigma_{\min}^2 (\basisMatrix)
      }{
        \lambda_{\max} \bigl(\asymptoticVariance(i, j) + \penalizer \identity \bigr)
      }
      \Vert
        \nodeCommunityMatrix_i
        -
        \nodeCommunityMatrix_j
      \Vert^2.
    \end{align*}
    Applying asymptotic properties of singular values from Lemma~\ref{lemma: F rows tensor product}, we complete the proof.
  \end{proof}

  \begin{lemma}
  \label{lemma: uniformly covariance estimation}
    Under Conditions~\ref{cond: nonzero B elements}-\ref{cond: theta distribution-a} it holds that
    \begin{align}
      \max_{i, j} 
          \left \Vert
            \asymptoticVariance(i, j)
            -
            \estimator[\asymptoticVariance](i,j)
          \right \Vert
      =
      O_{\prec} \left(
          \frac{
            1
          }{
            \nsize^2 \rho \sqrt{\nsize \sparsityParam}
          }
      \right).
    \end{align}
  \end{lemma}

  \begin{proof}
    This proof is a slight modification of the corresponding one of Theorem~5 from~\cite{Fan2020_ASYMPTOTICS}. We start considering
    \begin{align*}
      \asymptoticVariance(i, j)
      & =
      \probEigenvalues^{-1}
      \probEigenvectors^{\T}
      \left(
        \operatorname{diag} (
          \EE \displaceMatrix_i^2
          +
          \EE \displaceMatrix_j^2
        )
        - \EE \displaceMatrix_{ij}^2
        (\ev_i \ev_j^{\T} + \ev_j \ev_i^{\T})
      \right)
      \probEigenvectors
      \probEigenvalues^{-1}, \\
      \estimator[\asymptoticVariance](i, j)
      & =
      \debiasedEigenvalues^{-1}
      \adjacencyEigenvectors^{\T}
      \bigl(
        \operatorname{diag} (
          \estimator[\displaceMatrix]_i^2
          +
          \estimator[\displaceMatrix]_j^2
        )
        - \estimator[\displaceMatrix]_{ij}
        (\ev_i \ev_j^{\T} + \ev_j \ev_i^{\T})
      \bigr)
      \adjacencyEigenvectors
      \debiasedEigenvalues^{-1}.
    \end{align*}
    We begin with studying the sum for some particular values $k_1$ and $k_2$:
    \begin{align*}
      \sum_{l = 1}^\nsize \probEigenvectors_{l k_1} \probEigenvectors_{l k_2} (\displaceMatrix_{il}^2 - \EE \displaceMatrix_{il}^2).
    \end{align*}
    It is a sum of independent random variables. According to the Bernstein inequality, the above is greater than $t$ with probability at most
    \begin{align*}
      \exp \left(
        -\frac{
          t^2
        }{
          \sum_{l = 1}^\nsize \probEigenvectors_{l k_1}^2 \probEigenvectors_{l k_2}^2 \EE \displaceMatrix_{il}^4 + \frac{C_{\probEigenvectors}^2 t}{3 \nsize} 
        }
      \right) 
      & \leqslant
      \exp \left(
        -\frac{
          t^2
        }{
          \frac{C_\probEigenvectors^2}{\nsize} \max_l \EE \displaceMatrix_{il}^4 + \frac{C_{\probEigenvectors}^2 t}{3 \nsize} 
        }
      \right) \\
      & 
      \leqslant
      \exp \left(
        -\frac{
          t^2
        }{
          \frac{C_\probEigenvectors^2}{\nsize} 2 \sparsityParam + \frac{C_{\probEigenvectors}^2 t}{3 \nsize} 
        }
      \right),
    \end{align*}
    where $C_{\probEigenvectors}$ is the uniform constant from Lemma~\ref{lemma: eigenvectors max norm}. For arbitrary $\varepsilon$ taking appropriate $t = \sqrt{\frac{\rho}{\nsize}} \nsize^{\delta}$, we observe that
    \begin{align*}
      \sum_{l = 1}^\nsize 
          \probEigenvectors_{l k_1} 
          \probEigenvectors_{l k_2} 
          (
            \displaceMatrix_{il}^2
            - \EE \displaceMatrix_{il}^2
            + \displaceMatrix_{jl}^2 
            - \EE \displaceMatrix_{jl}^2
          )
      =
      O_\prec \left(
        \sqrt{\frac{\sparsityParam}{\nsize}}
      \right)
    \end{align*}
    due to the definition of $O_\prec (\cdot)$. Moreover, due to Lemma~\ref{lemma: corrected eigs and noise},
    \begin{align*}
      & \quad \sum_{l = 1}^\nsize 
        \probEigenvectors_{l k_1} 
        \probEigenvectors_{l k_2} 
        (
          \estimator[\displaceMatrix]_{il}^2
          - \EE \displaceMatrix_{il}^2
          + \estimator[\displaceMatrix]_{jl}^2 
          - \EE \displaceMatrix_{jl}^2
        ) \\
      & =
      \sum_{l = 1}^\nsize 
        \probEigenvectors_{l k_1} 
        \probEigenvectors_{l k_2} 
        (
          \displaceMatrix_{il}^2
          - \EE \displaceMatrix_{il}^2
          + \displaceMatrix_{jl}^2 
          - \EE \displaceMatrix_{jl}^2
        )
      +
      O_\prec \left(
        \sqrt{\frac{\sparsityParam}{\nsize}}
      \right),
    \end{align*}
    and, consequently,
    \begin{align*}
      \sum_{l = 1}^\nsize 
        \probEigenvectors_{l k_1} 
        \probEigenvectors_{l k_2} 
        (
          \estimator[\displaceMatrix]_{il}^2
          - \EE \displaceMatrix_{il}^2
          + \estimator[\displaceMatrix]_{jl}^2 
          - \EE \displaceMatrix_{jl}^2
        )
      =
      O_\prec \left(
        \sqrt{\frac{\sparsityParam}{\nsize}}
      \right).
    \end{align*}
    Due to Lemma~\ref{lemma: adj eigenvectors displacement}, we have
    \begin{align*}
        \adjacencyEigenvectors_{i k} = \probEigenvectors_{ik} + O_{\ell} \Biggl ( \sqrt{\frac{\log \nsize}{\nsize^2 \sparsityParam}}\Biggr ).
    \end{align*}
    We may bound $\probEigenvectors_{ik} = O(\nsize^{-1/2})$ due to Lemma~\ref{lemma: eigenvectors max norm} and $(\nsize \sparsityParam)^{-1} \log \nsize = O(1)$ due to Condition~\ref{cond: sparsity param bound}. So $\adjacencyEigenvectors_{i k} = O_{\prec} \left( \nsize^{-1/2} \right)$. Hence, we get
    \begin{align*}
      & \quad \sum_{l = 1}^\nsize 
        \adjacencyEigenvectors_{l k_1} 
        \adjacencyEigenvectors_{l k_2}
        (
          \estimator[\displaceMatrix]_{il}^2
          + \estimator[\displaceMatrix]_{jl}^2 
        )
      =
      \sum_{l = 1}^\nsize 
        (\adjacencyEigenvectors_{l k_1} - \probEigenvectors_{l k_1})
        \adjacencyEigenvectors_{l k_2}
        (
          \estimator[\displaceMatrix]_{il}^2
          + \estimator[\displaceMatrix]_{jl}^2 
        ) \\
      & + 
      \sum_{l = 1}^\nsize 
        \probEigenvectors_{l k_1} 
        (\adjacencyEigenvectors_{l k_2} - \probEigenvectors_{l k_2})
        (
          \estimator[\displaceMatrix]_{il}^2
          + \estimator[\displaceMatrix]_{jl}^2 
        )
      +
      \sum_{l = 1}^\nsize 
        \probEigenvectors_{l k_1} 
        \probEigenvectors_{l k_2}
        (
          \estimator[\displaceMatrix]_{il}^2
          + \estimator[\displaceMatrix]_{jl}^2 
        )
      +
      O_\prec \left(
        \sqrt{\frac{\sparsityParam}{\nsize}}
      \right),
    \end{align*}
    and, finally,
    \begin{align*}
      \sum_{l = 1}^\nsize 
        \adjacencyEigenvectors_{l k_1} 
        \adjacencyEigenvectors_{l k_2} 
        (
          \estimator[\displaceMatrix]_{il}^2
          + \estimator[\displaceMatrix]_{jl}^2 
        )
      = 
      \sum_{l = 1}^\nsize 
        \probEigenvectors_{l k_1} 
        \probEigenvectors_{l k_2} 
        (
          \EE \displaceMatrix_{il}^2
          + \EE \displaceMatrix_{jl}^2
        )
      +
      O_\prec \left(
        \sqrt{\frac{\sparsityParam}{\nsize}}
      \right).
    \end{align*}
    In the same way,
    \begin{align*}
      \estimator[\displaceMatrix]_{ij}^2
      \left(
        \adjacencyEigenvectors_{i k_1}
        \adjacencyEigenvectors_{j k_2}
        +
        \adjacencyEigenvectors_{j k_1}
        \adjacencyEigenvectors_{i k_2}
      \right)
      = 
      \EE \displaceMatrix_{ij}^2
      \left(
        \probEigenvectors_{i k_1}
        \probEigenvectors_{j k_2}
        +
        \probEigenvectors_{j k_1}
        \probEigenvectors_{i k_2}
      \right)
      +
      O_\prec \left(
        \sqrt{\frac{\sparsityParam}{\nsize}}
      \right).
    \end{align*}
    Define
    \begin{align*}
        V(i, j) & = \probEigenvectors^{\T}
      \left(
        \operatorname{diag} (
          \EE \displaceMatrix_i^2
          +
          \EE \displaceMatrix_j^2
        )
        - \EE \displaceMatrix_{ij}^2
        (\ev_i \ev_j^{\T} + \ev_j \ev_i^{\T})
      \right)
      \probEigenvectors, \\
    \estimator[V](i, j) & = \adjacencyEigenvectors^{\T}
      \bigl(
        \operatorname{diag} (
          \estimator[\displaceMatrix]_i^2
          +
          \estimator[\displaceMatrix]_j^2
        )
        - \estimator[\displaceMatrix]_{ij}
        (\ev_i \ev_j^{\T} + \ev_j \ev_i^{\T})
      \bigr)
      \adjacencyEigenvectors, \\
    \Delta_{\probEigenvectors}(i, j) & =  V(i, j) - \estimator[V](i, j).
    \end{align*}
    Then $\Delta_{\probEigenvectors} = O_{\prec} \Bigl ( \sqrt{\frac{\sparsityParam}{\nsize}} \Bigr )$ and 
    \begin{align*}
       \Vert V(i, j) \Vert \le \Vert \operatorname{diag} (
          \EE \displaceMatrix_i^2
          +
          \EE \displaceMatrix_j^2
        )
        - \EE \displaceMatrix_{ij}^2
        (\ev_i \ev_j^{\T} + \ev_j \ev_i^{\T}) \Vert \le 4 \rho,
    \end{align*}
    so $\Vert \estimator[V](i, j) \Vert = O_{\prec} (\rho)$. We have
    \begin{align}
    \label{eq: lemma asymptotic variance - decomposition with spectral noise}
         \asymptoticVariance(i, j) - \estimator[\asymptoticVariance](i, j) = \probEigenvalues^{-1} \Delta_{\probEigenvectors}(i, j) \probEigenvalues^{-1} + \probEigenvalues^{-1} \estimator[V](i, j) (\probEigenvalues^{-1} - \debiasedEigenvalues^{-1}) + \debiasedEigenvalues^{-1} \estimator[V](i, j) (\probEigenvalues^{-1} - \debiasedEigenvalues^{-1})
    \end{align}
   Meanwhile, we have
    \begin{align*}
        & \quad \Vert \probEigenvalues^{-1} - \debiasedEigenvalues^{-1} \Vert = \bigl \Vert \probEigenvalues^{-1} - \probEigenvalues^{-1} \bigl(\identity + \probEigenvalues^{-1} (\debiasedEigenvalues - \probEigenvalues)\bigr)^{-1} \bigr \Vert \\
        & = \Bigl \Vert \probEigenvalues^{-1} - \probEigenvalues^{-1} \sum_{i = 0}^{\infty} (-1)^i \probEigenvalues^{-i} (\debiasedEigenvalues - \probEigenvalues)^{i} \Bigr \Vert 
        = \Bigl \Vert - \probEigenvalues^{-1} \sum_{i = 1}^{\infty} (-1)^{i} \probEigenvalues^{-i} (\debiasedEigenvalues - \probEigenvalues)^{i} \Bigr \Vert \\
        & = \Bigl \Vert \probEigenvalues^{-2}  (\debiasedEigenvalues - \probEigenvalues) \cdot \sum_{i = 0}^\infty (-1)^{i} \probEigenvalues^{-i} (\debiasedEigenvalues - \probEigenvalues)^{i} \Bigr \Vert
        \le \Vert \probEigenvalues \Vert^{-2}  \Vert \debiasedEigenvalues - \probEigenvalues \Vert  \cdot \frac{1}{1 + \Vert \probEigenvalues^{-1} (\debiasedEigenvalues - \probEigenvalues ) \Vert}.
    \end{align*}
    Since $\debiasedEigenvalues_{kk} = \probEigenvalues_{kk} + O_\ell(\sqrt{\sparsityParam \log \nsize})$ due to Lemma~\ref{lemma: debaised eigenvalues behaviour} and $\probEigenvalues_{kk} = \Theta(\nsize \sparsityParam)$ due to Lemma~\ref{lemma: eigenvalues asymptotics}, we obtain
    \begin{align*}
        \Vert \probEigenvalues^{-1} - \debiasedEigenvalues^{-1} \Vert = O\Bigl( \frac{1}{\nsize^2 \rho^2} \Bigr) \cdot O_{\ell} (\sqrt{\sparsityParam \log \nsize}).
    \end{align*}
    Thus, the dominating term in~\eqref{eq: lemma asymptotic variance - decomposition with spectral noise} is the first one, so
    \begin{align}
      \estimator[\asymptoticVariance](i, j) = 
      \asymptoticVariance(i, j) 
      + O_\prec \left( 
        \frac{1}{\nsize^2 \sparsityParam} \cdot \frac{1}{\sqrt{\nsize \sparsityParam}}
      \right).
    \end{align}
  \end{proof}

\subsubsection{Applicability of Lemma~\ref{lemma: fan eigenvectors series decomposition}}
  First, we compute the asymptotic expansion of some values presented in Table~\ref{tab: expansion notation}. Variables $\probEigenvalues_{-k}$ and $\probEigenvectors_{-k}$ are defined in the caption of Table~\ref{tab: expansion notation}.
  \begin{lemma}
  \label{lemma: auxiliary variables expansion}
    Under Conditions~\ref{cond: nonzero B elements}-\ref{cond: theta distribution-a} we have asymptotic expansions described in Table~\ref{tab: auxiliary variables expansion}.
  \end{lemma}

  \begin{table}[t!]
    \centering
    \begin{tabular}{|l|}
        \hline
        ``Resolvents'' approximation \\
        \hline \\
        $\resolvent(\ev_i, \probEigenvectors_{-k}, t_k) = -\frac{1}{t_k} \ev_i^{\T} \probEigenvectors_{-k} + O \left( t_k^{-2} / \sqrt{\nsize} \right)$ \\ \\
        $\resolvent(\uv_k, \probEigenvectors_{-k}, t_k) = -\frac{1}{t_k^3} \uv_k^{\T} \EE \displaceMatrix^2 \probEigenvectors_{-k} + O(t_k^{-5/2}) $ \\ \\
        $\resolvent(\uv_k, \uv_k, t_k) = - \frac{1}{t_k} - \frac{1}{t_k^3} \uv_k^{\T} \EE \displaceMatrix^2 \uv_k + O(t_k^{-5/2})$ \\ \\
        \hline
        0-degree coefficients approximation \\
        \hline \\
        $\meanFactor_{\uv_k, k, t_k} = -1 - \frac{1}{t_k^2} \uv_k^{\T} \EE \displaceMatrix^2 \uv_k + O(t_k^{-3/2})$ \\ \\
        $\meanFactor_{\ev_i, k, t_k} = 
        - \probEigenvectors_{i k} 
      - \frac{1}{t_k^2} \ev_i^{\T} \EE \displaceMatrix^2 \uv_k 
      - \frac{1}{t_k^2} \sum_{k' \in [\nclusters] \setminus \{k\}} \frac{\lambda_{k'} \probEigenvectors_{i k'} }{\lambda_{k'} - t_k} \cdot \uv_{k'}^\T \EE \displaceMatrix^2 \uv_k + O(t_k^{-3/2}/\sqrt{\nsize})$ \\ \\
        $\pFactor_{k, t_k} = 1 - \frac{3}{t_k^2} \uv_k^{\T} \EE \displaceMatrix^2 \uv_k + O(t_k^{-3/2})$ \\ \\
        \hline
        Vector auxiliary variables \\
        \hline \\
        $\bv_{\ev_i, k, t_k} = \ev_i + O(\nsize^{-1/2}) $ \\ \\
        $\bv_{\uv_k, k, t_k} = \uv_k + O(t_k^{-1}) $\\ \\
        \hline
        Matrix auxiliary variables \\
        \hline \\
        $\left[
          \probEigenvalues_{-k}^{-1}
          +
          \resolvent(\probEigenvectors_{-k}, \probEigenvectors_{-k}, t_k)
        \right]^{-1}
        =
        \diag \left(\frac{\lambda_{k'} t_k}{t_{k} - \lambda_{k'}}\right)_{k' \in [\nclusters] \setminus \{k\}} + O(1) $ \\ \\
        \hline
    \end{tabular}
    \caption{Asymptotic expansion of some variables from Lemma~\ref{lemma: auxiliary variables expansion}.}
  \label{tab: auxiliary variables expansion}
  \end{table}

  \begin{proof}
    From Lemma~\ref{lemma: power expectation} we have for any distinct $k, k'$ and $l \ge 2$:
    \begin{equation*}
        \ev_i^{\T} \EE \displaceMatrix^{l} \uv_{k} = O(\alpha_{\nsize}^l \Vert \uv_{k} \Vert_{\infty}), \qquad
        \uv_{k}^{\T} \EE \displaceMatrix^l \uv_{k'} = O(\alpha_{\nsize}^l).
    \end{equation*}
    According to Lemma~\ref{lemma: eigenvectors max norm}, we have $\Vert \uv_{k} \Vert_{\infty} = O(\nsize^{-1/2})$. Theorem~\ref{theorem: conditions satisfuction}, Lemma~\ref{lemma: eigenvalues asymptotics} and Lemma~\ref{lemma: t_k is well-definied} guarantee that $\alpha_{\nsize} = O(t_{k}^{1/2})$. Finally, $\uv_k^{\T} \probEigenvectors_{-k} = \mathbf{O}$ and $\probEigenvectors_{-k}^{\T} \probEigenvectors_{-k} = \identity$ because of eigenvectors' orthogonality.
    All the above deliver us the following expansion:
    \begin{align*}
      \resolvent(\ev_i, \probEigenvectors_{-k}, t_k) 
      & = - \frac{1}{t_k} \ev_i^{\T} \probEigenvectors_{-k} - \sum_{l = 2}^L t_k^{- (l + 1)} \ev_i^{\T} \EE \displaceMatrix^l \probEigenvectors_{-k}\\
      & = -\frac{1}{t_k} \ev_i^{\T} \probEigenvectors_{-k} + O \left( t_k^{-3} \alpha_{\nsize}^2 / \sqrt{\nsize} \right) 
      = 
      -\frac{1}{t_k} \ev_i^{\T} \probEigenvectors_{-k} + O \left(t_k^{-2} / \sqrt{\nsize}\right),
    \end{align*}
    \begin{align*}
      \resolvent(\uv_k, \probEigenvectors_{-k}, t_k)
      & = - \frac{1}{t_k} \uv_{k}^{\T} \probEigenvectors_{-k} - \frac{1}{t_k^3} \uv_{k}^{\T} \EE \displaceMatrix^2 \probEigenvectors_{-k} - \sum_{l = 3}^L t_{k}^{-(l + 1)} \uv_{k}^{\T} \EE \displaceMatrix^l \probEigenvectors_{-k} \\
      & = -\frac{1}{t_k^3} \uv_k^{\T} \EE \displaceMatrix^2 \probEigenvectors_{-k} + O(t_k^{-4} \alpha_{\nsize}^3)
      =
      -\frac{1}{t_k^3} \uv_k^{\T} \EE \displaceMatrix^2 \probEigenvectors_{-k} + O(t_k^{-5/2}),
      \\
      \resolvent(\uv_k, \uv_k, t_k)
      & = - \frac{1}{t_k} \uv_k^{\T} \uv_k - \frac{1}{t_k^3} \uv_k^{\T} \EE \displaceMatrix^2 \uv_k - \sum_{l = 3}^L t_k^{- (l + 1)} \uv_k^{\T} \EE \displaceMatrix^{l} \uv_k \\
      & = -\frac{1}{t_k} - \frac{1}{t_k^3} \uv_k^{\T} \EE \displaceMatrix^2 \uv_k + O(t_k^{-4} \alpha_{\nsize}^3 )
      =
      -\frac{1}{t_k} - \frac{1}{t_k^3} \uv_k^{\T} \EE \displaceMatrix^2 \uv_k + O(t_k^{-5/2}),
    \end{align*}
    \begin{align*}
      \resolvent(\probEigenvectors_{-k}, \probEigenvectors_{-k}, t_k) & = - \frac{1}{t_k} \probEigenvectors_{-k}^{\T} \probEigenvectors_{-k} - \sum_{l = 2}^{L} t_k^{-(l + 1)} \probEigenvectors_{-k}^{\T} \EE \displaceMatrix^l \probEigenvectors_{-k}
      = - \frac{1}{t_k} \identity + O(t_k^{-2}),
      \\
      \resolvent(\ev_i, \uv_k, t_k)
      & = -\frac{1}{t_k}\ev_i^{\T} \uv_k - \frac{1}{t_k^3} \ev_i^{\T} \EE \displaceMatrix^2 \uv_k - \sum_{l = 3}^L t_k^{-(l + 1)} \ev_i^{\T} \EE \displaceMatrix^l \uv_k \\
      & = -\frac{1}{t_k} \probEigenvectors_{i k} - \frac{1}{t_k^3} \ev_i^{\T} \EE \displaceMatrix^2 \uv_k + O\left(t_k^{-4} \alpha_{\nsize}^3 / \sqrt{\nsize}\right) \\
      & =
      -\frac{1}{t_k} \probEigenvectors_{i k} - \frac{1}{t_k^3} \ev_i^{\T} \EE \displaceMatrix^2 \uv_k + O(t_k^{-5/2} / \sqrt{\nsize}).
    \end{align*}
    
    Next we estimate 
    $\left[
      \probEigenvalues_{-k}^{-1}
      +
      \resolvent(\probEigenvectors_{-k}, \probEigenvectors_{-k}, t_k)
      \right]^{-1}
    $. Since 
    \begin{align*}
      \probEigenvalues_{-k}^{-1} - \frac{1}{t_k} \identity =
      \diag \left(
        \frac{ 
          t_k - \lambda_{k'}
        }{
          \lambda_{k'} t_k
        } 
      \right)_{k' \in [\nclusters] \setminus \{k\}}
    \end{align*}
    has order $\Omega(t_k^{-1})$ due to Condition~\ref{cond: eigenvalues divergency} and Lemma~\ref{lemma: t_k is well-definied}, 
    \begin{equation*}
      \left[
        \probEigenvalues_{-k}^{-1}
        +
        \resolvent(\probEigenvectors_{-k}, \probEigenvectors_{-k}, t_k)
      \right]^{-1}
      =
      \diag \left( \frac{\lambda_{k'} t_k}{t_k - \lambda_{k'}}\right) 
      \left[
        \identity + O(t_k^{-1})
      \right]^{-1}
      = 
      \diag \left(\frac{\lambda_{k'} t_k}{t_k - \lambda_{k'}}\right) + O(1).
    \end{equation*}
    After that, we are able to establish asymptotics of $\meanFactor_{\uv_k, k, t_k}$ and $\meanFactor_{\ev_i, k, t_k}$. Indeed,
    \begin{align*}
      \meanFactor_{\uv_k, k, t_k}
      & = 
      - 1
      - \frac{1}{t_k^2} \uv_k^{\T} \EE \displaceMatrix^2 \uv_k 
      + O(t_k^{-3 / 2}) 
      - \left[ 
        -\frac{1}{t_k^3} \uv_k^{\T} \EE \displaceMatrix^2 \probEigenvectors_{-k} + O(t_k^{-5/2})
      \right] \times \\
      & \quad \, \times \left[
        \diag \left( \frac{\lambda_{k'} t_k}{t_k - \lambda_{k'}} \right)_{k' \in [\nclusters] \setminus \{k\}} + O(1)
      \right] \times \left[ 
        -\frac{1}{t_k^2} \probEigenvectors_{-k}^{\T} \EE \displaceMatrix^2 \uv_k + O(t_k^{-3/2}) 
      \right] \\
      & = - 1
      - \frac{1}{t_k^2} \uv_k^{\T} \EE \displaceMatrix^2 \uv_k 
      + O(t_k^{-3/2})
    \end{align*}
    since $\frac{\lambda_{k'} t_k}{t_k - \lambda_{k'}} = O(t_k)$ and $\uv_k^{\T} \EE \displaceMatrix^2 \probEigenvectors_{-k} = O(t_k)$. Similarly,
    \begin{align*}
      \meanFactor_{\ev_i, k, t_k} 
      & = - \probEigenvectors_{i k} - \frac{1}{t_k^2} \ev_i^{\T} \EE \displaceMatrix^2 \uv_k + O(t_k^{-3/2} / \sqrt{\nsize})
      -
      \left[ 
        - \frac{1}{t_k} \ev_i^{\T} \probEigenvectors_{-k} + O(t_k^{-2} / \sqrt{\nsize})
      \right] \times \\
      & \quad \, \times \left[ 
        \diag \left( \frac{\lambda_{k'} t_k}{t_k - \lambda_{k'}} \right)_{k' \in [\nclusters] \setminus \{k\}} + O(1)
      \right] \times \left[ 
        - \frac{1}{t_k^2} \probEigenvectors_{-k}^{\T} \EE \displaceMatrix^2 \uv_k + O(t_k^{-5/2}) 
      \right] \\
      & = - \probEigenvectors_{i k} 
      - \frac{1}{t_k^2} \ev_i^{\T} \EE \displaceMatrix^2 \uv_k 
      - \frac{1}{t_k^2} \sum_{k' \in [\nclusters] \setminus \{k\}} \frac{\lambda_{k'} \probEigenvectors_{i k'} }{\lambda_{k'} - t_k} \cdot \uv_{k'}^\T \EE \displaceMatrix^2 \uv_k + O(t_k^{-3/2} / \sqrt{\nsize}),
    \end{align*}
    where we use Lemma~\ref{lemma: eigenvectors max norm} to estimate $\ev_i^{\T} \probEigenvectors_{-k}$. After that we are able to approximate $\pFactor_{k, t_k}$:
    \begin{align*}
        \pFactor_{k, t_k} & = \left [ t_k^2 \frac{d}{d t_k} \frac{\meanFactor_{\uv_k, k, t_k}}{t_k} \right ]^{-1} = \left [ t_k^2 \frac{d}{d t_k} \left(
            - \frac{1}{t_k} - \frac{1}{t_k^3} \uv_k^\T \EE \displaceMatrix^2 \uv_k + O (t_k^{-5/2})
        \right)\right ]^{-1} \\
        & = \left [ 1 + \frac{3}{t_k^2} \uv_k^\T \EE \displaceMatrix^2 \uv_k + O(t_k^{-3/2}) \right ]^{-1}
        g= 1 - \frac{3}{t_k^2} \uv_k^\T \EE \displaceMatrix^2 \uv_k + O(t_k^{-3/2}).
    \end{align*}
    Finally,
    \begin{align*}
      \bv_{\ev_i, k, t_k} & = \ev_i - \probEigenvectors_{-k} 
      \left[ 
        \diag \left(\frac{\lambda_{k'} t_k}{t_k - \lambda_{k'}}\right) + O(1)
      \right] \times \left[ - \frac{1}{t_k}\probEigenvectors^{\T}_{-k} \ev_i + O(t_k^{-2} / \sqrt{\nsize}) \right] \\
      & = \ev_i + \frac{1}{t_k} \probEigenvectors_{-k} \left( \sum_{k' \in [\nclusters] \setminus k} \frac{\lambda_{k'} t_k}{t_k - \lambda_{k'}} \ev_{k'} \ev_{k'}^\T \right) \probEigenvectors_{-k}^{\T} \ev_{i} + O(t_k^{-1} / \sqrt{\nsize}) \\
      & = \ev_i + \frac{1}{t_k} \sum_{k' \in [\nclusters] \setminus k} \frac{\lambda_{k'} t_k}{t_k - \lambda_{k'}} (\probEigenvectors_{-k} \ev_{k'}) (\ev_{k'}^\T  \probEigenvectors_{-k}^{\T} \ev_{i}) + O(t_k^{-1} / \sqrt{\nsize}) \\
      & = \ev_i + \sum_{k' \in [\nclusters] \setminus \{k\}} \frac{\lambda_{k'}}{t_k - \lambda_{k'}} \uv_{k'} \cdot \probEigenvectors_{i k'} + O(t_k^{-1} / \sqrt{\nsize}) \\
      & = \ev_i + O(\nsize^{-1/2}),
    \end{align*}
    since, slightly abusing notation, we have $\probEigenvectors_{-k} \ev_{k'} = \uv_{k'}$, $\Vert \uv_{k'} \Vert = 1$ and $\probEigenvectors_{i k'} = O(\nsize^{-1/2})$. Analogously,
    \begin{align*}
      \bv_{\uv_k, k, t_k} & = \uv_k - \probEigenvectors_{-k} 
      \left[ 
        \diag \left( \frac{\lambda_{k'} t_k}{t_k - \lambda_{k'}} \right) + O(1)
      \right] \times \left[ 
        -\frac{1}{t_k^{3}} \probEigenvectors_{-k}^{\T} \EE \displaceMatrix^2 \uv_k + O(t_k^{-5/2})
      \right] \\
      & = \uv_k + \frac{1}{t_k^3} \probEigenvectors_{-k} \left(
        \sum_{k' \in [\nclusters] \setminus \{k\}} \frac{\lambda_{k'} t_k}{t_k - \lambda_{k'}} \ev_{k'} \ev_{k'}^\T
      \right) \probEigenvectors_{-k}^\T \EE \displaceMatrix^2 \uv_k + O(t_k^{-3/2}) \\
      & = \uv_k + \frac{1}{t_k^3} 
        \sum_{k' \in [\nclusters] \setminus \{k\}} \frac{\lambda_{k'} t_k}{t_k - \lambda_{k'}} (\probEigenvectors_{-k} \ev_{k'})  (\ev_{k'}^\T
       \probEigenvectors_{-k}^\T \EE \displaceMatrix^2 \uv_k) + O(t_k^{-3/2}) \\
      & = \uv_k + 
      \sum_{k' \in [\nclusters] \setminus \{k'\}} 
        \frac{\lambda_{k'}}{t_k - \lambda_{k'}} \uv_{k'} \cdot
        \frac{\uv_{k'} \EE \displaceMatrix^2 \uv_k}{t_k^2} 
      + O(t_{k}^{-3/2}) \\
      & = \uv_k + O(t_k^{-1}),
    \end{align*}
    where we use $\uv_{k'} \EE \displaceMatrix^2 \uv_k = O(t_k)$ and $\Vert \uv_k \Vert = 1$.
  \end{proof}
  
  \begin{lemma}
  \label{lemma: fan asymptotic expansion w/o sigma}
    Under Conditions~\ref{cond: nonzero B elements}-\ref{cond: theta distribution-a}, for $\xv \in \{ \uv_k, \ev_i\}$, it holds that
    \begin{align*}
      \xv^{\T} \estimator[\uv]_k \estimator[\uv]_k^{\T} \uv_k
      & = a_k 
      + 
      \tr [
        \displaceMatrix \jMatrix_{\xv, \uv_k, k, t_k} - (\displaceMatrix^2 - \EE \displaceMatrix^2) \lMatrix_{\xv, \uv_k, k, t_k}
      ] \\
      & \quad + \tr (\displaceMatrix \uv_k \uv_k^{\T}) \tr (\displaceMatrix \qMatrix_{\xv, \uv_k, k, t_k})
      + O_{\prec} \left( \frac{1}{\nsize^2 \sparsityParam^2}\right),
    \end{align*}
    where $a_k = \meanFactor_{\xv, k, t_k} \meanFactor_{\uv_k, k, t_k} \pFactor_{k, t_k}$.
  \end{lemma}

  \begin{proof} In Lemma~\ref{lemma: fan eigenvectors series decomposition}, we present the statement provided by~\cite{Fan2020_ASYMPTOTICS}. The authors need $\sigma_k^2$ and $\tilde{\sigma}_k^2$ to establish asymptotic distribution of the form $\xv^\T \widehat{\uv}_k \widehat{\uv}_k^\T \yv$, while we require only concentration properties. Thus, the condition regrading $\sigma_k^2$ and $\tilde{\sigma}_k^2$ can be omitted.
    
    The only remaining issue is to replace $O_p(t_k^{-2})$ with $O_{\prec}(t_k^{-2})$. Notice that the source of $O_p(\cdot)$ in Lemma~\ref{lemma: fan eigenvectors series decomposition} are random values of the form
    \begin{align*}
      \xv^{\T} (\displaceMatrix^\ell - \EE \displaceMatrix^\ell) \yv,
    \end{align*}
    where $\xv$ and $\yv$ are unit vectors. In~\cite{Fan2020_ASYMPTOTICS}, authors bounded it using the second moment. At the same time, they obtain an estimation
    \begin{align*}
      \xv^{\T} (\displaceMatrix^\ell - \EE \displaceMatrix^\ell) \yv = O_{\prec} \left( \min(\alpha_{\nsize}^{\ell - 1}, \Vert \xv \Vert_{\infty} \alpha_{\nsize}^\ell, \Vert \yv \Vert_{\infty} \alpha_{\nsize}^\ell) \right)
    \end{align*}
    in~\cite{Fan2019_SIMPLE} using all moments provided by Lemma~\ref{lemma: power deviation}.

    Due to Lemma~\ref{lemma: eigenvalues asymptotics} and Lemma~\ref{lemma: t_k is well-definied}, we have $O_{\prec}(t_k^{-2}) = O_{\prec}\left( [\nsize \sparsityParam]^{-2}\right)$.  That delivers the statement of the lemma.
  \end{proof}

\subsubsection{SPA consistency}
  \begin{lemma}
  \label{lemma: log estimate vector difference}
    For any unit $\xv$ and $\yv$, we have 
    \begin{align*}
      \xv^{\T} \displaceMatrix \yv = O_\ell \left(\max \left\{\sqrt{\frac{\sparsityParam}{\log \nsize}}, \Vert \xv \Vert_{\infty} \cdot \Vert \yv \Vert_{\infty}\right\}\log \nsize \right).
    \end{align*}
  \end{lemma}

  \begin{proof}
    We rewrite the bilinear form using the Kronecker delta:
    \begin{align*}
      \xv^{\T} \displaceMatrix \yv = 
      \sum_{1 \le i \le j \le \nsize} \displaceMatrix_{i j} (\xv_i \yv_j + \xv_j \yv_i) \left(1 - \frac{\delta_{i j}}{2} \right).
    \end{align*}
    Now it is the sum of independent random variables with variance
    \begin{align*}
      & \quad \Var \sum_{1 \le i \le j \le \nsize} \displaceMatrix_{i j} (\xv_i \yv_j + \xv_j \yv_i) \left(1 - \frac{\delta_{i j}}{2} \right) = 
      \sum_{1 \le i \le j \le \nsize} \EE \displaceMatrix_{i j}^2 (\xv_i \yv_j + \xv_j \yv_i)^2 \left(1 - \frac{\delta_{i j}}{2} \right)^2 \\
      & \le \sparsityParam \sum_{1 \le i \le j \le \nsize} (\xv_i^2 \yv_j^2 + \xv_j^2 \yv_i^2 + 2 \xv_i \xv_j \yv_i \yv_j ) \left(1 - \frac{\delta_{i j}}{2} \right)^2
      \le \sparsityParam \left(
        \Vert \xv \Vert^2 \cdot \Vert \yv \Vert^2 + \langle \xv, \yv \rangle^2
      \right) \le 2 \sparsityParam,
    \end{align*}
    and each element bounded by
    \begin{align*}
      \left|
        \displaceMatrix_{i j} (\xv_i \yv_j + \xv_j \yv_i) \left(1 - \frac{\delta_{i j}}{2} \right)
      \right| \le 2 \Vert \xv \Vert_{\infty} \cdot \Vert \yv \Vert_{\infty}.
    \end{align*}
    Applying the Bernstein inequality (Lemma~\ref{lemma: bernstein inequality}), we obtain
    \begin{align*}
      \PP \left( 
        \xv^{\T} \displaceMatrix \yv \ge t
      \right)
      \le
      \exp \left( 
        -\frac{t^2 / 2}{2 \sparsityParam + \frac{2 \Vert \xv \Vert_{\infty} \cdot \Vert \yv \Vert_{\infty}}{3} t}
      \right).
    \end{align*}
    Given $\varepsilon$, choose $\delta$ such that $\frac{\delta}{1 + \sqrt{\delta} / 3} \ge 4 \varepsilon$. If $\sqrt{\frac{\sparsityParam}{\log \nsize}} \ge \Vert \xv \Vert_{\infty} \cdot \Vert \yv \Vert_{\infty}$, then for $t = \sqrt{\delta \sparsityParam \log \nsize}$
    \begin{align*}
      \frac{t^2/4}{\sparsityParam + \Vert \xv \Vert_{\infty} \cdot \Vert \yv \Vert_{\infty} t/3} & = \frac{\delta \sparsityParam \log \nsize / 4}{\sparsityParam + \Vert \xv \Vert_{\infty} \cdot \Vert \yv \Vert_{\infty} \sqrt{\delta \sparsityParam \log \nsize} / 3} \\
      & \ge \frac{\delta \sparsityParam \log \nsize / 4}{\sparsityParam + \sparsityParam \sqrt{\delta} / 3} \\
      & \ge \frac{\delta / 4}{1 + \sqrt{\delta} / 3} \log \nsize \ge \varepsilon \log \nsize.
    \end{align*}
    That implies $\PP \left( \xv^{\T} \displaceMatrix \yv \ge t\right) \le \nsize^{-\varepsilon}$. The case of $\sqrt{\frac{\sparsityParam}{\log \nsize}} \le \Vert \xv \Vert_{\infty} \cdot \Vert \yv \Vert_{\infty}$ can be processed analogously. Thus, the statement holds.
  \end{proof}

  \begin{lemma}
  \label{lemma: adj eigenvectors displacement}
    Under Conditions~\ref{cond: nonzero B elements}-\ref{cond: theta distribution-a} we have
    \begin{align*}
      \max_i 
          \Vert 
            \adjacencyEigenvectors_i
            -
            \probEigenvectors_i
          \Vert
      = O_\ell \left( 
          \sqrt{
            \frac{\log \nsize}{\nsize^2 \sparsityParam}
          }
      \right).
    \end{align*}
  \end{lemma}

  \begin{proof}
    Due to Lemma~\ref{lemma: adj_eigenvectors_displacement}:
    \begin{align}
    \label{eq: lemma 13 U initial decomposition}
      \adjacencyEigenvectors_{ik} = 
      \probEigenvectors_{ik} + 
      \frac{1}{t_k}\displaceMatrix_i \uv_k +
      O_{\prec} \left(
        \frac{1}{\sqrt{\nsize} \lambda_k(\probMatrix)}
      \right)
    \end{align}
    as $t_k = \Theta\bigl(\lambda_k(\probMatrix)\bigr)$ due to Lemma~\ref{lemma: t_k is well-definied}, $\lambda_k(\probMatrix) = \Theta(\nsize \sparsityParam)$ due to Lemma~\ref{lemma: eigenvalues asymptotics} and $\alpha_{\nsize} = \Theta(\sqrt{\nsize \sparsityParam})$ due to Theorem~\ref{theorem: conditions satisfuction}.
    Thus, we can rewrite it in the following way:
    \begin{align*}
      \adjacencyEigenvectors_i = \probEigenvectors_i +
      \displaceMatrix_i \probEigenvectors \meanEigs^{-1} + 
      O_{\prec} \left(
          \frac{
            1
          }{
            \sqrt{\nsize} \lambda_\nclusters(\probMatrix)
          } 
      \right)
    \end{align*}
    for $\meanEigs = \diag(t_k)_{k \in [\nclusters]}$. Due to Lemma~\ref{lemma: log estimate vector difference}, Condition~\ref{cond: sparsity param bound} and Lemma~\ref{lemma: eigenvectors max norm}, we obtain
    \begin{align}
    \label{eq: displace estimation}
      \Vert 
        \displaceMatrix_i 
        \probEigenvectors
        \meanEigs^{-1} 
      \Vert 
      = O_{\ell}(\sqrt{\sparsityParam \log \nsize}) \cdot \Vert \meanEigs^{-1} \Vert.
    \end{align}
    Lemma~\ref{lemma: eigenvalues asymptotics} and Lemma~\ref{lemma: t_k is well-definied} guarantee that $\Vert \meanEigs^{-1} \Vert_{2} = O \left(\frac{1}{\nsize \sparsityParam} \right)$. Thus, 
    \begin{align*}
      \Vert \probEigenvectors_{i} - \adjacencyEigenvectors_i \Vert = O_\ell \left( \sqrt{\frac{\log \nsize}{\nsize^2 \sparsityParam}}\right).
    \end{align*}
    For each $i$, we have the same probabilistic reminder in~\eqref{eq: lemma 13 U initial decomposition}. In~\cite{Fan2019_SIMPLE}, it appears due to superpolynomial moment bounds of probability obtained from Lemma~\ref{lemma: power expectation} uniformly over $i$. Thus, the maximal reminder over $i \in [\nsize]$ has the same order. Similarly, we can take the maximum over $i$ for inequality~\eqref{eq: displace estimation} since superpolynomial bounds are provided via the Bernstein inequality and do not depend on $i$.
  \end{proof}

  \begin{lemma}
  \label{lemma: spa selection}
    Assumed Conditions~\ref{cond: nonzero B elements}-\ref{cond: theta distribution-a} to be satisfied, SPA chooses nodes $i_{1}, \ldots, i_{\nclusters}$ such that
    \begin{align*}
      \max_k \Vert
            \probEigenvectors_{i_k} - \basisMatrix_k
          \Vert
      =
      O_\ell \left(
        \frac{\sqrt{\log \nsize}}{\nsize \sqrt{\sparsityParam}}
      \right).
    \end{align*}
  \end{lemma}

  \begin{proof}
    To estimate error of SPA we need to apply Lemma~\ref{lemma: spa robustness} and, hence, we should estimate the difference between observed and real eigenvectors. From Lemma~\ref{lemma: adj eigenvectors displacement} we obtain that
    \begin{align*}
      \max_i
        \Vert
          \adjacencyEigenvectors_i
          -
          \probEigenvectors_i
        \Vert
      \leqslant
      \frac{\delta_1 \sqrt{\log \nsize}}{\nsize \sqrt{\sparsityParam}}
    \end{align*}
    with probability at least $1 - \nsize^{- \varepsilon}$ for any $\varepsilon$ and large enough $\delta_1$. Thus, due to Lemma~\ref{lemma: spa robustness} we conclude that SPA chooses some indices $i_1, \ldots, i_k$ such that
    \begin{align*}
      \PP \left(
          \max_k
          \Vert 
            \adjacencyEigenvectors_{i_k} - \basisMatrix_k
          \Vert
          \geqslant
          \frac{
            \delta_1 \sqrt{\log \nsize }
          }{
            \nsize \sqrt{\sparsityParam}
            \bigl(
              1 + 80 \kappa (\mathbf{F})
            \bigr)^{-1}
          }
      \right)
      \leqslant \nsize^{-\varepsilon}.
    \end{align*}
    Using triangle inequality, we notice
    \begin{align*}
      \Vert
        \probEigenvectors_{i_k} - \basisMatrix_k 
      \Vert
      \leqslant
      \Vert
        \probEigenvectors_{i_k} -
        \adjacencyEigenvectors_{i_k}
      \Vert
      +
      \Vert
        \adjacencyEigenvectors_{i_k}
        -
        \basisMatrix_k
      \Vert,
    \end{align*}
    and it implies that there is some constant $C$ such that:
    \begin{align*}
      \PP \left(
        \max_k \Vert
          \probEigenvectors_{i_k} - \basisMatrix_k
        \Vert
        \geqslant
        \frac{
          C \sqrt{\log \nsize}
        }{
          \nsize \sqrt{\sparsityParam}
        }
      \right)
      \leqslant
      \nsize^{ -\varepsilon}
    \end{align*}
    since $\kappa (\basisMatrix)$ is bounded by a constant due to Lemma~\ref{lemma: F rows tensor product}.
  \end{proof}

\subsubsection{Eigenvalues behavior}
  \begin{lemma}
  \label{lemma: F rows tensor product}
    Under Condition~\ref{cond: theta distribution-a} the singular numbers of the matrix $\sqrt{\nsize} \basisMatrix$ are bounded away from 0 and $\infty$. Moreover, for any set $\beta_1, \ldots, \beta_\nclusters$ of positive numbers, bounded away from 0 and $\infty$, the matrix
    \begin{align*}
      \sumMatrix = \sum_{k = 1}^\nclusters \beta_k \basisMatrix_k^{\T} \basisMatrix_k
    \end{align*}
    is full rank, and there are such constants $C_1$, $C_2$ that
    \begin{align*}
      \frac{C_1}{\nsize} 
      \leqslant 
      \lambda_{\min} (\sumMatrix)
      \leqslant
      \lambda_{\max} (\sumMatrix)
      \leqslant
      \frac{C_2}{\nsize}.
    \end{align*}
  \end{lemma}

  \begin{proof}
    Since the matrix $\basisMatrix$ is full rank, its rows are linearly independent. Hence, if $\beta_k > 0$, matrix $\sumMatrix$ is full rank. Now we want to estimate eigenvalues of $\sumMatrix$:
    \begin{align*}
      \lambda_{\min} (\sumMatrix) & = 
      \inf_{\Vert \vv \Vert = 1} \vv^{\T} \sumMatrix \vv
      = 
      \inf_{\Vert \vv \Vert = 1}
         \sum_{k = 1}^\nclusters
              \beta_k
              (\vv^{\T} \basisMatrix_k^{\T})^2 \\
      & \geqslant
          (\min_{k} \beta_k)
          \inf_{\Vert \vv \Vert = 1}
          \sum_{k = 1}^{\nclusters}
            \vv^{\T}
            \basisMatrix^{\T}
            \ev_k
            \ev_k^{\T}
            \basisMatrix
            \vv \\
      & = (\min_k \beta_k)
      \inf_{\Vert \vv \Vert = 1}
      \vv^{\T} \basisMatrix^{\T} \basisMatrix \vv = 
      \lambda_{\min} (\basisMatrix^{\T} \basisMatrix) \min_k \beta_k.
    \end{align*}
    In the other side, using multiplicative Weyl's inequality we obtain
    \begin{align*}
      \sigma_{\min} (\communityMatrix) = 
      \sigma_{\min} (\basisMatrix \probEigenvalues \basisMatrix^{\T}) =
      \sigma_{\min} (\basisMatrix^{\T} \basisMatrix \probEigenvalues)
      \leqslant
      \sigma_{\min} (\basisMatrix^{\T} \basisMatrix) \sigma_{\max} (\probEigenvalues).
    \end{align*}
    Hence,
    \begin{align*}
      \lambda_{\min} (\basisMatrix^{\T} \basisMatrix) 
      \geqslant
      \frac{
        |\lambda_{\min} (\communityMatrix)|
      }{
        |\lambda_1(\probMatrix)|
      }
      \geqslant
      \frac{
        |\lambda_{\min} (\constCommunityMatrix)|
      }{
        C' \nsize
      },
    \end{align*}
    where constant $C'$ was taken from Lemma~\ref{lemma: eigenvalues asymptotics}. Similarly, we have 
    \begin{align*}
        \lambda_{\max}(\sumMatrix) \le (\max_{k} \beta_k) \sigma_{\max}(\basisMatrix^{\T} \basisMatrix) \le \frac{\sigma_{\max}(\communityMatrix)}{\sigma_{\nclusters}(\probMatrix)}.
    \end{align*} 
    We finally conclude that
    \begin{align*}
      \frac{C_1}{\nsize} 
      \leqslant 
      \lambda_{\min} (\sumMatrix)
      \leqslant
      \lambda_{\max} (\sumMatrix)
      \leqslant
      \frac{C_2}{\nsize},
    \end{align*}
    where
    \begin{equation*}
      C_1 = 
      \frac{
        \lambda_{\min} (\constCommunityMatrix) \min_k \beta_k 
      }{
        C' \nsize
      }, \quad
      C_2 = 
      \frac{
        \lambda_{\max} (\constCommunityMatrix) \max_k \beta_k 
      }{
        c' \nsize
      }.
    \end{equation*}
  \end{proof}

  \begin{lemma}
  \label{lemma: eigenvalues asymptotics}
    Under Condition~\ref{cond: theta distribution-a} there are such constants $c, C, c', C'$ that
    \begin{align*}
      c \nsize \leqslant \lambda_{\nclusters} (\nodeCommunityMatrix^{\T} \nodeCommunityMatrix) \leqslant \lambda_{\max}(\nodeCommunityMatrix^{\T} \nodeCommunityMatrix) \leqslant C \nsize 
    \end{align*}
    and
    \begin{align*}
      c' \nsize \sparsityParam 
      \leqslant| \lambda_{\nclusters} (\probMatrix)| 
      \leqslant |\lambda_{\max}(\probMatrix)| 
      \leqslant C' \nsize \sparsityParam.
    \end{align*}
  \end{lemma}

  \begin{proof}
    By Condition~\ref{cond: theta distribution-a}, we have $\lambda_K(\nodeCommunityMatrix^\T \nodeCommunityMatrix) \ge c n$ for some constant $c$. Thus, to get the first statement of the lemma, it is enough to bound the norm of $\nodeCommunityMatrix^\T \nodeCommunityMatrix$:
    \begin{align*}
      \Vert \nodeCommunityMatrix^{\T} \nodeCommunityMatrix \Vert \le 
      \sum_{i = 1}^\nsize
         \Vert \nodeCommunityMatrix_i^{\T} \nodeCommunityMatrix_i \Vert
      = \sum_{i = 1}^n \Vert \nodeWeights_i \Vert^2 \le n.
    \end{align*}
 The eigenvalues of $\probMatrix$ we bound using multiplicative Weyl's inequality for singular numbers:
    \begin{align*}
      |\lambda_k (\nodeCommunityMatrix \communityMatrix \nodeCommunityMatrix^{\T})| 
      = \sigma_{k} (\nodeCommunityMatrix \communityMatrix \nodeCommunityMatrix^{\T}),
      \quad
      \sigma_{\min}^2 (\nodeCommunityMatrix) \sigma_{\min} (\communityMatrix)
      \leqslant
      \sigma_{k} (\nodeCommunityMatrix \communityMatrix \nodeCommunityMatrix^{\T})
      \leqslant
      \sigma_{\max}^2 (\nodeCommunityMatrix) \sigma_{\max} (\communityMatrix).
    \end{align*}
    The previous statement and the fact that $\sigma_{k} (\communityMatrix) = \sparsityParam \sigma_k (\constCommunityMatrix)$ prove the lemma.
  \end{proof}

\subsection{Tools}
\subsubsection{Useful lemmas from previous studies}
\newtheorem{conditionFan}{Condition}
\renewcommand*{\theconditionFan}{\Alph{conditionFan}}
\newtheorem{theoremFan}{Theorem}
\renewcommand*{\thetheoremFan}{\Alph{theoremFan}}

We widely use results from~\cite{Fan2019_SIMPLE} and~\cite{Fan2020_ASYMPTOTICS}, so we write a special section that summarizes these results.

\subsubsection{Conditions}
First, we must show that conditions demanded in~\cite{Fan2019_SIMPLE} and~\cite{Fan2020_ASYMPTOTICS} hold under our conditions. Let us first review these conditions. 

\begin{conditionFan}
\label{cond: fan_condition_1}
  There exists some positive constant $c_0$ such that 
  \begin{align*}
    \min \left\{
      \frac{|\lambda_i(\probMatrix)|}{|\lambda_j (\probMatrix)|} 
      \mid
      1 \leqslant i < j \leqslant \nclusters,
      \lambda_i(\probMatrix) \neq \lambda_j(\probMatrix)
    \right\}
    \geqslant
    1 + c_0.
  \end{align*}
  In addition, 
  \begin{align*}
    \alpha_{\nsize} := \left\{
      \max_{1 \leqslant j \leqslant \nsize}
        \sum_{i = 1}^\nsize
        \Var(\displaceMatrix_{ij})
    \right\}^{1/2} \underset{\nsize \to \infty}{\longrightarrow} \infty.
  \end{align*}
\end{conditionFan}

\begin{conditionFan}
\label{cond: fan_condition_2}
  There exist some constants $0 < c_0, c_1 < 1$ such that $\lambda_{\nclusters}(\nodeCommunityMatrix^{\T} \nodeCommunityMatrix) \geqslant c_0 \nsize$, $|\lambda_{\nclusters}(\probMatrix)| \geqslant c_0$, and $\sparsityParam \geqslant \nsize^{-c_1}$.
\end{conditionFan}

%

In this way, we prove the following theorem.
\begin{theoremFan}
\label{theorem: conditions satisfuction}
  Assume Conditions~\ref{cond: nonzero B elements}-\ref{cond: theta distribution-a} hold. Then Conditions~\ref{cond: fan_condition_1}-\ref{cond: fan_condition_2} are satisfied. Moreover, $\alpha_{\nsize} = O(\sqrt{\nsize \sparsityParam})$.
\end{theoremFan}

\begin{proof}
  Condition~\ref{cond: eigenvalues divergency} implies Condition~\ref{cond: fan_condition_1} directly. Condition~\ref{cond: fan_condition_2} is valid due to Lemma~\ref{lemma: eigenvalues asymptotics} and Condition~\ref{cond: sparsity param bound}. Finally, we have
  \begin{align*}
      \alpha_n^2 = \max_{j} \sum_{i = 1}^\nsize \probMatrix_{ij} (1 - \probMatrix_{i j}) \le \sparsityParam \nsize. & 
  \end{align*}
\end{proof}

Thus, under Conditions~\ref{cond: nonzero B elements}-\ref{cond: theta distribution-a} we can use key statements from~\cite{Fan2019_SIMPLE} and~\cite{Fan2020_ASYMPTOTICS} that are summarized below.

\subsubsection{Lemmas}
\begin{lemma}[Lemma 6 from~\cite{Fan2019_SIMPLE}]
\label{lemma: eigenvectors max norm}
  Under Conditions~\ref{cond: fan_condition_1}-\ref{cond: fan_condition_2} there exists such constant $C_{\probEigenvectors}$ that
  \begin{align}
    \max_{ij} |\probEigenvectors_{ij}| \leqslant \frac{C_{\probEigenvectors}}{\sqrt{\nsize}}.
  \end{align}
\end{lemma}

Next, we provide an asymptotic expansion of $\xv^{\T} \estimator[\uv]_k \estimator[\uv]_k^{\T} \yv$. Its form is a bit sophisticated and demands auxiliary notation described in Table~\ref{tab: expansion notation}. In addition, it involves the solution of equation~\eqref{eq: t_k definition}. The following lemma guarantees that it is well-defined.
\begin{lemma}[Lemma 3 from~\cite{Fan2020_ASYMPTOTICS}]
\label{lemma: t_k is well-definied}
  Under Condition~\ref{cond: fan_condition_1}, equation~\eqref{eq: t_k definition} has an unique solution in the interval $z \in [a_k, b_k]$ and thus $t_k$'s are well-defined. Moreover, for each $1 \le k \le \nclusters$, we have $t_k/\lambda_k(\probMatrix) \to 1$ as $\nsize \to \infty$.
\end{lemma}

Now we provide the necessary asymptotics.
\begin{lemma}[Theorem 5 from~\cite{Fan2020_ASYMPTOTICS}]
\label{lemma: fan eigenvectors series decomposition}
  Assume that Conditions~\ref{cond: fan_condition_1}-\ref{cond: fan_condition_2} hold and $\xv$ and $\yv$ are two $\nsize$-dimensional unit vectors. Then for each $1 \le k \le \nclusters$, if $\sigma_k^2 = O(\tilde{\sigma}_k^2)$ and $\tilde{\sigma}_k^2 \gg t_k^{-4} (|\meanFactor_{\xv, k, t_k}| + |\meanFactor_{\yv, k, t_k}|)^2 + t_k^{-6}$, we have the asymptotic expansion
  \begin{align*}
    \xv^{\T} \estimator[\uv]_k \estimator[\uv]_k^{\T} \yv = \, & a_k 
    + 
    \tr [
      \displaceMatrix \jMatrix_{\xv, \yv, k, t_k} - (\displaceMatrix^2 - \EE \displaceMatrix^2) \lMatrix_{\xv, \yv, k, t_k}
    ]
    + \tr (\displaceMatrix \uv_k \uv_k^{\T}) \tr (\displaceMatrix \qMatrix_{\xv, \yv, k, t_k}) \\
    & + O_p \left(|t_k|^{-3} \alpha_\nsize^2 (|\meanFactor_{\xv, k, t_k}| + |\meanFactor_{\yv, k, t_k}|) + |t_k|^{-3} \right),
  \end{align*}
  where $a_k = \meanFactor_{\xv, k, t_k} \meanFactor_{\yv, k, t_k} \pFactor_{k, t_k}$.
\end{lemma}

\begin{lemma}[see Lemma~10 from~\cite{Fan2019_SIMPLE} and its proof]
\label{lemma: corrected eigs and noise}
  Under Conditions~\ref{cond: fan_condition_1}-\ref{cond: fan_condition_2} it holds that
  \begin{align}
    \debiasedEigenvalues_{kk} = \lambda_k(\probMatrix) + O_{\prec} \left(
      \sqrt{\sparsityParam} + \frac{1}{\sqrt{\nsize \sparsityParam}}
    \right)
  \end{align}
  and uniformly over all $i, j$ 
  \begin{align}
    \estimator[\displaceMatrix]_{ij} = \displaceMatrix_{ij} + O_{\prec}\left( \sqrt{\frac{\sparsityParam}{\nsize}}\right).
  \end{align}
\end{lemma}

\begin{lemma}[Lemma 9 from~\cite{Fan2019_SIMPLE}]
\label{lemma: adj_eigenvectors_displacement}
  Under Conditions~\ref{cond: fan_condition_1}-\ref{cond: fan_condition_2}, we have
  \begin{align}
    \adjacencyEigenvectors_{ik} = \probEigenvectors_{ik} + \frac{1}{t_{k}} \displaceMatrix_{i} \uv_k + O_{\prec} \left( 
      \frac{\alpha_\nsize^2}{\sqrt{\nsize} t_{k}^2}
      +
      \frac{1}{|t_k| \sqrt{\nsize}}
    \right),
  \end{align}
  where $\uv_k$ is the $k$-th column of the matrix $\probEigenvectors$.
\end{lemma}

\begin{lemma}[Lemma 8 from~\cite{Fan2019_SIMPLE}]
\label{lemma: eigenvalues difference}
  Under Conditions~\ref{cond: fan_condition_1}-\ref{cond: fan_condition_2}, for each $1 \le k \le \nclusters$ we have
  \begin{align*}
    \adjacencyEigenvalues_{kk} - t_k = \uv_k^{\T} \displaceMatrix \uv_k + O_{\prec} \left( \frac{\alpha_{\nsize}^2}{\sqrt{\nsize} \lambda_k(\probMatrix)}\right).
  \end{align*}
\end{lemma}

\begin{lemma}[Lemma 11 and Corollary 3 from~\cite{Fan2019_SIMPLE}]
\label{lemma: power deviation}
  For any $\nsize$-dimensional unit vectors $\xv, \yv$ and any positive integer $r$, we have 
  \begin{align*}
    \EE \left[ 
      \xv^{\T} (\displaceMatrix^\ell - \EE \displaceMatrix^\ell) \yv
    \right]^{2r} \le C_r (\min(\alpha_{\nsize}^{\ell - 1}, \Vert \xv \Vert_{\infty} \alpha_{\nsize}^\ell, \Vert \yv \Vert_{\infty} \alpha_{\nsize}^\ell)^{2r},
  \end{align*}
  where $\ell$ is any positive integer and $C_r$ is some positive constant determined only by $r$. Additionnally, we have
  \begin{align*}
      \xv^\T (\displaceMatrix^\ell - \EE \displaceMatrix^\ell ) \yv = O_{\prec} (\min(\alpha_{\nsize}^{\ell - 1}, \Vert \xv \Vert_{\infty} \alpha_{\nsize}^\ell, \Vert \yv \Vert_{\infty} \alpha_{\nsize}^\ell).
  \end{align*}
\end{lemma}

\begin{lemma}[Lemma 12 from~\cite{Fan2019_SIMPLE}]
\label{lemma: power expectation}
  For any $\nsize$-dimensional unit vectors $\xv$ and $\yv$, we have
  \begin{align*}
    \EE \xv^{\T} \displaceMatrix^\ell \yv = O(\alpha_{\nsize}^\ell),
  \end{align*}
  where $\ell \ge 2$ is a positive integer. Furthermore, if the number of nonzero components of $\xv$ is bounded, then it holds that
  \begin{align*}
    \EE \xv^{\T} \displaceMatrix^\ell \yv = O(\alpha_{\nsize}^{\ell} \Vert \yv \Vert_{\infty}).
  \end{align*}
\end{lemma}

Table~\ref{tab: expansion notation} summarizes the notations from~\cite{Fan2020_ASYMPTOTICS} that are needed for the proofs of our results.
\begin{table}[t!]
  \centering
  \begin{tabular}{|l|}
      \hline
       Auxiliary variables \\
      \hline
      \\
       $L = \min \left\{\ell \mid \left( \frac{\alpha_{\nsize}}{\max\{|a_k|, |b_k|\}}\right)^\ell \le \min \left\{\frac{1}{\nsize^4}, \frac{1}{\max\{|a_k|^4, |b_k|^4\}}\right\} \right\}$ \\ \\
      $\resolvent(\mathbf{M}_1, \mathbf{M}_2, t) = 
    - \frac{1}{t} \mathbf{M}_1^{\T} \mathbf{M}_2 - \sum_{l = 2}^{L} 
    t^{-(l + 1)} \mathbf{M}_1^{\T} \EE \displaceMatrix^{l} \mathbf{M}_2$
    \\ \\
    $\tResolvent(\mathbf{M}_1, \mathbf{M}_2, t) = t \resolvent(\mathbf{M}_1, \mathbf{M}_2, t)$
    \\ \\
    $\bv_{\xv, k, t} = \xv - \probEigenvectors_{-k}
    \left[ 
      \probEigenvalues^{-1}_{-k} + 
      \resolvent(\probEigenvectors_{-k}, \probEigenvectors_{-k}, t)
    \right]^{-1}
    \resolvent^{\T}(\xv, \probEigenvectors_{-k}, t)$
    \\ \\
    \hline
    0-degree coefficients \\
    \hline
    \\
    $\meanFactor_{\xv, k, t} = \tResolvent (\xv, \uv_k, t) - 
    \tResolvent(\xv, \probEigenvectors_{-k}, t)
    \left[
      t \probEigenvalues_{-k}^{-1} + 
      \tResolvent(\probEigenvectors_{-k}, \probEigenvectors_{-k}, t)
    \right]^{-1}
    \tResolvent (\probEigenvectors_{-k}, \uv_k, t)$
    \\ \\
    $\pFactor_{k, t} = \left[
      t^2 \frac{d}{dt} \left(
        \frac{\meanFactor_{\uv_k, k, t}}{t}
      \right)
    \right]^{-1}$
    \\ \\
    \hline
    First degree coefficients \\
    \hline \\
    $\jMatrix_{\xv, \yv, k, t_k} 
        = 
        - \pFactor_{k, t_k} t_k^{-1} \uv_k
        \left(
          \meanFactor_{\yv, k, t_k} \bv_{\xv, k, t_k}^{\T} 
          + 
          \meanFactor_{\xv, k, t_k} \bv_{\yv, k, t_k}^{\T} 
          +
          2 \meanFactor_{\xv, k, t_k} \meanFactor_{\yv, k, t_k} \pFactor_{k, t_k} \uv_k^{\T}
       \right)$\\ \\
      \hline
      Second degree coefficients \\
      \hline \\
      $\lMatrix_{\xv, \yv, k, t_k} = 
      \pFactor_{k, t_k} t_k^{-2} \uv_k 
      \Bigl\{ 
        \bigl[
          \meanFactor_{\yv, k, t_k} \resolvent(\xv, \probEigenvectors_{-k}, t_k) 
          +
          \meanFactor_{\xv, k, t_k} \resolvent(\yv, \probEigenvectors_{-k}, t_k)
        \bigr] \times $ \\
        \qquad \qquad \; 
        $\times \bigl[
          \probEigenvalues_{-k}^{-1} + \resolvent(\probEigenvectors_{-k}, \probEigenvectors_{-k}, t_k)
        \bigr]^{-1} \probEigenvectors_{-k}^{\T} + \meanFactor_{\yv, k, t_k} \xv^{\T} + \meanFactor_{\xv, k, t_k} \yv^{\T} + 3 \meanFactor_{\xv, k, t_k} \meanFactor_{\yv, k, t_k} \uv_k^{\T}
      \Bigr\}$ \\ \\
      $\qMatrix_{\xv, \yv, k, t_k}  = 
      \lMatrix_{\xv, \yv, k, t_k} 
      - 
      \pFactor_{k, t_k} t_k^{-2} \meanFactor_{\xv, k, t_k} 
      \meanFactor_{\yv, k, t_k} \uv_k \uv_k^{\T}$ \\
      $\qquad \qquad\quad  + 
      2 \pFactor^2_{k, t_k} t_k^{-2} \uv_k 
      \left(
        \meanFactor_{\xv, k, t_k} \bv_{\xv, k, t_k}^{\T} + \meanFactor_{\yv, k, t_k} \bv_{\yv, k, t_k}^{\T}
      \right)
      $ \\ \\
      \hline
      Applicability parameters \\
      \hline \\
      $\sigma_k^2 = \Var\bigl[\tr(\displaceMatrix \jMatrix_{\xv, \yv, k, t_k})\bigr]$ \\ \\
      $\tilde{\sigma}_k^2 = \Var \left\{
        \tr \bigl[
          \displaceMatrix \jMatrix_{\xv, \yv, k, t_k}
          -
          (\displaceMatrix^2 - \EE \displaceMatrix) \lMatrix_{\xv, \yv, k, t_k}
        \bigr]
        +
        \tr \left( \displaceMatrix \uv_k \uv_k^{\T} \right)
        \tr \left( \displaceMatrix \qMatrix_{\xv, \yv, k, t_k} \right)
      \right\}$ \\ \\
      \hline
  \end{tabular}
  \caption{
  Here $\probEigenvectors_{-k}$ is the matrix $\probEigenvectors$ with a $k$-th column removed and $\probEigenvalues_{-k}$ is a diagonal matrix that contains all eigenvalues except $k$-th one, while $t_k$ is the solution of~\eqref{eq: t_k definition}.
  }
  \label{tab: expansion notation}
\end{table}

\subsubsection{Concentration inequalities}
Across this paper, we use several concentration inequalities. We listed them here. The first one is the Bernstein inequality. For the proof one can see, for example, \S~2.8 in the book by~\cite{Boucheron2013}.

\begin{lemma}[Bernstein inequality]
\label{lemma: bernstein inequality}
  Let $X_1, \ldots, X_\nsize$ be independent random variables with zero mean. Assume that each of them is bounded by some constant $M$. Then for all $t > 0$:
  \begin{align*}
    \PP \left( 
      \sum_{i = 1}^\nsize X_i \ge t
    \right)
    \le 
    \exp \left(
      -\frac{t^2 / 2}{\sum_{i = 1}^\nsize \EE X_i^2 + M t / 3}
    \right).
  \end{align*}
\end{lemma}


The Bernstein inequality can be generalized for random matrices:
\begin{lemma}[Matrix Bernstein inequality]
\label{lemma: matrix bernstein inequality}
  Let $\mathbf{X}_1, \ldots, \mathbf{X}_\nsize$ be independent zero-mean $a \times b$ random matrices such that their norms are bounded by some constant $M$. Then, for all $t > 0$ it holds that
  \begin{align*}
    \PP \left( 
      \left \Vert \sum_{i = 1}^\nsize \mathbf{X}_i \right \Vert \ge t
    \right)
    \le (a + b) \exp \left(
      -\frac{t^2/2}{\sigma^2 + M t / 3}
    \right),
  \end{align*}
  where 
  \begin{align*}
    \sigma^2 = \max \left( 
      \left \Vert 
        \sum_{i = 1}^\nsize \EE (\mathbf{X}_i \mathbf{X}_i^{\T})
      \right \Vert,
      \left \Vert
        \sum_{i = 1}^\nsize \EE (\mathbf{X}_i^{\T} \mathbf{X}_i)
      \right \Vert
    \right).
  \end{align*}
\end{lemma}
For the proof we refer reader to the book by~\cite{Tropp2015}.

\subsubsection{Properties of SPA}
This part describes the properties of SPA procedure, see Algorithm~\ref{algo: spa}. Here we use the same notation as~\cite{Mizutani2016}. Thus, we denote
\begin{align}
  \mathbf{A} = \mathbf{F} \mathbf{W} \text{ for } \mathbf{F} \in \RR^{d \times r}_+ \text{ and } \textbf{W} = (\identity, \mathbf{K}) \mathbf{\Pi} \in \RR^{r \times m}_+,
\label{eq:spa_decomposition}
\end{align}
where $\identity$ is an $r \t r$ identity matrix, $\mathbf{K}$ is an $r \t (m - r)$ nonnegative matrix, and $\mathbf{\Pi}$ is an $m \t m$ permutation matrix. Then, the following theorem holds.
\begin{lemma}[Theorem 1 from~\cite{Mizutani2016}]
\label{lemma: spa robustness}
  Let $\widetilde{\mathbf{A}} = \mathbf{A} + \mathbf{N}$ for $\mathbf{A} \in \RR^{d \times m}$ and $\mathbf{N} \in \RR^{d \times \nsize}$. Suppose that $r > 2$ and $\mathbf{A}$ satisfies equation~\eqref{eq:spa_decomposition}. If row $\mathbf{n}_i$ of $\mathbf{N}$ satisfies $\Vert \mathbf{n}_i \Vert_2 \leqslant \varepsilon$ for all $i = 1, \ldots, m$ with
  \begin{align}
    \varepsilon < \min \left(
      \frac{1}{
        2 \sqrt{r - 1}
      },
      \frac{1}{4}
    \right)
    \frac{
      \sigma_{min} (\mathbf{F})
    }{
      1 + 80 \kappa(\mathbf{F})
    },
  \end{align}
  then, SPA with input $(\widetilde{\mathbf{A}}, r)$ returns the output $\mathcal{I}$ such that there is an order of the elements in $\mathcal{I}$ satisfying
  \begin{align}
    \Vert
      \widetilde{\mathbf{a}}_{\mathcal{I}(j)} - \mathbf{f}_j
    \Vert_2
    \leqslant
    \varepsilon 
    \bigl(1 + 80 \kappa (\mathbf{F})\bigr).
  \end{align}
\end{lemma}

\section{Tools for Theorem~\ref{theorem: lower bound}}
\subsection{Lower bound on risk based on two hypotheses}

Let $\Theta$ be an arbitrary parameter space, equipped with semi-distance $d: \Theta \times \Theta \to [0, +\infty)$, i.e.
  \begin{enumerate}
    \item for any $\theta, \theta' \in \Theta$, we have $d(\theta', \theta) = d(\theta, \theta')$,

    \item for any $\theta_1, \theta_2, \theta_3 \in \Theta$, we have $d(\theta_1, \theta_2) + d(\theta_2, \theta_3) \ge d(\theta_1, \theta_3)$,

    \item for any $\theta \in \Theta$, we have $d(\theta, \theta) = 0$.
  \end{enumerate}
  For $\theta \in \Theta$, we denote the corresponding distribution by $\PP_{\theta}$. The following lemma bounds below the risk of estimation of parameter $\theta$ for the loss function $d(\cdot, \cdot)$ and any estimator $\estimator[\theta]$.

  \begin{lemma}[Theorem 2.2,~\cite{tsybakov_introduction_2009}]
    \label{lemma: minimax bounds for 2 hypotheses}
    Suppose that for two parameters $\theta_1, \theta_0$ such that we have $d(\theta_1, \theta_0) \ge s$ and $\KL(\PP_{\theta_1} \Vert \PP_{\theta_0}) \le \alpha$. Then
    \begin{align*}
      \inf_{\hat{\theta}} \sup_{\theta \in \{\theta_1, \theta_0\} } \PP \left (d (\estimator[\theta], \theta) \ge s/2 \right ) \ge \frac{1}{4} e^{-\alpha}.
    \end{align*}    
  \end{lemma}

\subsection{Asymptotically good codes}
  To prove Theorem~\ref{theorem: lower bound}, we use a variation of Fano's lemma based on many hypotheses. A common tool to construct such hypotheses is the following lemma from the coding theory.

  \begin{lemma}[Lemma 2.9,~\cite{tsybakov_introduction_2009}]
  \label{lemma: varshamov-gilbert bound}
    Let $m \ge 8$. Then there exists a subset $\{\omega^{(0)}, \omega^{(1)}, \ldots, \newline \omega^{(M)}\}$ of $\{0, 1\}^m$ such that $\omega^{(0)} = \zero$, for any distinct $i, j = 0, \ldots, M$, we have
    \begin{align*}
      d_H(\omega^{(i)}, \omega^{(j)}) \ge \frac{m}{8},
    \end{align*}
    and
    \begin{align*}
      M \ge 2^{m / 8}.
    \end{align*}
  \end{lemma}

\subsection{Lower bound on risk based on many hypotheses} 

 The following lemma generalizes Lemma~\ref{lemma: minimax bounds for 2 hypotheses} in the case of many hypotheses.

  \begin{lemma}[Theorem 2.5,~\cite{tsybakov_introduction_2009}]
  \label{lemma: fano lemma}
    Assume that $M \ge 2$ and suppose that $\Theta$ contains elements $\theta_0, \theta_1, \ldots, \theta_M$ such that:
    \begin{enumerate}[label=(\roman*)]
      \item for all distinct $i,j$, we have $d(\theta_i, \theta_j) \ge 2 s > 0$,
      \item for the KL-divergence it holds that
      \begin{align*}
          \frac{1}{M} \sum_{j = 1}^M \KL(\PP_{\theta_j} \Vert \PP_{\theta_0}) \le \alpha \log M
      \end{align*}
      for $\alpha \in (0, 1/8)$.
    \end{enumerate}
    Then
    \begin{align*}
      \inf_{\estimator[\theta]} \sup_{\theta \in \Theta} \PP \left( d(\estimator[\theta], \theta) \ge s \right)
      \ge 
      \frac{\sqrt{M}}{1 + \sqrt{M}}
      \left( 
          1 - 2 \alpha - \sqrt{
              \frac{2 \alpha}{M}
          }
      \right).
    \end{align*}
  \end{lemma}

\subsubsection{Gershgorin's circle theorem}
  We use the following theorem that is a common tool to bound eigenvalues of arbitrary matrix. For the proof, one can see the book~\cite{horn_johnson_2012}.

  \begin{lemma}
  \label{lemma: gershgorin circle theorem}
    Let $\mathbf{X}$ be a complex $\nsize \times \nsize$ matrix. For $i \in [\nsize]$, define
    \begin{align*}
      R_i = \sum_{j \neq i} |\mathbf{X}_{ij}|.
    \end{align*}
    Let $B({\bf X}_{ii}, R_i) \subset \mathbb{C}$, $i \in [\nsize]$, be a circle on the complex plane with the center ${\bf X}_{ii}$ and the radius $R_i$. Then all eigenvalues of $\mathbf{X}$ are contained in $\bigcup_{i \in [\nsize]} B({\bf X}_{ii}, R_i)$, and each connected component of $\bigcup_{i \in [\nsize]} B({\bf X}_{ii}, R_i)$ contains at least one eigenvalue.
  \end{lemma}


\end{document}